\documentclass[preprint,12pt]{elsarticle}

\usepackage{amssymb}
\usepackage{amsmath}
\usepackage{booktabs}
\usepackage{multirow}
\usepackage{comment}
\usepackage{enumitem}
\usepackage{wrapfig}
\usepackage{enumitem}
\usepackage{setspace}
\usepackage{makecell}
\usepackage{xspace}
\usepackage[table]{xcolor}
\usepackage{tabularx}
\usepackage{soul}
\usepackage{hyperref}
\usepackage{algorithm,algpseudocode}
\usepackage{longtable}
\usepackage[dvipsnames]{xcolor}
\usepackage{tikz}
\usetikzlibrary{positioning}
\usetikzlibrary{matrix,arrows.meta}
\newcommand{\set}[1]{\{#1\}}

\newcommand{\LTL}{\textsc{ltl}\xspace}
\newcommand{\LTLF}{\textsc{ltl}\ensuremath{_f}\xspace}
\newcommand{\FLTLF}[1][x]{\textsc{fltl}\ensuremath{_f^{#1}}\xspace}
\newcommand{\FLTL}{\textsc{fltl}\xspace}

\renewcommand{\P}{\mathcal{P}}

\newcommand{\X}{\mathcal{X}}
\newcommand{\p}{p\xspace}

\NewDocumentCommand\tomorrow    {}{\mathsf{X}}
\NewDocumentCommand\until       {}{\mathbin{\mathsf{U}}}
\NewDocumentCommand\always      {}{\mathsf{G}}
\NewDocumentCommand\eventually  {}{\mathsf{F}}
\NewDocumentCommand\release     {}{\mathbin{\mathsf{R}}}
\NewDocumentCommand\mrelease     {}{\mathbin{\mathsf{M}}}

\newcommand{\declare}{{\textsc{declare}}\xspace}
\newcommand{\limp}{\mathbin{\rightarrow}}

\usetikzlibrary{tikzmark,arrows.meta, shapes, positioning, calc, decorations.pathmorphing, backgrounds}

 \usetikzlibrary{
  arrows,
   arrows.meta,
   matrix,
   tikzmark,
 }

\newcommand{\bgcell}[2]{
|[fill=#2]|#1
}

 \newcommand{\fuzzyeval}[3]{\bgcell{#1}{#3!\fpeval{100*#2}}}

\newcommand{\fuzzyval}[1]{
\fuzzyeval{#1}{#1}{LimeGreen}
}

\newcommand{\fuzzyg}[2]{
\fuzzyeval{#1}{#2}{LimeGreen}
}
\newcommand{\fuzzyp}[2]{
\fuzzyeval{#1}{#2}{LimeGreen}
}

\newcommand{\formulabox}[3]{
\subnode[inner sep=0pt,outer sep=0pt]
    {#1}
    {
        \colorbox{#3}{\ensuremath{#2}}
    }
}

\tikzstyle{table} = [
    matrix of nodes,
    ampersand replacement=\&,
    row sep=-\pgflinewidth,
    column sep=.5em,
    nodes={
      rectangle,
      draw=black,
      align=center,
      text centered,
      font=\footnotesize\ttfamily
    },
    text height=1em,
    text depth=0.5em,
    nodes in empty cells,
]

\tikzstyle{trace-table} = [
	table,
    row 1/.style={
      nodes={
        fill=white,
        draw=white,
        text=
            black,
            font=\footnotesize
      }
    }
]

\tikzstyle{eval-table} = [
    table,
]

\tikzstyle{link} = [
    arrows= {-Stealth},
    thick,
]

\definecolor{cA}{HTML}{2166AC}
\definecolor{cB}{HTML}{1B7837}
\definecolor{cO}{HTML}{762A83}
\definecolor{cG}{HTML}{C0392B}
\definecolor{cF}{HTML}{E8A33D}

\usepackage{amsthm}%
\usepackage{multicol}
\newtheorem{theorem}{Theorem}%

\newtheorem{example}{Example}%
\newtheorem{definition}{Definition}%
\newtheorem{proposition}[theorem]{Proposition}% 

\newcommand{\fluent}[1]{\mathtt{#1}\xspace}

\newcommand{\U}{\mathsf{U}\xspace}
\renewcommand{\X}{\mathsf{X}\xspace}
\newcommand{\Xw}{\mathsf{X}_w\xspace}
\newcommand{\G}{\mathsf{G}\xspace}
\newcommand{\W}{\mathsf{W}\xspace}
\newcommand{\F}{\mathsf{F}\xspace}

\newcommand{\R}{\mathsf{R}\xspace}
\newcommand{\M}{\mathsf{M}\xspace}

\newcommand{\FuzzyDFA}{FuzzyA\xspace}
\newcommand{\NeSyA}{NeSyA\xspace}
\newcommand{\TILR}{T-ILR\xspace}
\newcommand{\LTLZinc}{LTLZinc\xspace}

\newcommand{\DiffLTLf}{\ensuremath{\partial}LTL\ensuremath{_f}\xspace}

\newcommand{\ocircarrow}{%
  \mathbin{
    \ooalign{
      \scalebox{0.78}{$\bigcirc$}\cr
      \hidewidth\raisebox{0.1ex}{$\mkern+.5mu\scriptscriptstyle\rightarrow$}\hidewidth\cr
    }
  }
}

\newcommand{\tand}{\otimes}
\newcommand{\tor}{\oplus}
\newcommand{\tneg}{\ominus}
\newcommand{\timpl}{\ocircarrow}%{\oslash}

\journal{Artificial Intelligence}
\begin{document}

\begin{frontmatter}



\title{%\DiffLTLf: Scalable Neurosymbolic Learning for LTLf via Direct Fuzzy Semantics
%
%An exploration into fuzzy semantics for temporal neuro-symbolic learning
%
%Time to Reason: Fuzzy Semantics for Temporal Neuro-Symbolic Learning
%
% Time to Reason: Characterizing Fuzzy Semantics for Scalable Temporal Neurosymbolic Integration
%
%Time to Reason: Characterizing Fuzzy Semantics for Direct Temporal Neurosymbolic Learning
%
%Time to Reason: Characterizing Fuzzy Semantics for Temporal Neurosymbolic Integration
%
% Time to Reason: Defining Fuzzy Semantics for Temporal Neurosymbolic Integration
%
% Time to Reason: Exploring Fuzzy Semantics for Temporal Neurosymbolic Integration
%
% Navigating LTLf Fuzzy Semantics for Scalable Neurosymbolic Learning
%
% Time to Reason: Fuzzy Semantics for Temporal Neurosymbolic Integration
%
% Time to Reason: Temporal Neurosymbolic Integration with Fuzzy Semantics 
% \DiffLTLf: Characterizing Fuzzy Semantics for Scalable Temporal Neurosymbolic Learning
%
%for LTLf and a Theoretical Study of Its Fuzzy Semantics
%
% \DiffLTLf: Scalable Temporal Neurosymbolic Learning and a Theoretical Study of Its Fuzzy Semantics
% Time to Reason: Scalable Neurosymbolic Learning for LTLf via Direct Fuzzy Semantics Integration
Time to Reason: Scalable Neurosymbolic Learning for LTLf via Fuzzy Semantics
}


\author[FBK,UNIBZ]{Riccardo Andreoni} %% Author name
\author[FBK]{Andrei Buliga} %% Author name
\author[UNIBZ]{Alessandro Daniele} %% Author name
\author[UNIBO]{Paolo Felli} %% Author name
\author[UNIBZ]{Chiara Ghidini} 
\author[UNIBZ]{Marco Montali} 
\author[FBK]{Massimiliano Ronzani} %% Author name

\affiliation[FBK]{organization={Fondazione Bruno Kessler},
            addressline={Via Sommarive, 18}, 
            city={Trento},
            postcode={38123}, 
            country={Italy}}  

\affiliation[UNIBZ]{organization={Free University of Bozen-Bolzano},
            addressline={Via Bruno Buozzi, 1}, 
            city={Bolzano},
            postcode={39100}, 
            country={Italy}}       

\affiliation[UNIBO]{organization={Università di Bologna},
            addressline={Via Mura Anteo Zamboni 7}, 
            city={Bologna},
            postcode={40126}, 
            country={Italy}}      
%% Abstract
\begin{abstract}
Neurosymbolic (NeSy) Artificial Intelligence  aims to integrate Deep Learning (DL) architectures with symbolic reasoning. 
While initial NeSy approaches have targeted mainly symbolic reasoning in propositional and first-order logics, recent works have started to address the construction of neurosymbolic frameworks for Temporal Logics, and in particular for \LTLF.
These approaches have established temporal NeSy as a promising research direction, laying the foundations for learning under temporal constraints.
Nonetheless, they leave many questions unanswered. 
From a theoretical perspective, several differentiable semantics for interpreting \LTLF have been proposed but have not yet been formally and systematically defined within a unified framework.
Moreover, existing approaches commonly rely on automata to represent temporal knowledge, resulting in limited scalability.
Motivated by this research gap, this paper provides the following contributions:
(i) formally defining different fuzzy semantics for \LTLF, and systematically analysing theoretical properties regarding equivalences and dualities of temporal operators; (ii) showing how these semantics can be directly integrated within a novel NeSy framework, called \DiffLTLf, enabling flexible and scalable learning without relying on the usage of automata; and (iii) introducing a novel evaluation protocol of increased complexity of learning tasks w.r.t.~existing benchmarks.

Our results show that the choice of fuzzy semantics has a significant impact on predictive performance. Moreover, \DiffLTLf achieves performance on par with, and sometimes superior to, state-of-the-art probabilistic approaches while substantially improving scalability. Taken together, these results establish direct fuzzy interpretations as a competitive and scalable alternative to existing temporal NeSy frameworks.

% \riccardo{
% Our theoretical results show that the classical \LTLF equivalences between temporal operators are preserved in full only under the Gödel semantics, whereas under Product and Łukasiewicz some operators lose their interdefinability and require a native definition.
% }

\end{abstract}


% \begin{highlights}
% \end{highlights}

% %% Keywords
% \begin{keyword}
% %% keywords here, in the form: keyword \sep keyword

% %% PACS codes here, in the form: \PACS code \sep code

% %% MSC codes here, in the form: \MSC code \sep code
% %% or \MSC[2008] code \sep code (2000 is the default)

% \end{keyword}

\end{frontmatter}



\section{Introduction}
\label{sec:intro}

Neurosymbolic (NeSy) Artificial Intelligence (AI) aims to integrate Deep Learning (DL) architectures with symbolic reasoning techniques~\cite{DBLP:series/faia/BesoldGBBDHKLLPPPZ21}. This integration aims to combine the strengths of deep neural networks and symbolic AI while overcoming their respective weaknesses, addressing important problems related to data efficiency, fairness, trust, and safety of AI. 

%This integration is motivated by the observation that purely data-%driven models often struggle to incorporate prior knowledge, %enforce logical constraints, and generalize reliably in structured %domains.

A number of NeSy approaches have been proposed in recent years, differing both in the symbolic formalisms they adopt and in how symbolic knowledge is incorporated into gradient-based learning. Concerning the former, most of the efforts have been focused on propositional and first-order logics~\cite{DBLP:conf/nips/ManhaeveDKDR18, ltn}. Concerning the latter, end-to-end training is typically enabled by relaxing symbolic representations into numerical forms. As a result, different interpretations, including fuzzy and probabilistic semantics, have been explored to align symbolic reasoning with the learning dynamics of neural models \cite{ltn, SBR, DBLP:conf/nips/ManhaeveDKDR18, DBLP:conf/icml/XuZFLB18}.

In the last few years the NeSy community has witnessed the emergence of a new series of works (see in particular \cite{ltlzincIJCAI, DBLP:conf/kr/UmiliCG23, DBLP:conf/ijcai/ManginasPR25,DBLP:conf/nesy/AndreoniBDGMR25}) that aim to build neurosymbolic frameworks for Temporal Logics, and in particular for Linear time Temporal Logic (\LTL)~\citep{Pnue77} and its variant \LTLF on finite traces~\cite{DBLP:conf/ijcai/GiacomoV13}. This is not surprising. In fact, \LTL and \LTLF are widely used across a range of domains—including formal methods \citep{Pnue77}, automated planning \citep{DeDM14}, AI-augmented process mining \citep{declare_handbook}, and reinforcement learning \citep{restraining_bolt}. 
	These new temporal NeSy formalisms are characterised by the fact that learning must account for \emph{sequential data} and \emph{time-dependent constraints}. 
	An important difference between these works concerns the way in which 
	the temporal knowledge is incorporated within the framework. The works of \citet{DBLP:conf/kr/UmiliCG23} and \citet{DBLP:conf/ijcai/ManginasPR25} rely on a finite-state machine (automata) transformation of the temporal formula that is used to define supervision signals over input sequences. The former uses fuzzy automata under the Product semantics, while the latter uses probabilistic automata. 
	 Instead, the work of \citet{DBLP:conf/nesy/AndreoniBDGMR25} interprets directly \LTLF specifications through fuzzy logic under the G\"{o}del semantics to guide the learning process. This is to overcome the computational overhead required to (1) build the automaton and (2) traverse it multiple times during the training procedure. 
	 
These works have a great merit, as they have extended NeSy frameworks to reference logics for dynamic domains. Nonetheless, they are still first, and somehow preliminary, proposals towards the solid definition of temporal NeSy frameworks. From the description above it is easy to see that they all exploit different differentiable semantics to interpret classical \LTL, and none of them takes into account the problem of investigating the impact that this choice has, both in terms of the relationship between the differentiable semantics and the classical one, and in terms of performances. This is a crucial aspect, especially for formalisms based on fuzzy logics, as prior work on propositional logics shows that the choice of fuzzy semantics can significantly affect the predictive performance within NeSy frameworks~\cite{DBLP:journals/ai/KriekenAH22}.
 	 
A second important merit of the works above is to have identified a clear task for the evaluation of these methods, that is weakly supervised symbol grounding (see Section~\ref{sec:task_overview}). Nonetheless, the works mentioned above provide limited attention to the issue of scalability. While this is understandable, as they are the first works in the field, scalability is an important property that needs to be considered and investigated when developing a NeSy framework. %Related to this, a recent benchmark that was proposed for this task~\cite{ltlzincIJCAI} works only for approaches that rely on the explicit construction of the automaton corresponding to the temporal formula and this can hamper the comparison of different types of systems.
Related to this, \citet{ltlzincIJCAI} recently proposed a benchmark for this task. However, it applies only to NeSy methods that compile temporal knowledge into automata, which hampers the assessment of systems relying on alternative integration approaches.

This work has the aim to overcome the limitations listed above: (i) generalizing the work of \citet{DonadelloFIMM25}, we provide a systematic study of fuzzy semantics in temporal NeSy learning, where
we show that the classical \LTLF equivalences between temporal operators are preserved in full only under the Gödel semantics, whereas under Product and Łukasiewicz some operators lose their interdefinability and require a native definition;
% (ii) together with theoretical insights into the interpretation of fuzzy truth values in temporal logics, clarifying how different semantics affect both predictive performance and the evaluation of temporal operators. 
% This study aims to provide the theoretical basis for the usage of fuzzy temporal logics in NeSy settings; 
(ii) we show how different fuzzy semantics can be directly integrated within a novel NeSy framework, called \DiffLTLf, enabling flexible learning without having to rely on the usage of automata; (iii) we take inspiration from \citet{ltlzincIJCAI} and introduce an extensive evaluation protocol for temporal NeSy, that works for both automata-based and automata-free frameworks. 
To start addressing the problem of scalability, we increase the complexity of learning scenarios w.r.t.~\cite{ltlzincIJCAI}. 
The evaluation shows that: (1) different fuzzy semantics affect predictive performance, and (2) the direct fuzzy interpretations of  \DiffLTLf overall achieve comparative performance with the state-of-the-art method proposed in \cite{DBLP:conf/ijcai/ManginasPR25} while significantly increasing scalability, thus confirming trends already exposed for first-order logics~\cite{DBLP:conf/nips/MaeneR23} where probabilistic approaches tend to improve predictive performance, while fuzzy logics tend to increase computational efficiency.

Besides this particular work, and the proposal of \DiffLTLf, we believe that the systematic theoretical investigation of fuzzy \LTLF in Section~\ref{sec:fuzzyLTL}, and the novel evaluation setting provided in Section~\ref{sec:evaluation} can provide a solid contribution to the development of future temporal NeSy frameworks, thus contributing to the consolidation of this field. 

The paper is structured as follows. 
% \todo[inline]{P: placeholder:}
%
In Section \ref{sec:background} we summarise the required preliminaries and background material; in Section \ref{sec:conceptual} we illustrate the state of the art in temporal NeSy, including available benchmarks, focusing on the most relevant approaches against which our experimental evaluation is carried out; in Section \ref{sec:fuzzyLTL} we introduce our fuzzy variant of the \LTLF temporal logic, formalising and studying three distinct semantics; in Section \ref{sec:diffLTLf} we illustrate the overall \DiffLTLf framework and formally state the NeSy task at hand; in Section \ref{sec:evaluation} we formulate critical research questions and carry out experimental evaluation, comparing \DiffLTLf against the existing approaches in the literature that were already identified; in Section \ref{sec:results} we present the results of our evaluation with respect to the research questions that were formulated; in Section \ref{sec:threat_validity} we underline some limitations of the current study. Conclusions are summarised in Section \ref{sec:conclusions}, and some technical implementation details and ancillary experimental results are given in appendix.
The \DiffLTLf implementation and the experiments are available at \href{https://github.com/andreoniriccardo/DiffLTLf}{github.com/andreoniriccardo/DiffLTLf}.
\section{Background}
\label{sec:background}
In this section we revise the background notions needed in the paper. We start by introducing the task of weakly supervised symbol grounding which we aim to solve with the proposed \DiffLTLf. We continue by revising the definitions of Linear Time Temporal logic both in its classical setting and in the setting of finite traces. For the sake of readability we have decided to omit a separate background section on propositional fuzzy logics. Instead, we introduce the key notions of fuzzy logics when they are needed in defining the fuzzy \LTLF variants in Section \ref{sec:fuzzyLTL}.    

\subsection{Weakly Supervised Symbol Grounding under Distant Temporal Supervision}
\label{sec:task_overview}

In this section, we revise the \textit{weakly supervised symbol grounding under distant temporal supervision} problem addressed by previous temporal neurosymbolic frameworks \cite{ltlzincIJCAI, DBLP:conf/kr/UmiliCG23, DBLP:conf/ijcai/ManginasPR25,DBLP:conf/nesy/AndreoniBDGMR25}. In this task, a model must discover the symbolic meaning of visual observations from sequences, supervised only by a binary label for each sequence, where the label expresses compliance with a temporal specification.

The setting is as follows: a model observes sequences of images, where each image represents one of a finite set of categories unknown to the model.
Each sequence is annotated with a single binary label indicating whether the sequence, interpreted at the symbolic level, satisfies or violates a given temporal specification (i.e., a rule that constrains the admissible ordering of symbols over time).
Crucially, the model has no access to the category of the individual images: it is never told which category any particular image belongs to.
The learning objective is to train a perception model that maps each raw image to a symbolic category, using only the sequence-level compliance label as supervision.

\begin{figure}[bt]
% \centering\includegraphics[trim={.7cm .4cm .7cm .4cm}, clip, width=1.\linewidth]{Figures/task_overview_industrial_v2.pdf}
\centering\includegraphics[trim={.6cm .4cm .6cm .4cm}, clip, width=1.\linewidth]{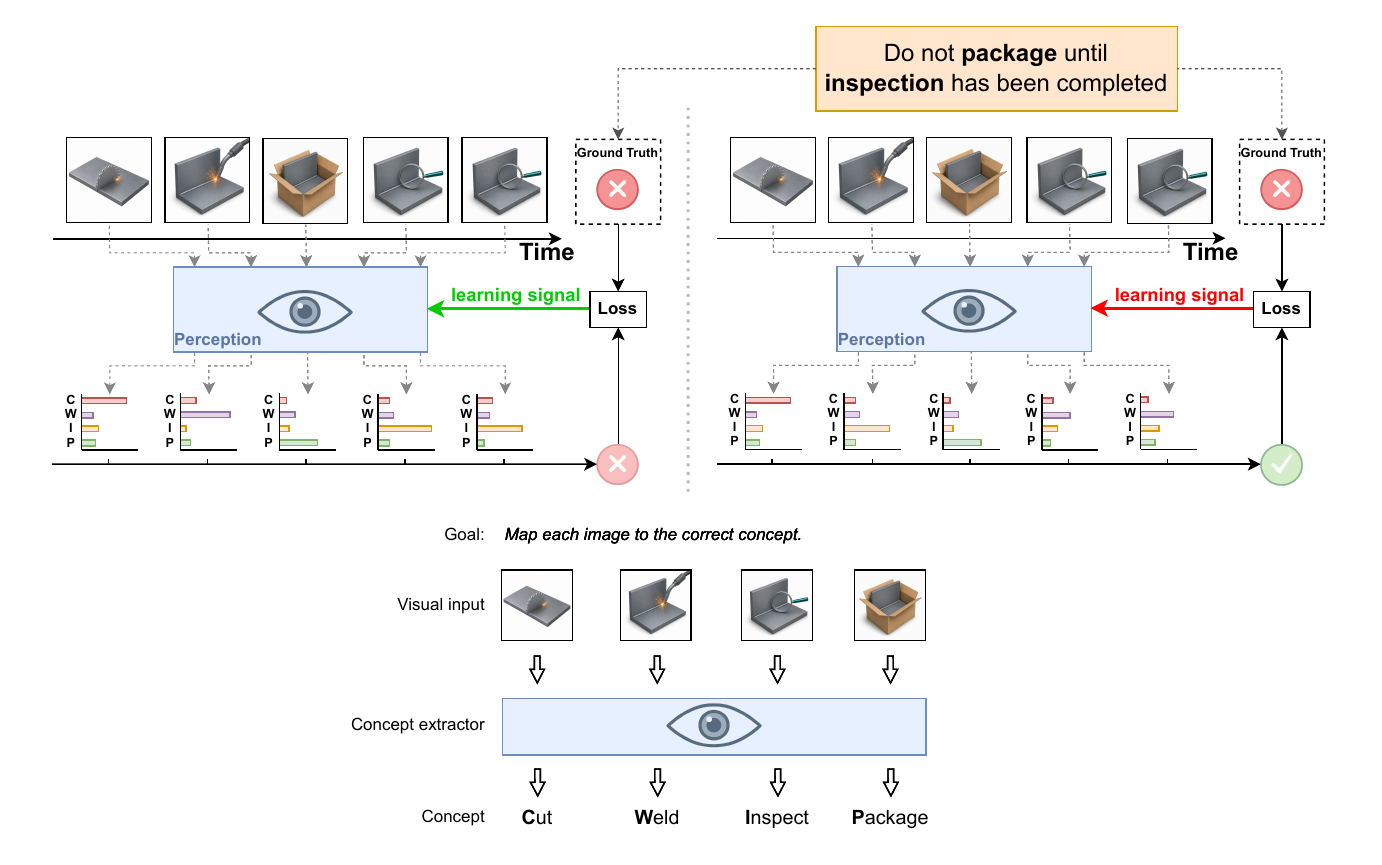}
  \caption{Illustrative example of the weakly supervised learning task on a simplified manufacturing scenario. Left: correct perception leads to a correct assessment of non-compliance, matching the ground-truth label. Right: incorrect perception leads to a wrong assessment of compliance, producing a corrective learning signal.}
  \label{fig:task_overview}
\end{figure}

Figure~\ref{fig:task_overview} 
illustrates the task's learning process on a simplified manufacturing example. 
This example involves four possible activities: cutting, welding, inspecting, and packaging. Moreover, the scenario at hand is subject to the temporal specification ``\textit{do not package until inspection has been completed'}'.
Both sides of the figure show the same input sequence, which violates this rule.
On the left hand side, the perception module correctly recognises each image, resulting in a correct classification of the sequence as non-compliant.
On the right hand side, the perception module misclassifies some images, resulting in an incorrect classification of the sequence as compliant. The mismatch with the ground-truth label produces a large loss, and the resulting learning signal pushes the perception module to correct its predictions (red arrow).
In both cases, the model receives only the binary compliance label, and the symbolic category of the individual images is never provided.

This example serves as a high-level illustration of the learning problem; the scenario and the images in Figure~\ref{fig:task_overview} are not part of the experimental evaluation. The formal definition of the task, including the temporal specification language and the proposed neurosymbolic framework, is presented in Sections \ref{sec:fuzzyLTL} and \ref{sec:diffLTLf} respectively. 
The experimental evaluation is detailed in Section~\ref{sec:evaluation}.

\subsection{Linear Temporal Logic on Finite Traces}
\label{sec:ltlf}

Linear-time logics, that is, temporal logics predicating on traces, provide the most natural choice to express symbolic knowledge in our setting. Traditionally, traces are assumed to have an infinite length, as witnessed by Linear Temporal Logic (\LTL)~\citep{Pnue77}. 
In several application domains, such as planning and process mining, the dynamics of the system are more naturally captured using unbounded, but \emph{finite}, traces \cite{DeDM14}. This led to \emph{\LTL on finite traces} (\LTLF)~\citep{DBLP:conf/ijcai/GiacomoV13}. 

An \LTLF formula $\varphi$ over a finite set $\P$ of propositional symbols follows the grammar \cite{DBLP:conf/ijcai/GiacomoV13}:
%\begin{equation*}
%\varphi  ::=  \p \mid \lnot \varphi \mid %\varphi_1 \lor \varphi_2 \mid \tomorrow %\varphi \mid \varphi_1 \until \varphi_2, %\text{ where $\p \in \P$.}
%\end{equation*}
\begin{equation*}
\varphi  ::=  
%\top 
%\mid \bot 
%\mid 
\p 
\mid \neg \varphi \mid \varphi \lor \varphi 
\mid \X \varphi 
\mid \varphi \U \varphi
\end{equation*}
where $\p \in \P$. 
%and $\top$ denotes the \emph{true} value. 

The semantics of these formulae is  defined over \emph{finite traces}, usually formalised as non-empty sequences $\lambda = \langle A_1,A_2 \ldots,A_n \rangle$ where each $A_i\in 2^{\P}$ for $i> 0$ is the set of propositional symbols that are true at instant $i$, namely is a propositional assignment over $\P$. For uniformity with the fuzzy setting discussed later, and in particular with the definition of traces used in Section~\ref{sec:fuzzyLTL}, we adopt a functional, equivalent definition, where 
%However, in this paper we adopt an equivalent representation for uniformity with the definition of traces adopted in later sections, when the fuzzy-counterpart of \LTLF is introduced. 
%
a trace is represented as a non-empty series of instants, each mapping \emph{every} propositional symbol from $\mathcal{P}$ to either $0$ (for \emph{false}) or $1$ (for \emph{true}).  

%
%As we are dealing with traces of finite length, we denote the last instance of a trace as $last(\lambda)$.

\begin{definition}\label{def:traces_ltlf}
A \emph{trace} of length $n$ is a non-empty vector $\lambda=\langle \lambda_0,\ldots,\lambda_{n-1}\rangle$ of functions in which, for each  instant $i = 0,\ldots,n{-}1$, the $i$-th function $\lambda_i : P \mapsto \{0,1\}$ assigns a truth value to each symbol $p\in \mathcal{P}$. Given $i \in \{0,\ldots,n-1\}$ and $p \in P$, $\lambda_i(p)$ is equivalently denoted by $\lambda(p,i)$. 
%
%For brevity, we denote by $\lambda(p,i)$ the value associated to $p$ at instant $i$, namely $\lambda_i(p)$. Moreover, in Section \ref{sec:diffLTLf} and following we write  $\lambda(p)=\langle p_0, p_1, \ldots p_{n-1}\rangle$ to denote the sequence of all values of $p$ along the trace, namely the sequence of values $p_i=\lambda(p,i)$.   
\end{definition}
Notably, in this paper we consider only non-empty traces of finite length, and denote the last instant of a trace $\lambda$ of length $n$ by  $last(\lambda)=n{-}1$. 

\begin{figure}[t]
\centering
\noindent
\resizebox{\textwidth}{!}{
\newcommand{\cellw}{18ex}
\begin{tikzpicture}[
  node distance=0.8cm and 1.3cm,
  bend angle=45,
  font=\sffamily, 
  remember picture,
  auto,
]

\matrix[
  trace-table, 
  text width = \cellw,
] (l) {
 $i=0$ \& $i=1$ \& $i=2$ \& $i=3$ \& $i=4$\\
\fuzzyval{1} \& \fuzzyval{1} \& \fuzzyval{0} \& \fuzzyval{1} \& \fuzzyval{1} \\
\fuzzyval{0} \& \fuzzyval{0} \& \fuzzyval{1} \& \fuzzyval{1} \& \fuzzyval{0} \\
};
\node[
    right=0mm of l-1-1.west,
    anchor=east,
] {$\lambda$};
\node[
    right=0mm of l-2-1.west,
    anchor=east,
] {$a$};
\node[
    right=0mm of l-3-1.west,
    anchor=east,
] {$b$};

\matrix[
  eval-table, 
  below= 5mm of l,
  text width = \cellw,
] (u) {
\fuzzyval{1} \& \fuzzyval{1} \& \fuzzyval{1} \& \fuzzyval{1} \& \fuzzyval{0} \\
\fuzzyval{0} \& \fuzzyval{1} \& \fuzzyval{1} \& \fuzzyval{0} \& \fuzzyval{0} \\
};
\node[
    right=0mm of u-1-1.west,
    anchor=east,
] {$a \U b$};
\node[
    right=0mm of u-2-1.west,
    anchor=east, 
] {$\X b$};
\end{tikzpicture}
}
 \caption{Graphical depiction of an \LTLF trace of length 5, and satisfaction of the two formulae $a \U b$ and $\X b$ in every instant of the trace. }
    \label{fig:ltlf-example}
\end{figure}

\begin{example}
\label{ex:ltlf-trace}
    Figure~\ref{fig:ltlf-example} (top) graphically depicts an \LTLF trace $\lambda$ of length 5 defined over $\P = \set{a,b}$. We have $last(\lambda) = 4$. Also, $\lambda(\fluent{a},0) = 1$ is the truth value associated to  symbol $\fluent{a}$ at instant $0$ (the beginning of the trace), while $\lambda(\fluent{a},1) = 1$ is the value of $\fluent{a}$ in the following instant. 
\end{example}

% \todo{Transform this sentence into an example, possibly using also the graphical notation, together with the evaluation of a formula (to be discussed in a later example).} For instance, $\lambda(\fluent{a},0)$ is the truth value associated to  symbol $\fluent{a}$ at instant $0$ in the trace $\lambda$, namely at the beginning of the trace, while $\lambda(\fluent{a},1)$ is the value of $\fluent{a}$ in the following instant. 
%

% A \emph{trace} \lambda over $\P$ is a finite, non-empty sequence $\lambda = \langle \lambda_1,\lambda_2 \ldots,\lambda_n \rangle$ where each $\lambda_i, i> 0$ is a propositional assignment (in symbols $\lambda_i\in 2^{\P}$) indicating which propositional atoms from $\P$ are true at instant $i$ in the trace. The length $n$ of $\lambda$ is denoted $\length(\lambda)$.
 
The grammar above extends propositional logic on $\P$ with formulae employing the temporal operators $\X$ and $\U$, respectively representing \emph{(strong) next} and \emph{(strong) until}. Intuitively, when evaluated in an instant of the trace: $\X \varphi$ states that there exists a next instant, and $\varphi$ holds therein; $\varphi_1 \U \varphi_2$ states that $\varphi_2$ holds in the current or a later instant, and in all instants in between, $\varphi_1$ holds.

 %For $i<\length(\lambda)$, $\lambda(\colon i)$ is the prefix $\lambda(0),\lambda(1)\ldots,\lambda(i)$ of $\lambda$. 

Formally, for an \LTLF formula $\varphi$, a trace $\lambda$, and an instant $i \in \set{0,\ldots,last(\lambda)}$, we inductively define that $\varphi$ is true in instant $i$ of $\lambda$, written $\lambda,i\models \varphi$, as  \cite{DBLP:conf/ijcai/GiacomoV13}: 
\[
\begin{array}{l l l}
\lambda,i \models  \p\in \P & \text{if\;\;} & \lambda(p,i)=1 \\
\lambda,i \models  \lnot \varphi & \text{if} & \lambda,i \not \models \varphi\\
\lambda,i \models \varphi_1 \lor \varphi_2 & \text{if} & \lambda,i \models \varphi_1 \text{ or }  \lambda,i \models \varphi_2\\
\lambda,i \models \X \varphi &
\text{if} & i < last(\lambda) \text{ and } \lambda,i{+}1 \models \varphi \\
\lambda,i \models \varphi_1 \U \varphi_2 
&
\text{if } & \lambda,j \models \varphi_2 \text{ for some } j \text{ s.t.~} i \leq j \leq last(\lambda) 
\text{ and } \\
&  & \qquad \lambda,k \models \varphi_1 \text{ for every } k \text{ s.t.~} i \leq k < j
\end{array}
\]
%
%\noindent 
We say that $\lambda$ satisfies $\varphi$, written $\lambda \models \varphi$, if $\lambda,0 \models \varphi$. 

\begin{example}
\label{ex:ltlf-semantics}
Figure~\ref{fig:ltlf-example} (bottom) graphically depicts, instant by instant, the truth values of the two \LTLF formulae $a \U b$ and $\X b$, evaluated over the trace $\lambda$ from Example~\ref{ex:ltlf-trace} (and shown in the top part of the same figure). 

Formula $\X b$ essentially shifts of one instant ahead the truth values of $b$ as assigned by $\lambda$, and is false in instant $4 = last(\lambda)$ since there is no next instant in $\lambda$.

Formula $a \U b$ evaluates to false in instant $4$: while the left-hand side holds therein, there is no later instant in $\lambda$ where $b$ holds. The formula instead evaluates to true in all other instants: in instants $2$ and $3$ because the right-hand side holds, and in instants $0$ and $1$ because the left-hand side holds there and in later instants until, in instant $3$, the right-hand side holds. The truth values at instant $0$ indicate the overall verdict for the entire trace: $\lambda \models a \U b$, and $\lambda \not \models \X b$.
\end{example}

The syntax and semantics of the other standard boolean connectives $\top$, $\bot$, $\land$, $\limp$ are derived as usual. Further key temporal operators are derived from $\tomorrow$ and $\until$ as follows.
\begin{itemize}
\item \emph{weak next} ($\Xw$): $\Xw \varphi \equiv \X \varphi \lor \neg \X \top$, capturing that if there is a next instant, then $\varphi$ holds therein;
\item \emph{release} ($\release$): $\varphi_1 \release \varphi_2 \equiv \neg (\neg \varphi_1 \until \neg \varphi_2)$, capturing that $\varphi_2$ holds until, and including the moment, where $\varphi_1$ holds (if this condition never occurs, then $\varphi_2$ holds throughout the entire trace);
\item \emph{globally} ($\always$): $\always \varphi \equiv \bot \release \varphi$, capturing that $\varphi$ holds throughout the entire trace;
\item \emph{eventually} ($\eventually$): $\eventually \varphi \equiv \top \until \varphi$, capturing that $\varphi$ holds at some instant in the trace.
\item \emph{strong release} ($\mrelease$): $\varphi \mrelease \psi \equiv \varphi \release \psi \land \eventually \varphi$, capturing a stronger form of \emph{release} where $\varphi$ holds at some point.
\item \emph{weak until} ($\W$): $\varphi \W \psi \equiv \varphi \U \psi \lor \G \varphi$, capturing a weaker form of \emph{until} where $\psi$ is not required to hold at some point and, in that case, $\varphi$ must globally hold. 
\end{itemize}

The process mining community has selected a number of patterns of \LTLF formulas that are particularly significant for describing business processes in a declarative manner. These patterns constitute the \declare modelling language~\citep{declare_handbook}. Examples of \declare patterns are provided in Table~\ref{tab:declare} together with their \LTLF formalisation and intuitive meaning. Following \cite{DBLP:conf/kr/UmiliCG23,DBLP:conf/nesy/AndreoniBDGMR25,DBLP:conf/ecai/UmiliC24}, we use \declare formulae in the evaluation of our approach in Section~\ref{sec:evaluation}.
\begin{table}[tb]
\centering
\caption{\declare pattern templates used in the benchmark. $A$ and $B$ denote stream-specific propositions or Boolean combinations of them.}
\label{tab:declare}
\begin{tabularx}{\columnwidth}{llX}
\toprule
\declare Pattern & \LTLF Template & Description \\
\midrule
Response            & $\always(A \rightarrow \eventually\, B)$            & If $A$ occurs, $B$ must eventually follow \\
Not Succession      & $\always(A \rightarrow \neg\eventually\, B)$        & If $A$ occurs, $B$ must never follow \\
Chain Succession    & $\always(A \rightarrow \tomorrow\, B)$            & If $A$ occurs, $B$ must hold at the next timestep \\
Not Chain Succession & $\always(A \rightarrow \neg\tomorrow\, B)$       & If $A$ occurs, $B$ must not hold at the next timestep \\
Precedence          & $(\neg B \;\until\; A) \lor \always(\neg B)$                              & $B$ cannot occur before $A$ has occurred \\
Chain Precedence    & $\always(\tomorrow\, B \rightarrow A)$            & Every occurrence of $B$ is immediately preceded by $A$ \\
Responded Existence & $\eventually\, A \rightarrow \eventually\, B$            & If $A$ occurs at some point, $B$ must also occur \\
Not Co-existence    & $\eventually\, A \rightarrow \neg\eventually\, B$        & $A$ and $B$ cannot both occur in the sequence \\
\bottomrule
\end{tabularx}
\end{table}

\section{State of the Art in Temporal NeSy}
\label{sec:conceptual}
The majority of works in the NeSy community focus on static, relational domains~\cite{DBLP:series/faia/BesoldGBBDHKLLPPPZ21}. Nonetheless recent efforts have started proposing solutions for NeSy learning over temporal domains, dealing with different types of formalisms, architectures, and evaluation protocols.
In this section, we describe the main existing frameworks with the goal of highlighting their contributions and their limitations. The limitations illustrated here motivate our overall work and the definition of \DiffLTLf.   

The works we take into account here are \FuzzyDFA~\cite{DBLP:conf/kr/UmiliCG23}\footnote{\citet{DBLP:conf/kr/UmiliCG23} do not introduce a name for their system. We use here the same name used by \citet{DBLP:conf/ijcai/ManginasPR25} to refer to the system in \cite{DBLP:conf/kr/UmiliCG23}.}, \NeSyA~\cite{DBLP:conf/ijcai/ManginasPR25}, \TILR~\cite{DBLP:conf/nesy/AndreoniBDGMR25}, and the \LTLZinc benchmark~\cite{ltlzincIJCAI}. We focus on these papers because they provide the most recent and impactful contributions.\footnote{Early work that combine temporal logics and neural systems operate on symbolic inputs, assuming propositional truth values rather than grounding them in perceptual data. They then compile temporal (and epistemic) logics into the architecture and weights of neural networks for logical deduction or runtime monitoring~\cite{NIPS2003_34766559,6889961}. Another, more recent, work~\cite{an_eye_into} exploits \LTLF for suffix prediction under \LTLF constraints but does not rely on the direct, differentiable integration of temporal logic into NeSy architectures. For these reasons, they are not analysed in this section. The works of~\citet{DBLP:conf/nesy/UmiliABC23} and \citet{DBLP:conf/pmai/UmiliLP24}, which address the tasks of non-Markovian reinforcement learning and  suffix prediction under \LTLF constraints, respectively, are instead not analysed here as they follow the same conceptual approach of \FuzzyDFA~\cite{DBLP:conf/kr/UmiliCG23}.}

We first review the three frameworks of \FuzzyDFA~\cite{DBLP:conf/kr/UmiliCG23}, \NeSyA~\cite{DBLP:conf/ijcai/ManginasPR25}, and \TILR~\cite{DBLP:conf/nesy/AndreoniBDGMR25}, and we conclude the section with the \LTLZinc benchmark~\cite{ltlzincIJCAI}.

\subsection{NeSy Frameworks for Temporal Logic}
\label{sec:sota-temporal}
A first important observation is that all these recent efforts focus on a particular variant of Temporal Logics, that is, Linear time Temporal Logics on finite traces (\LTLF). Having said so, they solve the task of Weakly Supervised Symbol Grounding under Distant Temporal Supervision in different manners. We can make explicit these differences along two main dimensions: (i) how the temporal knowledge is integrated into the learning process, and (ii) which differentiable semantics is used to interpret the temporal formula. Let us explore these two dimensions one by one.

\paragraph{Integration of Knowledge}
Concerning the way in which temporal knowledge is integrated into the learning process, the dominant approach compiles the \LTLF specification into a deterministic finite-state automaton (DFA), which is then evaluated in a differentiable manner over the perception outputs. The two most representative instances of this approach are \FuzzyDFA \cite{DBLP:conf/kr/UmiliCG23} and \NeSyA \cite{DBLP:conf/ijcai/ManginasPR25}. In these frameworks, at each timestep, a neural perception module produces soft truth values for the atomic propositions, which are used to evaluate the automaton's transition guards.
The guard values then update the truth values associated with each automaton state, recursively from the first to the last instant of the sequence.

The explicit construction of a finite-state automaton from the \LTLF formula originates two distinct limitations:
% Generally, automata-based approaches suffer from two limitations.
first, the automaton must be built before training, and its size is, in the worst case, double-exponential in the size of the \LTLF formula~\cite{DBLP:conf/ijcai/GiacomoV13}.
Second, during training the evaluation of the state of the automaton is computed recurrently over the sequence, each timestep depending on the previous one, so the evaluation is inherently sequential, preventing parallelisation over time. 

To overcome these limitations, \TILR \cite{DBLP:conf/nesy/AndreoniBDGMR25} operates directly on the logical structure of the formula, foregoing the need for an intermediate DFA-based representation.
In particular, \TILR proposes a differentiable framework in which \LTLF formulas are evaluated through fuzzy semantics without compiling them into an intermediate automaton. This allows the model to compute a degree of satisfaction directly from the neural outputs, while still enabling gradient-based optimization. 

In this work we follow the approach of \TILR, and base our proposal \DiffLTLf on this latter approach of evaluating directly \LTLF formulas through a differentiable semantics without compiling them into an intermediate automaton.

\paragraph{Differentiable Semantics}
Concerning the differentiable semantics used to interpret the temporal formula, we can observe a heterogeneous situation. In fact, \FuzzyDFA,  \TILR, and \NeSyA adopt three different semantics. 

More in detail, \FuzzyDFA adopts a fuzzy interpretation of the automaton, and exploits the \textit{Product semantics} for the fuzzy connectives. The authors of \FuzzyDFA claim to be inspired, in their work, by Logic Tensor Networks \cite{ltn}, and the choice of Product semantics appears to be grounded on this inspiration. 
Similarly to \FuzzyDFA, \TILR chooses to rely on a fuzzy approach, but opts for the \textit{G\"{o}del semantics}. This choice is justified in terms of the theoretical investigation of fuzzy \LTLF provided in \citet{DonadelloFIMM25}. 
Finally, \NeSyA replaces fuzzy semantics with \textit{probabilistic reasoning}. The probabilistic semantics is exact and avoids the approximation introduced by fuzzy evaluations.
% It comes, however, at a computational cost, as Weighted Model Counting must be carried out for every transition of the automaton at every timestep.\todo{@Ale: inserire footnote Weighted Model Counting per NesYA}
It comes, however, at a computational cost, as exact probabilistic inference requires solving a Weighted Model Counting (WMC) problem.\footnote{In practice, WMC is often evaluated through compiled probabilistic circuits rather than explicit model enumeration, which can enable efficient inference when compact compiled representations are available.}

The divide between probabilistic and fuzzy approaches is something to be expected, and mimics what has happened for the more established NeSy frameworks for static, relational domains, that is, domains where knowledge is formalised using propositional or first-order logics. 

Nonetheless, we can observe a lack of systematic investigation of the choices one has and of the impact of these choices on the resulting framework. This is a crucial limitation both for formalisms based on fuzzy logics, and for the choice between probabilistic and fuzzy logics. In fact, prior work on NeSy frameworks for Propositional and First Order  Logics shows that the choice of fuzzy semantics can significantly affect the predictive performance within NeSy frameworks~\cite{DBLP:journals/ai/KriekenAH22}, and a trade off exists between probabilistic and fuzzy approaches, where probabilistic approaches tend to improve the predictive performance, while fuzzy logics tend to improve (that is, decrease) the computational cost.

In this work, we aim for scalability and therefore opt for the fuzzy approach. Nonetheless we do it in a principled manner, first by providing theoretical investigation of different fuzzy semantics for \LTLF obtained using different t-norms, and then by defining  \DiffLTLf as a semantics-agnostic, fuzzy framework for temporal data where one can opt for different fuzzy semantics. Moreover, we aim at providing a first comprehensive and systematic evaluation of the different fuzzy and probabilistic semantics in terms of predictive performance and computational cost, thus obtaining systematic insights on the impact that the different differentiable semantics can have on the task. 

\subsection{Temporal NeSy benchmarks}
It is worth noting that, despite their architectural differences, the state-of-the-art methods described above are all evaluated on the task introduced in Section~\ref{sec:task_overview}, where supervision is provided only at the sequence level.

A recent benchmark related to this task is that of \LTLZinc~\cite{ltlzincIJCAI}, which targets sequence classification under temporal specifications. While the proposal of a benchmark is extremely important in building a solid research community, we can observe several limitations in this work when we aim to use it for  comparing the different aspects of temporal NeSy methods under purely distant supervision.
First, its specifications interleave temporal operators with relational constraints, so that the contribution of the temporal reasoning component is not isolated.
Second, it introduces intermediate supervision signals at several levels of the pipeline, together with a supervised pre-training of the perception module. While this facilitates training, it makes the benchmark less suitable for assessing how methods perform in more challenging (and possibly more realistic) scenarios when only distant supervision is available.
Third, it relies on the supervision on the next state in the finite-state machine. This makes it applicable only to automata-based methods, excluding approaches that evaluate the temporal specification directly.

These limitations motivate the evaluation protocol we introduce in Section \ref{sec:evaluation}, which relies only on the sequence-level label and applies to both automata-based and automata-free methods. Moreover, our evaluation setting disentangles two distinct critical aspects that affect scalability, by varying the size of the temporal specification and the length of the input sequences in a separate manner.

\section{Fuzzy LTLF}
\label{sec:fuzzyLTL}
%% Labels are used to cross-reference an item using \ref command.

In this section we introduce the syntax and the semantics of the temporal fuzzy logic that we use to formalise specifications. 
To define the semantics, as customary for t-norm fuzzy logics, one has to choose how to evaluate the logical connectives, namely conjunction, disjunction, negation and implication, as this is not fixed univocally. 
Instead of committing to a specific choice of evaluation, in this work we consider three alternatives, in turn obtaining three fuzzy variants of \LTLF, which we term \FLTLF[G], \FLTLF[P], and \FLTLF[L]. \FLTLF[G] corresponds to the logic previously introduced in \cite{DonadelloFIMM25}, that is, fuzzy \LTLF with the Gödel semantics; \FLTLF[P] and \FLTLF[L]  adopt, respectively, a variant of the Product semantics and the Łukasiewicz semantics.  
For convenience of presentation, we denote this family of logics as \FLTLF. 

The remainder of this section is organised as follows: 
in Section~\ref{subsec:syntax_semantics} we provide the syntax and semantics of the three logics, focussing on the core temporal operators $\X$ (next) and $\U$ (until). In  Section~\ref{subsec:equivalences} we consider additional temporal operators that are well known in the literature on temporal logic, and that in the boolean case are derivable from $\X$ and $\U$. In particular, we investigate whether the equivalences known to hold for \LTLF hold for our case as well. We show that, irrespective of the semantics chosen, some of these operators can be defined from the two fundamental temporal operators $\X$ and $\U$ introduced in Section~\ref{subsec:syntax_semantics}, while others cannot be derived as abbreviations, and thus their semantics needs to be defined natively.

\subsection{\texorpdfstring{\FLTLF}{FLTLf} Syntax and Semantics}
\label{subsec:syntax_semantics}

The logics we consider are finite-trace variants of the known Fuzzy Linear-time Temporal Logic (\FLTL)~\cite{LK00,FrigeriPS14} that is interpreted over finite traces, and thus can be seen as fuzzy counterparts of \LTLF (see Section~\ref{sec:ltlf} for preliminaries on \LTLF).

Given a finite set $\mathcal{P}$ of propositional symbols, an \FLTLF formula $\varphi$ is written according to the following grammar:
\[
\varphi := \top ~|~ \bot ~|~ p ~|~ \neg \varphi ~|~ \varphi \land \varphi ~|~ \varphi \lor \varphi ~|~ \varphi \rightarrow \varphi ~|~ \X \varphi ~|~ \varphi \U \varphi
\]
where $p\in \mathcal{P}$, the symbols $\top$ and $\bot$ are the \emph{true} and \emph{false} constants, respectively, and $\neg$, $\land$, $\lor$, $\rightarrow$ are fuzzy connectives for negation, conjunction, disjunction, implication.  
As in \LTLF, the fundamental temporal operators are $\X$ (next) and $\U$ (until). 

The semantics of these formulae is defined again over finite traces over propositional symbols $\P$.  
However, whereas the sequences in  Section~\ref{sec:ltlf} assigned only crisp values, namely either $0$ or $1$, the traces used for \FLTLF are non-empty sequences of instants in which a \emph{real value in the interval} $[0,1]$ is assigned to each propositional symbol $p \in \mathcal{P}$, representing the \emph{degree of truth} of $p$ at each instant.  

\begin{definition}\label{def:traces}
A \emph{trace} of length $n$ is a vector $\lambda=\langle \lambda_0,\ldots,\lambda_{n-1}\rangle$ of functions in which, for each  instant $i = 0,\ldots,n{-}1$, the $i$-th function $\lambda_i : P \mapsto [0,1]$ assigns to each symbol $p\in \mathcal{P}$ a real value. 
For brevity, we denote by $\lambda(p,i)$ the value associated to $p$ at instant $i$, namely $\lambda_i(p)$. Moreover, in Section \ref{sec:diffLTLf} and following we write  $\lambda(p)=\langle p_0, p_1, \ldots, p_{n-1}\rangle$ to denote the sequence of all values of $p$ along the trace, namely the sequence of values $p_i=\lambda(p,i)$.   
\end{definition}

%
%For instance, $\lambda(\fluent{a},0)$ is the value associated to  symbol $\fluent{a}$ at instant $0$ in the trace $\lambda$, namely at the beginning of the trace, while $\lambda(\fluent{a},1)$ is the value of $\fluent{a}$ in the following instant. 
%
We restrict to non-empty traces of finite length, thus continue to denote the last instant of a trace $\lambda$ of length $n$ by  $last(\lambda)=n{-}1$.

As a consequence, differently from \LTLF, the evaluation of an \FLTLF formula $\varphi$ over a trace $\lambda$ at instant $i\geq 0$, denoted by $v(\varphi,\lambda,i)$, returns a real value in $[0,1]$ rather than necessarily a boolean value in $\{0,1\}$. 
This is defined below, where $p\in \mathcal{P}$ and $\varphi_1,\varphi_2$ are \FLTLF formulae:  
\[
\begin{array}{r c l}
v(\top,\lambda,i) & = & 1 \text{ and } v(\bot,\lambda,i)  =  0; \\ 
v(p,\lambda,i) & = & \lambda(p,i); \\
v(\neg \varphi,\lambda,i) & = & \tneg v(\varphi,\lambda,i) ;\\
v(\varphi_1 \land \varphi_2,\lambda,i) & = & v(\varphi_1,\lambda,i)~\tand~ v(\varphi_2,\lambda,i); \\
v(\varphi_1 \lor \varphi_2,\lambda,i) & = & v(\varphi_1,\lambda,i)~\tor~ v(\varphi_2,\lambda,i) ;\\
v(\varphi_1 \rightarrow \varphi_2,\lambda,i) & = &  v(\varphi_1,\lambda,i)~\timpl~ v(\varphi_2,\lambda,i) ; \\
v(\X \varphi,\lambda,i) & = & v(\varphi,\lambda,i{+}1) \text{ if } i<last(\lambda) \text{, } 0 \text{ otherwise}; \\
v(\varphi_1 \U \varphi_2,\lambda,i) & = &  v(\varphi_2,\lambda,i) ~\tor~ (v(\varphi_1,\lambda,i)~\tand~  v(\X(\varphi_1 \U \varphi_2),\lambda,i) ).
\end{array}
\]

\noindent
where $\tand$, $\tor$, $\tneg$, $\timpl$ are the chosen t-norm, t-conorm (or s-norm), negation function and implication function, respectively. 
As discussed in the remainder of this paper, this is indeed one crucial distinction between \LTLF and this fuzzy logic: the semantics of propositional connectives is explicitly given. 

It is also worth noting that the semantics of $\U$ (until) is here given using the alternative, equivalent recursive formulation based on the expansion law for the until operator \cite{BaiK08}, which in the propositional setting has the form $\varphi_1 \U \varphi_2 = \varphi_2 \lor (\varphi_1 \land \X(\varphi_1 \U \varphi_2))$. This makes explicit the role of the chosen t-norm (conjunction) and t-conorm (disjunction) in its semantics and thus in the evaluation.

We say that an \FLTLF formula $\varphi$ \emph{has value $k$} in a trace $\lambda$ at instant $i\geq 0$ iff $v(\varphi,\lambda,i)=k$.    
Similarly, we say that $\varphi$ has value $k$ in $\lambda$ iff $v(\varphi,\lambda,0)=k$, also written   $v(\varphi,\lambda)=k$. 
Finally, we say that two \FLTLF formulae $\varphi_1,\varphi_2$ are equivalent, written $\varphi_1\equiv\varphi_2$, if and only if they have the same value in every possible trace.

%Given a \FLTLF formula $\varphi$ and a trace $\lambda$, the verification task is to determine the value of $\varphi$ in $\lambda$, namely $v(\varphi,\lambda)$. 
%%

\medskip
Intuitively, the value of a propositional symbol $p\in\mathcal{P}$ at instant $i$ is simply the value assigned to $p$ by the trace, while boolean connectives are computed according to the chosen t-norm, t-conorm, negation and implication functions. 
A formula of the form $\X\varphi$ is evaluated at $i$ on the next instant $i{+}1$ of the trace, and if this does not exist, then the value is $0$. Instead, $\varphi_1 \U \varphi_2$ expresses that $\varphi_2$ eventually contributes positively to the value of the entire formula which is until then supported by the value of $\varphi_1$. 
However, as customary for t-norm fuzzy logics, it is evident that the precise semantics depends on the choice of connectives, which is not fixed univocally.

We consider the following three distinct semantics, thus a family of three resulting logics. 

\paragraph{{\FLTLF[G]}: G\"odel semantics} We fix the interpretation for the generic connectives in the semantics above to match the standard Zadeh logic. Namely, for any values $\alpha,\beta\in[0,1]$:

\begin{multicols}{2}
\begin{itemize}[itemsep=0pt,parsep=3pt,topsep=0pt]
\item $\alpha \tand \beta = \min\{ \alpha, \beta \}$	
\item $\alpha \tor \beta = \max\{ \alpha, \beta \}$
\item $\tneg \alpha = 1 - \alpha$
\item $\alpha \timpl \beta = (\tneg \alpha) \tor \beta$	
\end{itemize}
\end{multicols}

% \InlineNode[options]{name}{text}
%\NewDocumentCommand{\inlinenode}{O{} m m}{%
%  \subnode[inner sep=0pt,outer sep=0pt,#1]{#2}{#3}%
%\subnode[fill=cyan]{c}{ccccc}
%}

\newcommand{\uformula}{
\ensuremath{a \U b}
}

\newcommand{\nuformula}{
\ensuremath{\X(\uformula)}
}

\newcommand{\aformula}{
\ensuremath{a \land \nuformula}
}

\newcommand{\nucolorzero}{Goldenrod}
\newcommand{\nucolorone}{Dandelion}
\newcommand{\nucolortwo}{BurntOrange}

\newcommand{\acolorzero}{SpringGreen}
\newcommand{\acolorone}{LimeGreen}
\newcommand{\acolortwo}{JungleGreen}

\newcommand{\ucolorzero}{Lavender}
\newcommand{\ucolorone}{Orchid}
\newcommand{\ucolortwo}{Plum}

\newcommand{\atomcolor}{gray!50}

% \tikzset{
%   every matrix/.append style={
%     every cell/.style 2 args={
%       /utils/exec={
%         \ifnum##2>1\relax
%           \tikzset{nodes={text width = 1.5em}}
%         \fi
%       }
%     }
%   }
% }

\begin{figure}[t]
\centering

\noindent
\hspace*{-0.03\textwidth}%
\resizebox{1.05\textwidth}{!}{

\newcommand{\cellw}{42ex}
\begin{tikzpicture}[
  node distance=0.8cm and 1.3cm,
  bend angle=45,
  font=\sffamily, 
  remember picture,
  auto,
]

\matrix[
  trace-table, 
  text width = \cellw,
] (l) {
 $i=0$ \& $i=1$ \& $i=2$ \\
\fuzzyval{0.5} \& \fuzzyval{0.9} \& \fuzzyval{0.2} \\
\fuzzyval{0.2} \& \fuzzyval{0.3} \& \fuzzyval{0.8} \\
};
\node[
    right=0mm of l-1-1.west,
    anchor=east,
] {$\lambda$};
\node[
    right=0mm of l-2-1.west,
    anchor=east,
] {$a$};
\node[
    right=0mm of l-3-1.west,
    anchor=east,
] {$b$};

\matrix[
  eval-table, 
  below= 0mm of l,
  text width = \cellw,
] (u) {
    %|[fill=\nucolorzero]|
    $
    \formulabox{f0}{v(\uformula,0)}{Goldenrod} =
     \formulabox{b0}{v(b,0)}{\atomcolor}
     \tor 
     (  \formulabox{a0}{v(a,0)}{\atomcolor}
        \tand 
        \formulabox{u1}{v(\X(\uformula), 0)}{Dandelion}
    )
    $
  \&
   $\formulabox{f1}{v(\uformula, 1)}{Dandelion} =
   \formulabox{b1}{v(b,1)}{\atomcolor}
     \tor 
     (  \formulabox{a1}{v(a,1)}{\atomcolor}
        \tand 
        \formulabox{u2}{v(\X(\uformula), 1)}{BurntOrange}
    )
   $
  \&
   $\formulabox{f2}{v(\uformula, 2)}{BurntOrange} =
   \formulabox{b2}{v(b,2)}{\atomcolor}
     \tor 
     (  \formulabox{a2}{v(a,2)}{\atomcolor}
        \tand 
        \formulabox{u3}{v(\X(\uformula), 2)}{\atomcolor}
    )
   $
  \\
  \fuzzyg{$\max\{$ 0.2 $,\min\{$ 0.5 $,$ 0.8 $\}\}=$ 0.5}{0.5}
  \&
  \fuzzyg{$ \max\{$ 0.3 $, \min\{$ 0.9 $,$ 0.8 $ \}\}=$ 0.8}{0.8}
  \&
  \fuzzyg{$\max\{$ 0.8 $,\min\{$ 0.2 $,$ 0 $\}\}=$ 0.8 $= v(b,2)$}{0.8}
  \\
  \fuzzyp{0.2 + (0.5 $\cdot$ 0.8) - 0.2 $\cdot$ (0.5 $\cdot$ 0.8) $=$ 0.52}{0.52}
  \&
  \fuzzyp{0.3 + (0.9 $\cdot$ 0.8) - 0.3 $\cdot$ (0.9 $\cdot$ 0.8) $\simeq$ 0.8}{0.8}
  \&
  \fuzzyp{$(0.8 + 0.2 \cdot 0) - 0.8 \cdot (0.2 \cdot 0) =$ 0.8 $= v(b,2)$}{0.8}
  \\
  \fuzzyg{$\min\{$ 1$, $ 0.2 + $\max\{$ 0$, $ 0.5+1-1$\}\}=$ 0.7}{0.7}
  \&
  \fuzzyg{$\min\{$ 1$, $ 0.3 + $\max\{$ 0$, $ 0.9+0.8-1$\}\}=$ 1}{1}
  \&
  \fuzzyg{$\min\{$ 1$,  $0.8 + $ \max\{$0$,  $0.2+0-1$\}\}=$ 0.8 $= v(b,2)$}{0.8}
  \\
};
\node[
    right=0mm of u-1-1.west,
    anchor=east,
] {\uformula};
\node[
    right=0mm of u-2-1.west,
    anchor=east,
] {\textbf{G}};
\node[
    right=0mm of u-3-1.west,
    anchor=east,
] {\textbf{P}};
\node[
    right=0mm of u-4-1.west,
    anchor=east,
] {\textbf{L}};

\draw[out=80,in=-80] (a0.north) edge[link] (a0|-l-2-1.south);
\draw[out=100,in=-100] (b0.north) edge[link] (b0|-l-3-1.south);
\draw[out=80,in=-80] (a1.north) edge[link] (a1|-l-2-1.south);
\draw[out=100,in=-100] (b1.north) edge[link] (b1|-l-3-1.south);
\draw[out=80,in=-80] (a2.north) edge[link] (a2|-l-2-1.south);
\draw[out=100,in=-100] (b2.north) edge[link] (b2|-l-3-1.south);
\draw[out=30,in=150] (u1.north) edge[link] (f1.north);
\draw[out=30,in=150] (u2.north) edge[link] (f2.north);

\end{tikzpicture}
}

 \caption{Graphical depiction of a \FLTLF trace of length 3, and satisfaction of the formula $a \U b$ in every instant of the trace, under all three semantics.  }
    \label{fig:figUexample}
\end{figure}
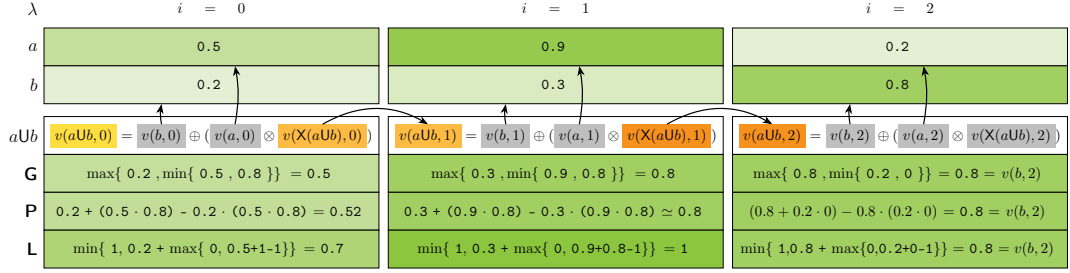

\noindent 
This logic matches the one considered in \cite{DonadelloFIMM25}. Intuitively, the t-norm preserves the intuitive behavior of the conjunction so that its value depends on the weakest evidence of truth, while the t-conorm preserves the intuitive behavior of disjunction and thus selects the strongest evidence. The remaining connectives correspond to classical negation and the fuzzy counterpart of material implication. 

%For example, for evaluating $\varphi_1 \U \varphi_2$ in this logic, the different candidate witness instants for $\varphi_2$ is taken is compared by ta depends on the 

\paragraph{{\FLTLF[P]}: Product semantics, with standard negation} This semantics regards truth values as akin to  probabilistic information, namely, conjunction progressively weakens the value of truth evidence by multiplicative accumulation, while disjunction accounts for independent evidences. The values for conjunction and disjunction, therefore, behave as the probability of the intersection and union of independent events, respectively: 

\begin{multicols}{2}
\begin{itemize}[itemsep=0pt,parsep=3pt,topsep=0pt]
\item $\alpha \tand \beta = \alpha \cdot \beta $	
\item $\alpha \tor \beta = \alpha + \beta - \alpha \cdot \beta$
\item $\tneg \alpha = 1 - \alpha$
\item $\alpha \timpl \beta = (\tneg \alpha) \tor \beta$	
\end{itemize}
\end{multicols}

\paragraph{{\FLTLF[L]}: Łukasiewicz semantics} This semantics captures the notion that the truth value can be treated as a limited additive quantity, so by disjunction evidence of truth accumulates additively until a maximum, whereas by conjunction they combine additively but can compensate for each other until a point, so that an insufficient total value collapses to zero:
%\todo{Marco: sistemare la frase sopra (non mi viene in mente come)}

\begin{multicols}{2}
\begin{itemize}[itemsep=0pt,parsep=3pt,topsep=0pt]
\item $\alpha \tand \beta = \max\{ 0, \alpha + \beta-1\} $	
\item $\alpha \tor \beta = \min\{ 1, \alpha + \beta\}$
\item $\tneg \alpha = 1 - \alpha$
\item $\alpha \timpl \beta = (\tneg \alpha) \tor \beta$	
\end{itemize}
\end{multicols}

\bigskip

\begin{comment}

\begin{figure}[bt]
\centering
{\footnotesize
\[
\begin{array}{|@{\hspace{2pt}}c@{\hspace{2pt}}|@{\hspace{2pt}}c@{\hspace{2pt}}|@{\hspace{2pt}}c@{\hspace{2pt}}|@{\hspace{2pt}}c@{\hspace{2pt}}|}
\hline
\textbf{}
&
\textbf{G}
&
\textbf{P}
&
\textbf{L}
\\
\hline

v(a \U b,\lambda,2)
&
0.8
&
0.8
&
0.8
\\
\hline

v(a,\lambda,1) {\tand} v(a \U b,\lambda,2)
&
{\min(0.9,0.8)}{=}0.8
&
{0.9\cdot0.8}{=}0.72
&
{\max(0,0.9{+}0.8{-}1)}{=}0.7
\\
\hline

v(a \U b,\lambda,1)
&
{\max(0.3,0.8)}{=}0.8
&
{0.3{+}0.72{-}0.22}{=}0.8
&
{\min(1,0.3{+}0.7)}{=}1
\\
\hline

v(a,\lambda,0) {\tand} v(a \U b,\lambda,1)
&
{\min(0.5,0.8)}{=}0.5
&
{0.5\cdot0.8}{=}0.4
&
{\max(0,0.5{+}1{-}1)}{=}0.5
\\
\hline

v(a \U b,\lambda,0)
&
{\max(0.2,0.5)}{=}\mathbf{0.5}
&
{0.2{+}0.4{-}0.08}{=}\mathbf{0.52}
&
{\min(1,0.2{+}0.5)}{=}\mathbf{0.7}
\\
\hline

\end{array}
\]

  \caption{Evaluation steps for the formula $a \U b$ over trace $\lambda=\langle \{ a\mapsto 0.5, b\mapsto 0.2 \}, \{ a\mapsto 0.9, b\mapsto 0.3 \}, \{ a\mapsto 0.2, b\mapsto 0.8 \} \rangle$ for each of the three semantics. The evaluation proceeds top-down, from instant $2$ to instant $0$.}
  \label{fig:figUexample}
  
  }%footnotesize
  
\end{figure} 
\end{comment}

To see how the choice of semantics affects the meaning and evaluation of formulas, consider again a formula of the form $\phi \U \psi$ evaluated at instant $0$. 
Its value results in general from the contribution of multiple possible choices of the instant $j$ at which the temporal eventuality is realised, namely at which the value $v(\psi,\lambda,j)$ is taken together with the support provided by each $v(\phi,\lambda,i)$ until then for $0\leq i<j$.

The recursive semantics at each instant combines three ingredients: 
\begin{enumerate}
\item the value associated with the immediate realization of the eventuality (i.e., the value of $\psi$); 
\item the value of $\X(\phi \U \psi)$ associated with all the deferred realizations; 
\item the local support to these deferred realizations (i.e., the value of $\phi$). 
\end{enumerate}
Operationally, this occurs in two distinct steps, proceeding backwards, as illustrated in Figure \ref{fig:figUexample}. 
The t-norm $\tand$ updates the value of deferred realizations by combining it with their local support given by $\phi$. %, hence behaving as backward propagation through the current state. 
Instead, the t-conorm $\tor$ combines the so-obtained value with the value of $\psi$ (i.e., the value of the immediate realization). 

In \FLTLF[G], the t-norm selects the minimum between local support $v(\phi,\lambda,i)$ and the value of deferred realizations, so only one of these is considered. For instance, in the evaluation illustrated in Figure \ref{fig:figUexample} for formula $a\U b$, at instant $1$ this is value $\min\{v(a,\lambda,1), v(a \U b,\lambda,2)\}=0.8$. Then, the t-conorm selects the maximum, namely the best choice between immediate realization and possible deferred realizations. In Figure \ref{fig:figUexample}, for instant $1$ this is the value $\max\{v(b,\lambda,1), 0.8\}=0.8$. As a result, the final value of the formula corresponds exactly to the choice of $j$ such that $\min\{v(\phi,\lambda,0), \cdots ,\allowbreak v(\phi,\lambda,j{-}1) ,\allowbreak v(\psi,\lambda,j)\}$ is maximal. In Figure \ref{fig:figUexample}, we indeed see that by selecting $j=2$ we obtain the maximal value $0.5$: for $j=1$ and $j=0$ the value of $b$ is too low, while the local support given by the value of $a$ is at least $0.5$. So the final value is the minimum between $0.5$ and the value of $b$ at instant $2$. This behavior fundamentally relies on the fact that the t-conorm is associative and monotonic. Even if the local support at any step supports many deferred realizations at once, these repeated support values do not compound. We can actually think of realizations (choices of $j$) as evaluated independently, selecting the best one.

In \FLTLF[P], the t-norm updates the value of deferred realizations by multiplying it with their local support $v(\phi,\lambda,i)$, which typically weakens it, so both contribute. Also, to combine the so-obtained value with $v(\psi,\lambda,i)$ for immediate realization, the t-conorm operates a probabilistic-style accumulation: again both contribute, but they are treated as independent events and their value discounted to account for overlaps. Thus, unlike \FLTLF[G], all possible realizations (choices of $j$) contribute at once to the final value. 

In \FLTLF[L], local support $v(\phi,\lambda,i)$ to deferred realizations is combined with their value by a t-norm that remains positive whenever the sum is greater than $1$. This means that weaker support at earlier instants may be offset by stronger support at later instants that is propagated back (i.e., higher support of $\phi$ until the eventuality is realised and a higher value of $\psi$ later). The t-conorm combines the resulting value with $v(\psi,\lambda,i)$ for immediate realization additively (up to a value of $1$), so the value may progressively strengthen. As for \FLTLF[P], multiple possible realizations may contribute at once until saturation to $1$.
%\todo{Ale: cosa dite, lasciamo la frase così? Oppure aggiungiamo qualcosa sul fatto che non è detto? Di base, in questo caso, il min e max ci sono comunque, e a seconda dell'interpretazione specifica, e dell'operatore specifico, con \FLTLF[L] si hanno casi dove nessun valore contribuisce, o solo alcuni. Quindi: in Godel, sempre uno solo contribuisce; in product, sempre tutti; in Luka, dipende da interpretazione specifica}\todo{P: secondo me ora è corretto perchè diciamo che contribuiscono fino a saturare, quindi contributi successivi con contribuiscono. se vogliamo indebolire la frase possiamo scrivere che possono contribuire: may contribute.}

So for the same formula, the choice of semantics affects the value also because of how different ways of making the formula `true' interact with each other.

\subsection{Additional Temporal Operators}
\label{subsec:equivalences}

% Based on the basic temporal operators $\X$ and $\U$ of \FLTLF, 
In this section we consider additional temporal operators that are a fuzzy counterpart of well known \LTLF temporal operators, namely $\F$ (eventually), $\G$ (always), $\Xw$ (weak next), $\W$ (weak until), $\R$ (release), $\M$ (strong release). 
%
%As we are going to show, while $\F$, $\G$, $\Xw$ and $\R$ can be directly obtained as abbreviations from the base logic in the previous section, in \FLTLF[G] also $\W$ and $\M$ can be derived. 
As we are going to show, while $\F$, $\G$, $\Xw$ and $\R$ can be directly obtained as abbreviations from the base logic in the previous section, the $\W$ and $\M$ operators in general need to be explicitly defined, with the exception of the \FLTLF[G] case where they can also be derived.

We first give below the explicit semantics for each of these. %, then we prove that according to this semantics we obtain some of the equivalences known to hold in \LTLF. 
We extend \FLTLF as follows:
\[
\begin{array}{r c l}
v(\Xw \varphi,\lambda,i) & = & v(\varphi,\lambda,i{+}1) \text{ if } i<last(\lambda) \text{, }  1 \text{ otherwise}; \\
%v(\F \varphi,\lambda,i) & = & v(\varphi,\lambda,i) ~\tor~  v(\F \varphi,\lambda,i{+}1)  \text{ if } i\leq last(\lambda) \text{, } 0 \text{ otherwise;} \\
%v(\G\varphi,\lambda,i) & = & v(\varphi,\lambda,i) ~\tand~ v(\G\varphi,\lambda,i{+}1)  \text{ if } i\leq last(\lambda) \text{, } 1 \text{ otherwise;} \\
v(\F \varphi,\lambda,i) & = & v(\varphi,\lambda,i) ~\tor~  v(\X~\F \varphi,\lambda,i); \\% \text{ if } i\leq last(\lambda), 0 \text{ otherwise} ; \\
v(\G\varphi,\lambda,i) & = & v(\varphi,\lambda,i) ~\tand~ v(\Xw ~\G\varphi,\lambda,i); \\ % \text{ if } i\leq last(\lambda), 1 \text{ otherwise} ; \\
%v(\varphi_1 \W \varphi_2,\lambda,i) & = &  max\{ ~v(\varphi_2,\lambda,i) ~,~ v(\varphi_1,\lambda,i)~\tand~  v(\Xw(\varphi_1 \W \varphi_2),\lambda,i) ~\}  \\
v(\varphi_1 \W \varphi_2,\lambda,i) & = &  v(\varphi_2,\lambda,i) ~\tor~ ( v(\varphi_1,\lambda,i)~\tand~  v(\Xw(\varphi_1 \W \varphi_2),\lambda,i) );  \\
%nota: se tengo la riga "if" sotto (e qui uso i+1 invece di usare qui \Xw) allora uso la t-norm nell'ultimo istante: qui invece all'ultimo istante scelgo ancora la migliore spiegazione
%& & \text{if } i< last(\lambda) \text{, } v(\varphi_1,\lambda,i) ~\tor~ v(\varphi_2,\lambda,i) \text{ otherwise}. \\
%v(\varphi_1 \W \varphi_2,\lambda,i) & = &  max\{ ~v(\varphi_2,\lambda,i) ~,~ v(\varphi_1,\lambda,i)~\tand~  v(\varphi_1 \W \varphi_2,\lambda,i{+}1) \}  \\
%& & \text{if } i< last(\lambda) \text{, } v(\varphi_1,\lambda,i) ~\tor~ v(\varphi_2,\lambda,i) \text{ otherwise}. \\
%& & \text{if } i< last(\lambda) \text{, } v(\varphi_2,\lambda,i) \text{ otherwise}. \\
v(\varphi_1 \M \varphi_2,\lambda,i) & = &  v(\varphi_2,\lambda,i) ~\tand~ ( v(\varphi_1,\lambda,i) ~\tor~  v(\X (\varphi_1 \M \varphi_2),\lambda,i) );  \\
v(\varphi_1 \R \varphi_2,\lambda,i) & = &  v(\varphi_2,\lambda,i) ~\tand~ ( v(\varphi_1,\lambda,i) ~\tor~  v(\Xw (\varphi_1 \R \varphi_2),\lambda,i) ). \\
%& & \text{if } i< last(\lambda) \text{, } 
	%min\{~ v(\varphi_1,\lambda,i),~v(\varphi_2,\lambda,i)~\} \text{ otherwise}. \\
\end{array}
\]

Intuitively, $\Xw \varphi$ is analogous to $\X \varphi$, although it does not require the next instant to exist (in which case, the value is trivially $1$, i.e., maximally true).  
The value of $\F \varphi$ measures the evidence of $\varphi$ being true at any point in the remaining portion of the trace, hence its value is obtained by combining the current and future values of $\varphi$ using the t-conorm, to account for all possible ways in which the requirement can be satisfied. 
The value of $\G \varphi$ measures the evidence of $\varphi$ being true in all instants in the remaining portion of the trace, so its value depends on the value of $\varphi$ in the current instant and in each instant until the end. All these values are combined using the t-norm. 
The $\W$ operator is the weak counterpart of $\U$: the value of $\varphi_1 \W \varphi_2$ depends on the support given by $\varphi_1$ until the value of $\varphi_2$ is positive, but a trace in which $\varphi_2$ has no positive value can still have a positive value. 
$\varphi_1\M\varphi_2$ is used to express that eventually $\varphi_1$ has a positive value and up to (and including) that instant the support is given by the value of  $\varphi_2$. Intuitively, after that point, the value of $\varphi_2$ is `released' from providing support. %\todo{Ale: da non esperto in LTL, mi aveva un po' confuso questa definizione di M. Se non ho capito male, la differenza con l'Until, ad eccezione dell'ordine dei due operandi, è che nell'Until non appena $\varphi_1$ diventa vera, $\varphi_2$ può avere qualsiasi valore, mentre con la M, c'è uno step temporale in cui devono essere vere entrambe, prima che la $\varphi_2$ venga "rilasciata". Ho capito bene? Se si, penso che vada chiarita questa cosa}
$\R$ is the weak version of $\M$, so there is no requirement on $\varphi_1$ having a positive eventually.

%$\M$ is the strong counterpart of $\R$ \todo{Ale: aggiungere spiegazione di R} the value of $\varphi_1\M\varphi_2$ depends on the continuous support by the values of $\varphi_2$ until $\varphi_1$ contributes positively, which for $\R$ is not required to happen. 
%
%At each step, the value of $\varphi_2$ is combined through the t-norm with either the current contribution of $\varphi_1$ or the recursive continuation of the formula in the future. 

\smallskip
In the propositions we establish next, we use the following convention: whenever the proposition is stated for \FLTLF, this means that the proposition holds for all the three logics \FLTLF[G], \FLTLF[P], and \FLTLF[L]. As a preliminary step in our study of inter-definability of temporal operators, we recall two known results on the fuzzy logics we consider.   
%no, not complete: The former can be seen as direct consequence of the use of standard negation in all three semantics, while the latter by the choice of $\tand$ and $\tor$ in \FLTLF[G], which are idempotent and coincide with lattice meet/join on $[0,1]$. 
First, the chosen combinations of t-norms, t-conorms and standard negation form so-called De Morgan triplets \cite{book_fuzzy_KY95}. 
%
%Second, the totally ordered set $[0,1]$ forms a distributive lattice where the meet and join operations are taken as minimum and maximum. 
Second, the choice of $\min$ and $\max$ in this setting implies distributivity. 
As a direct application to our logics, we can state the following two propositions.
%
%\todo[inline]{Ale: visto che in tutte le prove formali $\lambda$ è fisso, forse conviene definire un abuso di notazione con $v(\varphi, \lambda, i)$ sostituito da $v(\varphi, i)$ quando la $\lambda$ è fissa, così da alleggerire un po' le prove formali}

\begin{proposition}
	The De Morgan's law holds for \FLTLF. Namely in all these logics 
	$v(\neg(\varphi_1 \land \varphi_2),\lambda, i)=v(\neg\varphi_1 \lor \neg\varphi_2,\lambda,i)$ and $v(\neg(\varphi_1 \lor \varphi_2),\lambda, i)=v(\neg\varphi_1 \land \neg\varphi_2,\lambda,i)$ for any $\lambda$ and $0\leq i \leq last(\lambda)$. 
\end{proposition}

Namely, 
	$\tneg (v(\varphi_1,\lambda, i) \tand v(\varphi_2,\lambda, i))=\tneg v(\varphi_1,\lambda, i) \tor (\tneg v(\varphi_2,\lambda,i))$ and $\tneg (v(\varphi_1,\lambda, i) \tor v(\varphi_2,\lambda, i))=\tneg  v(\varphi_1,\lambda, i) \tand (\tneg v(\varphi_2,\lambda,i))$.

\begin{proposition}\label{prop:distrib_onlyG}
	In \FLTLF[G], $\tand$ and $\tor$ coincide respectively with the meet and join operations of the totally ordered set $[0,1]$. Hence they form a distributive lattice and are distributive with respect to each other. However this is not the case for \FLTLF[P] and \FLTLF[L].
\end{proposition}

We can also immediately prove that the semantics for $\F$ and $\G$ operators can be simplified as follows, which allows for more efficient implementations.

\begin{proposition}\label{prop:unfold_FG} 
	In \FLTLF: 
\begin{equation}
\begin{array}{rcl}
	v(\F \varphi,\lambda,i) &=& v(\varphi,\lambda,i) \oplus \cdots \oplus v(\varphi,\lambda,last(\lambda))  \\
	v(\G \varphi,\lambda,i) &=& v(\varphi,\lambda,i) \otimes \cdots \otimes v(\varphi,\lambda,last(\lambda))  
\end{array}
\end{equation}
\end{proposition}

\begin{proof}
Recursively using the definition of $\F$ and $\G$ and the fact that $v(\X\varphi, \lambda, i)=0$ and $v(\X_w \varphi, \lambda, i)=1$ for $i\ge last(\lambda)$ we have
	$v(\F \varphi,\lambda,i) = v(\varphi,\lambda,i) \tor \ldots \tor v(\varphi,\lambda,last(\lambda)) \tor 0 \ldots  = v(\varphi,\lambda,i) \tor \ldots \tor v(\varphi,\lambda,last(\lambda))$
	and 
	$v(\G\varphi,\lambda,i) =  v(\varphi,\lambda,i) \tand \ldots \tand  v(\varphi,\lambda,last(\lambda)) \allowbreak \tand 1 \ldots = v(\varphi,\lambda,i) \tand \ldots \tand v(\varphi,\lambda,last(\lambda))$.  
\end{proof}

We finally look at these known equivalences for \LTLF to see whether these hold for \FLTLF as well, namely whether the additional temporal operators can be derived by others as abbreviations:
\begin{equation}\label{eq:equiv}
\begin{array}{rclrcl}
\F \varphi &\equiv& \top \U \varphi &
\G \varphi &\equiv& \neg \F \neg \varphi \\
\Xw \varphi &\equiv& \X\varphi \lor \neg \X\top \qquad  &
\varphi_1 \W \varphi_2 &\equiv& \varphi_1 \U \varphi_2 \lor \G \varphi_1 \\
%\varphi_1 \R \varphi_2 &\equiv& \neg (\neg \varphi_1 \U \neg  \varphi_2)  &
%\varphi_1 \M \varphi_2 &\equiv& \neg (\neg \varphi_1 \W \neg \varphi_2)  \\
\varphi_1 \M \varphi_2 &\equiv& \varphi_2 \U (\varphi_1 \land \varphi_2) \qquad &
\varphi_1 \R \varphi_2 &\equiv& \varphi_2 \W (\varphi_1 \land \varphi_2)
\end{array}
\end{equation}

%For \FLTLF[G], as this logic is the one considered in \cite{DonadelloFIMM25}, the fact that these equivalences hold could be directly taken from there.  %, so it would suffice to show the same for \FLTLF[P] and \FLTLF[L]. 
The results are summarised as follows. For ease of notation, in these proofs we write $v(\varphi,i)$ in place of $v(\varphi,\lambda,i)$ since the trace $\lambda$ is always fixed. 

\begin{proposition}\label{propF}
$\F \varphi \equiv \top \U \varphi$ in \FLTLF.
\end{proposition}
%
%\begin{proof}
%Since $v(\F \varphi,\lambda,i) =  v(\varphi,\lambda,i) \tor  v(\X\F \varphi,\lambda,i)$, we have
%
%$v(\top \U \varphi,\lambda,i) =  v(\varphi,\lambda,i) ~\tor~ (  v(\top,\lambda,i) ~\tand~  v(\X(\top \U \varphi),\lambda,i))$ which simplifies to $v(\varphi,\lambda,i) \tor$ $( v(\X(\top \U \varphi),\lambda,i))$. 
%
%Hence   
%
%$v(\top \U \varphi,\lambda,i) = 
%v(\varphi,\lambda,i) \tor  
%( v(\varphi,\lambda,i{+}1) \tor 
%( \ldots \tor ( 
% v(\varphi,\lambda,last(\lambda)) )))
%\allowbreak
%=v(\varphi,\lambda,i) \tor \allowbreak
%\ldots \tor \allowbreak
%v(\varphi,\lambda,last(\lambda))$, so the equivalence follows by Proposition.~\ref{prop:unfold_FG}.
%\end{proof}

\begin{proof}
Reasoning by backward induction, first we check the boundary condition: for $i=last(\lambda)$ we have $v(\F \varphi,i) = v(\varphi,i) = v(\top \U \varphi,i)$.  
In the inductive step, 
$v(\top \U \varphi,i) = v(\varphi,i) \tor ( v(\top,i) \tand v(\top \U \varphi,i{+}1))$ which by inductive hypothesis is equal to 
$v(\varphi,i) \tor v(\F\varphi,i{+}1)$, and thus equal to $v(\F \varphi,i)$.  
\end{proof}

\begin{proposition}\label{propG}
$\G \varphi \equiv \neg \F \neg \varphi$ in \FLTLF.
\end{proposition}
\begin{proof}
The result follows from Proposition~\ref{prop:unfold_FG} by applying De Morgan's law, as $\tneg ( \tneg v(\varphi,i) \tor \ldots \tor \tneg v(\varphi,last(\lambda)) )$ is equal to $v(\neg\F \neg \varphi,i)$. 
% %
% The result then follows by applying De Morgan's law, as $\tneg ( \tneg v(\varphi,\lambda,i) \tor \ldots \tor \tneg v(\varphi,\lambda,last(\lambda)) )$ is equal to $v(\neg\F \neg \varphi,\lambda,i)$. 
\end{proof}

\begin{proposition}
$\Xw \varphi \equiv \X\varphi \lor \neg \X\top$ in \FLTLF.
\end{proposition}

\begin{proof}
We directly show the equality at each instant. $v(\Xw \varphi,i)=v(\varphi,i+1)$ if $i<last(\lambda)$, $1$ otherwise. So for $i=last(\lambda)$ we have  $v(\Xw \varphi,i)=1=v(\neg\X\top)$ and $v(\X \varphi,i)=0$, hence equality holds since $0\tor 1=1$. For  $i<last(\lambda)$ we have $v(\Xw \varphi,i)=v(\varphi,i{+}1)=v(\X \varphi,i)$ while $v(\neg\X \top,i)=0$, and these are equal since $\alpha\tor 0=\alpha$. 
\end{proof}

%The remaining equivalences from \LTLF hold for \FLTLF[G] but not for the two other logics. Indeed, it is easy to see that the following proofs requires distributivity, which by Proposition~\ref{prop:distrib_onlyG} does not hold for \FLTLF[P] and \FLTLF[L].

%\todo[inline]{Ale: ok, ma a meno che non si sia dimostrato che la proprietà distributiva sia una proprietà necessaria, questa non è una prova. Il fatto che la prova formale proposta usi la distributività, non significa che per una semantica priva di tale distributività non valga il teorema.}
%\todo[inline]{Ale: Guardando meglio sotto, sulla proposizione 8 sembra invece che si vada nel dettaglio, e che si dica proprio che la distributività sia necessaria. Se è così, questa frase qua sopra va modificata. In generale però, è poco chiaro dalle prove il perché sia necessaria}

%For instance, thanks to the distributivity of $\tor$ and $\tand$, $\varphi_1 \W \varphi_2 \equiv \varphi_1 \U \varphi_2 \lor \G \varphi_1$ holds in \FLTLF[G] but not under the product semantics because the conjunction and disjunction are not idempotent.

\begin{proposition}\label{propW}
$\varphi_1 \W \varphi_2 \equiv (\varphi_1 \U \varphi_2) \lor \G \varphi_1$ only in \FLTLF[G].
\end{proposition}
\begin{proof}
%By backward induction, consider first $i=last(\lambda)$. $v(\varphi_1 \W \varphi_2,\lambda, i)= v(\varphi_2, \lambda, i) \tor v(\varphi_1, \lambda, i)$. Moreover, $v(\varphi_1 \U \varphi_2,\lambda, i) = v(\varphi_2, \lambda, i)$ and $v(\G \varphi_1,\lambda, i) = v(\varphi_1, \lambda, i)$. So the result holds.  
%
We start by showing that the equivalence holds for \FLTLF[G]. In particular, we prove the claim by backward induction. 
First, it is immediate to see that the boundary condition is true, namely for $i=last(\lambda)$   $v(\varphi_1 \W \varphi_2, i)= v(\varphi_2, i) \tor v(\varphi_1, i)$ while $v(\varphi_1 \U \varphi_2, i) = v(\varphi_2, i)$ and $v(\G \varphi_1, i) = v(\varphi_1, i)$. 
Then, assuming it holds for some $0<n\leq last(\lambda)$, consider the inductive step for $i={n}{-}1$. 
On the LHS and RHS we have, respectively: 
{\small
\begin{align*}
v(\varphi_1 \W \varphi_2, i)&= v(\varphi_2, i) \tor (v(\varphi_1,i) \tand v(\varphi_1 \W \varphi_2,i{+}1)), \\
v((\varphi_1 \U \varphi_2) {\lor} \G\varphi_1,i) &= (v(\varphi_2, i) {\tor} ( v(\varphi_1,i) {\tand} v(\varphi_1\U \varphi_2,i{+}1)) {\tor} ( v(\varphi_1, i) {\tand} v(\G \varphi_1,i{+}1)).
\end{align*}
}

%\noindent
%$
%{\small 
% \begin{array}{@{}r@{\;}c@{\;}l}
% v(\varphi_1 \W \varphi_2, i)&=& v(\varphi_2, i) \tor (v(\varphi_1,i) \tand v(\varphi_1 \W \varphi_2,i{+}1)) \\
% v((\varphi_1 \U \varphi_2) {\lor} \G\varphi_1,i) &=& (v(\varphi_2, i) {\tor} ( v(\varphi_1,i) {\tand} v(\varphi_1\U \varphi_2,i{+}1)) {\tor} ( v(\varphi_1, i) {\tand} v(\G \varphi_1,i{+}1)).
% \end{array}
%}
%$

Recalling that $\tand$ and $\tor$ are $\min$ and $\max$, respectively, on the RHS we get $\max\{ v(\varphi_2, i), \min \{ v(\varphi_1,i), \max\{ v(\varphi_1\U \varphi_2,i{+}1), v(\G \varphi_1,i{+}1) \}\}\}$  by associativity of $\max$ and distributivity of $\min$ over $\max$. On the LHS instead $\max\{ v(\varphi_2, i), \min \{ v(\varphi_1,i), v(\varphi_1\W\varphi_2, i{+}1) \}\}$. These are equal by the inductive hypothesis, namely $v(\varphi_1\W\varphi_2, i{+}1)=\max\{ v(\varphi_1\U \varphi_2,i{+}1), v(\G \varphi_1,i{+}1) \}$.  
%
%On the LHS, $v(\varphi_1 \W \varphi_2, i)= v(\varphi_2, i) \tor (v(\varphi_1,i) \tand v(\varphi_1 \W \varphi_2,i{+}1))$ which by hypothesis  is equal to  $v(\varphi_2, i) \tor (v(\varphi_1, i) \tand (v(\varphi_1 \U \varphi_2, i{+}1) \tor v(\G\varphi_1,i{+}1))$. 
%
%On the RHS, $v(\varphi_1 \U \varphi_2,i)=v(\varphi_2, i) \tor ( v(\varphi_1,i) \tand v(\varphi_1\U \varphi_2,i{+}1)$ and 
%$v(\G \varphi_1,  i)=v(\varphi_1, i) \tand v(\G \varphi_1,i{+}1)$. 
%
%So the equality is obtained by distributivity (see Prop.~\ref{prop:distrib_onlyG}), and 
Therefore the claim for \FLTLF[G] holds by backward induction.

Regarding the negative result for \FLTLF[P] and \FLTLF[L], the above reasoning fails due to the fact that distributivity does not hold. % (recall, instead, that t-norms and t-conorms are always associative). 
Indeed, it is immediate to find a counterexample: consider $\lambda = \langle \{ a\mapsto 0.5, b \mapsto 0.6  \}, \{ a\mapsto 0.4, b \mapsto 0.7  \}\rangle$. 
In \FLTLF[P] we have on LHS $v(a \W b,1)=0.7\tor (0.4\tand 1)=0.82$ and so $v(a \W b,0)=0.6\tor (0.5\tand 0.82)=0.764$, 
and on the RHS $v(a \U b,1)=0.7$ and so $v(a \U b,0)=0.6\tor (0.5\tand 0.7)=0.74$, and also  $v(\G a,1)=0.4$ and so $v(\G a,0)=0.5\tand 0.4=0.2$. By combining these values, we observe that $v(a \W b,0)\neq v(a \U b,0)\tor v(\G a,0)$, indeed $0.764\neq 0.74\tor 0.2 = 0.792$. 

For \FLTLF[L], we obtain $v(a \W b,1)=0.7\tor (0.4\tand 1)=1$, $v(a \W b,0)=0.6\tor (0.5\tand 1)=1$, $v(a \U b,1)=0.7$, $v(a \U b,0)=0.8$, $v(\G a,1)=0.4$ and $v(\G a,0)=0.5\tand 0.4=0$, by which on the LHS we have $v(a \W b,0)=1$ whereas on the RHS we get $v((a \U b)\lor \G a,0)=\min\{1,0.8+0\}=0.8$. 
\end{proof}

\begin{proposition}\label{propMR}
In \FLTLF, operators $\M$ and $\R$ are dual to $\W$ and $\U$, respectively:
$$
\varphi_1 \M \varphi_2 \equiv \neg (\neg \varphi_1 \W \neg \varphi_2)
\qquad
\varphi_1 \R \varphi_2 \equiv \neg (\neg \varphi_1 \U \neg \varphi_2)
$$
\end{proposition}

\begin{proof}
%The claim can be verified by expanding the semantics. 
%
We show this for $\M$. First, from $v(\varphi_1 \M \varphi_2,i) = v(\varphi_2,i) \tand (v(\varphi_1,i) \tor v( \X (\varphi_1 \M \varphi_2), i))$, by repeated application of De Morgan and since $\tneg v(\varphi,i) = v(\neg \varphi,i)$ and $\tneg v(\X\varphi,i) = v(\X\neg \varphi,i)$, we obtain 
%By negating both sides and applying De Morgan: $\tneg v(\varphi_1 \M \varphi_2,\lambda,i) = (\tneg v(\varphi_2,\lambda,i)) \tor \tneg (v(\varphi_1,\lambda,i) \tor v( \X (\varphi_1 \M \varphi_2, \lambda, i))$. By applying again De Morgan on the RHS we get $\tneg v(\varphi_1 \M \varphi_2,\lambda,i) = (\tneg v(\varphi_2,\lambda,i)) \allowbreak~\tor~\allowbreak (\tneg v(\varphi_1,\lambda,i) ~\tand~ \tneg v( \X (\varphi_1 \M \varphi_2, \lambda, i))$. Since $\tneg v(\X\varphi,\lambda,i) = v(\X\neg \varphi,\lambda,i)$, we obtain 
$v(\neg (\varphi_1 \M \varphi_2),i) =v(\neg \varphi_2,i) \allowbreak \tor \allowbreak ( v(\neg \varphi_1,i) \tand v( \X \neg (\varphi_1 \M \varphi_2), i))$. 

Now we show that $\neg (\varphi_1 \M \varphi_2) \equiv (\neg \varphi_1 \W \neg \varphi_2)$ by backward induction. 
The boundary condition is true: for $i=last(\lambda)$, on the LHS of the equivalence we have %$v(\neg (\varphi_1 \M \varphi_2), \lambda, i)=\tneg (v(\varphi_2,\lambda,i) \tand v(\varphi_1,\lambda,i))$ which is equal to 
$v(\neg (\varphi_1 \M \varphi_2), i)=
\tneg (v(\varphi_2,i)\tand (v(\varphi_1,i)\tor 0))=
v(\neg \varphi_2,i) \tor v(\neg \varphi_1,i)$. Likewise, on RHS $v(\neg \varphi_1 \W \neg \varphi_2, i) = v(\neg \varphi_2,i) \tor (v(\neg \varphi_1,i) \tand 1)=v(\neg \varphi_2,i) \tor v(\neg \varphi_1,i)$. 
For the inductive step (so for $0\leq i<last(\lambda)$), on the LHS and RHS we have, respectively: 
\begin{align*}
v(\neg(\varphi_1 \M \varphi_2),i) &= v(\neg\varphi_2,i) \tor (v(\neg\varphi_1,i) \tand v( \neg(\varphi_1 \M \varphi_2), i{+}1)) \\
v(\neg\varphi_1 \W \neg\varphi_2,i) &= v(\neg\varphi_2,i) \tor (v(\neg\varphi_1,i) \tand v( \neg\varphi_1 \W \neg\varphi_2, i{+}1))
\end{align*}
% $
% \begin{array}{@{}r@{\;}c@{\;}l}
% v(\neg(\varphi_1 \M \varphi_2)),i) &=& v(\neg\varphi_2,i) \tor (v(\neg\varphi_1,i) \tand v( \neg(\varphi_1 \M \varphi_2), i{+}1)) \\
% v(\neg\varphi_1 \W \neg\varphi_2,i) &=& v(\neg\varphi_2,i) \tor (v(\neg\varphi_1,i) \tand v( \neg\varphi_1 \W \neg\varphi_2, i{+}1))
% \end{array}
% $
which are equal by hypothesis $v( \varphi_1 \M \varphi_2, i{+}1) = v( \neg(\neg\varphi_1 \W \neg\varphi_2), i{+}1)$.

The proof for $\R$ (latter equivalence) has the exact same structure.
\end{proof}

\begin{proposition}\label{propMR2}
In \FLTLF[P] and \FLTLF[L] $\M$ and $\R$ cannot be derived by these standard \LTLF equivalences:
$$
\varphi_1 \M \varphi_2 \equiv \varphi_2 \U (\varphi_1 \land \varphi_2)
\qquad
\varphi_1 \R \varphi_2 \equiv \varphi_2 \W (\varphi_1 \land \varphi_2)
$$
whereas in \FLTLF[G] these equivalences hold. 
\end{proposition}

\begin{proof}
We first consider $\M$ (former equivalence).
Given any $\lambda$, assume that $v(\varphi_1 \M \varphi_2,n)=v(\varphi_2 \U (\varphi_1 \land \varphi_2),n)$ holds for some $0<n\leq last(\lambda)$. 
Consider $i=n{-}1$. 
%
%This requires the following two evaluations to be equal:
%\begin{itemize}
%	
%\item RHS: by expanding $ v(\varphi_2 \U (\varphi_1 \land \varphi_2),\lambda, i)$ we get:\\ 
%$(v(\varphi_1,\lambda,i) \tand v(\varphi_2,\lambda,i)) \tor (v(\varphi_2,\lambda,i) \tand  k ) $
%\item LHS: by expanding $ v(\varphi_1 \M \varphi_2,\lambda, i)$ we get:\\  
%$v(\varphi_2,\lambda,i) \tand ( v(\varphi_1,\lambda,i) \tor  k ) $
%\end{itemize}
%
%This is not guaranteed in \FLTLF[P] and \FLTLF[L].
%
By applying the semantics and the hypothesis, on the LHS and RHS we have, respectively: 
\begin{align*}
v(\varphi_1 \M \varphi_2, i) &= v(\varphi_2,i) \tand ( v(\varphi_1,i) \tor  v(  \varphi_1 \M \varphi_2,i{+}1) ) \\
v(\varphi_2 \U (\varphi_1 \land \varphi_2), i) &= (v(\varphi_1,i) \tand v(\varphi_2,i)) \tor (v(\varphi_2,i) \tand v(  \varphi_1 \M \varphi_2,i{+}1) ) 
\end{align*}
% $
% \begin{array}{@{}r@{\;}c@{\;}l}
% v(\varphi_1 \M \varphi_2, i) &=& v(\varphi_2,i) \tand ( v(\varphi_1,i) \tor  v(  \varphi_1 \M \varphi_2,i{+}1) ) \\
% v(\varphi_2 \U (\varphi_1 \land \varphi_2), i) &=& (v(\varphi_1,i) \tand v(\varphi_2,i)) \tor (v(\varphi_2,i) \tand v(  \varphi_1 \M \varphi_2,i{+}1) ) 
% \end{array}
% $
% \noindent
As to equate these directly requires distributivity, this fails for \FLTLF[P] and \FLTLF[L] since that property does not hold for these semantics. 

\smallskip
We provide a counterexample which also highlights this aspect: assume $\lambda = \langle \{ a\mapsto 0.5, b \mapsto 0.6  \}, \{ a\mapsto 0.4, b \mapsto 0.7  \}\rangle$. 

For \FLTLF[P], on the LHS we have $v(a\M b, 1) = 0.7\tand (0.4\tor 0)=0.7\cdot 0.4=0.28$ and so $v(a\M b, 0) = 0.6 \tand (0.5 \tor 0.28) = 0.6 \cdot (0.5+0.28-0.14) = 0.384$, while for the RHS we have $v(b \U (a\tand b), 1)  =0.7\tand (0.4\tor 0)= 0.28$ and so $v(b \U (a\tand b), 0) = (0.5\tand 0.6) \tor (0.6 \tand 0.28)=0.3 + 0.168 - 0.3\cdot 0.168=0.4176$. 

For \FLTLF[L], on the LHS we have $v(a\M b, 1) = 0.7\tand (0.4 \tor 0)=0.1$ and so $v(a\M b, 0) = 0.6 \tand (0.5 \tor 0.1) = \max\{0,0.6+\min\{1,0.5+0.1\}-1\} = 0.2$, while for the RHS we have $v(b \U (a\tand b), 1) = v(a\tand b, 1) = \max\{0,0.4+0.7-1\} = 0.1$ and so $v(b \U (a\tand b), 0) = (0.5\tand 0.6) \tor (0.6 \tand 0.1)=\min\{1,0.1+0\}=0.1$.

\smallskip
For \FLTLF[G] instead the equality in the inductive step holds because of distributivity. 
Moreover, the boundary condition holds as well (as it does for \FLTLF[P] and \FLTLF[L], since this only requires $\tor(x,0){=}x$ and commutativity of $\tand$): for $i=last(\lambda)$, on the LHS we have $v(\varphi_1 \M \varphi_2, i)=v(\varphi_2,i)\tand (v(\varphi_1,i)\tor 0)=v(\varphi_1,i)\tand v(\varphi_2,i)$ and on the RHS $v(\varphi_2 \U (\varphi_1 \land \varphi_2), i)=(v(\varphi_1,i)\tand v(\varphi_2,i))\tor (v(\varphi_2,i)\tand 0)=v(\varphi_1,i)\tand v(\varphi_2,i)$. Thus we obtain the claim by backward induction. 

%however, the boundary condition, which we need in order to apply backward induction, fails: for $i=last(\lambda)$, we have $v(\varphi_1 \M \varphi_2, i)=v(\varphi_2,i)$ and $v(\varphi_2 \U (\varphi_1 \land \varphi_2), i)=v(\varphi_1 \land \varphi_2,i)$.

For $\R$ (latter equivalence), in the inductive step (hence with $0<i<last(\lambda)$, namely the next instant exists) we apply  the semantics and the hypothesis (i.e., that $v(\varphi_1 \R \varphi_2,i{+}1)=v(\varphi_2 \W (\varphi_1 \land \varphi_2),i{+}1)$), obtaining on the LHS and RHS, respectively:
\begin{align*}
v(\varphi_1 \R \varphi_2, i) &= v(\varphi_2,i) \tand ( v(\varphi_1,i) \tor  v(  \varphi_1 \R \varphi_2,i{+}1) ) \\
v(\varphi_2 \W (\varphi_1 \land \varphi_2), i) &= (v(\varphi_1,i) \tand v(\varphi_2,i)) \tor (v(\varphi_2,i) \tand v(  \varphi_1 \R \varphi_2,i{+}1) ) 
\end{align*}
% $
% \begin{array}{@{}r@{\;}c@{\;}l}
% v(\varphi_1 \R \varphi_2, i) &=& v(\varphi_2,i) \tand ( v(\varphi_1,i) \tor  v(  \varphi_1 \R \varphi_2,i{+}1) ) \\
% v(\varphi_2 \W (\varphi_1 \land \varphi_2), i) &=& (v(\varphi_1,i) \tand v(\varphi_2,i)) \tor (v(\varphi_2,i) \tand v(  \varphi_1 \R \varphi_2,i{+}1) ) 
% \end{array}
% $
% \noindent
which has the exact shape obtained for the former equivalence, so the same considerations apply regarding \FLTLF[P] and \FLTLF[L]. 
Indeed, for the counter-example trace $\lambda$ as above, we obtain distinct values: for \FLTLF[P] $v(a \R b,1)=0.7$, $v(a \R b,0)=0.6\tand (0.5\tor 0.7)=0.51$, $v(b \W (a \land b),1)=(0.4\tand 0.7)\tor(0.7\tand 1)=0.784$, $v(b \W (a \land b),0)=(0.5\tand 0.6)\tor(0.6\tand 0.784)=0.62928$. For \FLTLF[L] $v(a \R b,1)=0.7$, $v(a \R b,0)=0.6\tand (0.5\tor 0.7)=0.6$, $v(b \W (a \land b),1)=(0.4\tand 0.7)\tor(0.7\tand 1)=0.8$, $v(b \W (a \land b),0)=(0.5\tand 0.6)\tor(0.6\tand 0.8)=0.5$.

Again, the boundary condition holds for \FLTLF[G]: for $i=last(\lambda)$ we have $v(\varphi_1 \R \varphi_2,i)= v(\varphi_2, i) \tand (v(\varphi_1, i) \tor 1)= \min\{v(\varphi_2, i), 1\}=v(\varphi_2, i)$. And likewise: $v(\varphi_2 \W (\varphi_1 \land \varphi_2),i)=(v(\varphi_1, i)\tand v(\varphi_2, i)) \tor (v(\varphi_2, i) \tand 1) = \max\{\min\{v(\varphi_1, i), v(\varphi_2, i)\}, v(\varphi_2, i)\}=v(\varphi_2, i)$. So the boundary condition holds and by distributivity so does the inductive equality, hence the claim follows by backward induction. 
\end{proof}

These results show that $\R$ can be obtained by duality from $\U$ in \FLTLF and $\M$ from $\W$, however $\W$ is an abbreviation only in \FLTLF[G].  
%Therefore, the temporal logical basis for \FLTLF[G] can be expressed as $\{ \X, \U \}$, whereas for the other semantics we consider $\W$ as primitive. %, obtaining the remaining ones as abbreviations.
Therefore, the temporal logical basis for \FLTLF[G] can be expressed as $\{ \X, \U \}$, whereas for the other semantics we consider also $\W$ as primitive.

As a final remark, note that irrespective of the semantics chosen, since standard negation is also crisp, these logics collapse to \LTLF in the special case in which only $0$/$1$ values are present in a trace (and thus it can be taken as a standard \LTLF trace as in Definition~\ref{def:traces_ltlf}). Indeed, on boolean values the t-norm corresponds to the boolean conjunction and the t-conorm to the boolean disjunction.  
% \section{Temporal LTN}
% \label{sec:TLTN}
\section{The \texorpdfstring{\DiffLTLf}{dLTLf} Framework}
\label{sec:diffLTLf}
% Outline:
% 5.1 Problem Formulation
% - weakly sup. seq class task
% - streams S, obs space X, classes C, atomic props partitioned across streams ...
% - input sequence x (with tuple of obs)
% - crisp symbolic trace (ONLY to define ground truth)?
% - dataset D
% - goal (learn parameters of the perception module from D with lambda never observed)
% 5.2 Framework Overview
% - pipeline figure (o figure 2 o un'altra)
% - 3 blocchi:
%     i) perception module (provides a fuzzy trace)
%         - what is the perception module (f_theta -> ...) with softmax over class logits
%         - applied independently per stream and instant
%         - gives the fuzzy trace \lambda (? controlla notation), valida per definition 1 sec 4
        
%     ii) the LTLf formula evaluated on the fuzzy trace under FLTLfx
%         - PENSARE BENE A COME DESCRIVERLO
%     iii) resulting truth val against GT with bce
%         - BCE loss between v(Phi, lambda, 0) and y
%         - gradient flow?
% - point out differences wrt DFA-based approaches (no DFA representation)
% - describe padding (?) e minibatch evaluation

In this section, we introduce the \DiffLTLf framework, starting with the problem formulation in Section~\ref{sec:problem}, before introducing the proposed approach in Section~\ref{sec:framework}, and its implementation in Section~\ref{sec:implementation}.%, concluding the section with the way we represent \FLTLF formulae in propositional ones in Section~\ref{sec:encoding}.

\subsection{Problem Formulation}
\label{sec:problem}

We consider a weakly supervised sequence classification task defined over multiple streams of observations, such as images. 
Let $\mathcal{S} = \{1, \dots, S\}$ be a set of distinct perceptual streams and let $\mathcal{X}$ be the observation space containing the raw data.
A sequence of length $n$ is $ \mathbf{x} = \langle \mathbf{x}_0, \dots, \mathbf{x}_{n-1}  \rangle$, where each element $\mathbf{x}_i = \langle x_i^1, \dots, x_i^{|\mathcal{S}|} \rangle$ at instant $i\in \{0,\dots,n-1\}$ is a tuple of observations, one for each stream $s\in\mathcal{S}$, with $x_i^s \in \mathcal{X}$.

We define two sets, $\mathcal{C}$ and $\mathcal{P}$.
The set of classes $\mathcal{C}$ is the set of possible labels of an observation (e.g., the digits $0, \dots, 9$ for the MNIST dataset), and, for simplicity, we assume it is shared across streams. \footnote{For notational simplicity, we assume that all streams share a common observation space $\mathcal{X}$. The framework extends simply to the general case of stream-specific observation spaces $\mathcal{X}^s$, one for each $s \in \mathcal{S}$, by replacing the shared perception module $f_\theta$ with stream-specific modules $f_\theta^s: \mathcal{X}^s \to [0,1]^{|\mathcal{C}^s|}$.}
The set of atomic propositions $\mathcal{P}$ is the alphabet over which the symbolic traces of Section~\ref{sec:fuzzyLTL} are defined.
For each stream $s$, we employ a distinct set of propositional symbols $\mathcal{P}^s$, where $p_c^s\in \mathcal{P}^s$ denotes the fact that an observation in stream $s$ belongs to class $c\in \mathcal{C}$. 
%
% We distinguish between the set of classes $\mathcal{C}$ (e.g., the digits $0, \dots, 9$), which represent the symbols groundable by the perception module and are shared across streams, and the set of atomic propositions $\mathcal{P}$ used in the symbolic traces. 
% %  
% For simplicity, we employ for each stream $s$ a distinct set of propositional symbols $\mathcal{P}^s$, where $p_c^s\in \mathcal{P}^s$ denotes the fact that an observation in stream $s$ belongs to class $c\in \mathcal{C}$. 
%
Therefore, $\mathcal{P} = \bigcup_{s\in \mathcal{S}}\mathcal{P}^s$ is the full propositional alphabet, with $\mathcal{P}^s\cap \mathcal{P}^{s'}=\emptyset$ for $s\neq s'$. 

The learning task is guided by a temporal specification $\Phi$, expressed as an \LTLF formula over the alphabet $\mathcal{P}$. 
Each input sequence of observations $\mathbf{x}$ corresponds to a crisp symbolic trace $\lambda^\mathbf{x} = \langle \lambda_0, \ldots, \lambda_{n-1} \rangle$ as defined in Definition~\ref{def:traces_ltlf}, with $\lambda_i(p_c^s) = 1$ if $x_i^s$ is of class $c$, and $0$ otherwise.
The binary label $y$ associated with $\mathbf{x}$ is $1$ if $\lambda^\mathbf{x}$ satisfies $\Phi$, and $0$ otherwise:
\begin{equation}
    y =
    \begin{cases}
    1 & \text{if } \lambda^\mathbf{x} \models_{\LTLF} \Phi \\
    0 & \text{otherwise}
    \end{cases}
\end{equation}
The crisp trace $\lambda^\mathbf{x}$, however, is not available during training.
The dataset $\mathcal{D}$  is a collection of $M$ pairs $(\mathbf{x}, y)$, where only the sub-symbolic input $\mathbf{x}$ and the binary label $y$ are observed.
The goal is to learn the parameters $\theta$ of a perception module, modelled as a function $f_\theta: \mathcal{X} \to [0,1]^{|\mathcal{C}|}$, shared across streams, that maps each raw observation to a vector of fuzzy truth values.
In other words, the task is an instance of weakly supervised symbol grounding, a common setting in neurosymbolic AI \cite{DBLP:conf/nips/ManhaeveDKDR18, DBLP:journals/ml/DanieleKSH23}. The learning of the concepts is not driven by direct per-concept supervision, but by background knowledge, in our case the temporal specification $\Phi$, together with the resulting sequence-level binary label $y$.

\subsection{Framework}
\label{sec:framework}

\begin{figure}[bt]
\centering
  % trim = {Left Bottom Right Top}, clip removes the trimmed area
  \includegraphics[trim={.7cm .3cm .7cm .3cm}, clip, width=1.\linewidth]{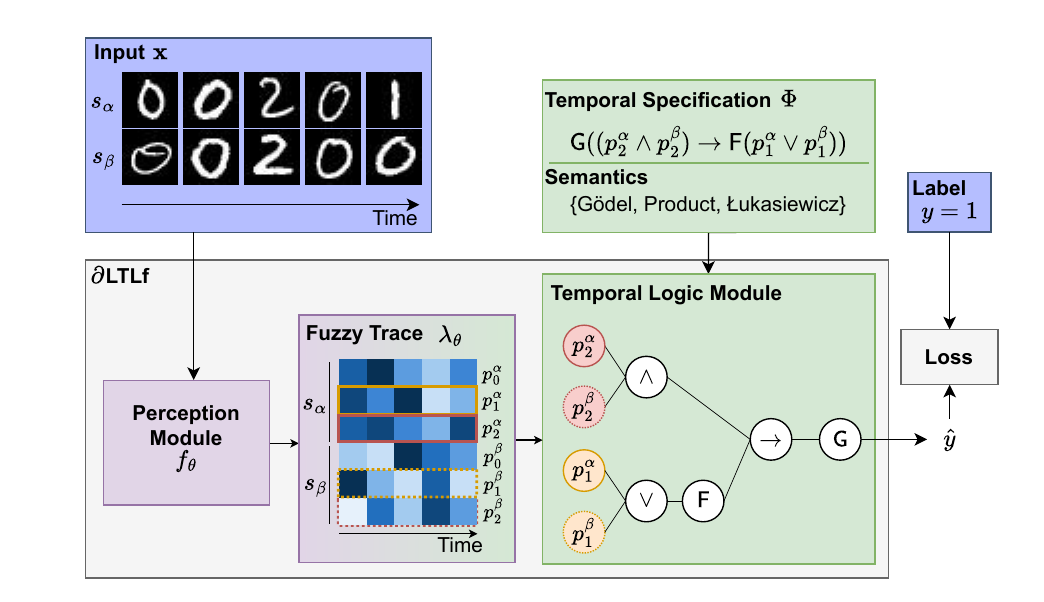}
  % \caption{Overview of the \DiffLTLf pipeline.}
  \caption{Overview of the \DiffLTLf framework. The perception module $f_\theta$ produces the fuzzy trace $\lambda_\theta$ (rows represent the propositions $p^s_c$ grouped by stream). The temporal logic module evaluates $\Phi$ under the \FLTLF[x] semantics. The loss evaluates the predictions of the temporal logic module in comparison with the ground truth label.}
  \label{fig:pipeline}
\end{figure}

\DiffLTLf is a neurosymbolic framework that, given an observation sequence $\mathbf{x}$ and a temporal specification $\Phi$, computes a fuzzy satisfaction value for formula $\Phi$. Internally, it first computes the fuzzy truth values of propositions $\mathcal{P}$ through a perception module $f_\theta$, and then it feeds them to the temporal logic module to compute the final truth value. 
Since the framework is differentiable almost everywhere, it can be trained end-to-end directly by applying a loss function to its outputs.
% The \DiffLTLf framework is a differentiable pipeline that, given an observation sequence $\mathbf{x}$ and a temporal specification $\Phi$, computes a fuzzy satisfaction value used to optimize the perception module parameters $\theta$.

The entire pipeline of \DiffLTLf is illustrated in Figure~\ref{fig:pipeline} and consists of three components: perception module, temporal logic module, and loss.

\paragraph{Perception module} The perception module $f_\theta$ is applied independently to each observation $x_i^s$ in every stream and at every instant. 
Each application of the perception module to an observation produces a vector in $[0,1]^{|\mathcal{C}|}$ whose $c$-th element is the fuzzy truth value assigned to the atomic proposition $p_c^s$ at that instant:
\begin{equation}
\lambda_i(p_c^s) = f_\theta(x_i^s)_c, \quad \text{for } c \in \mathcal{C},\ s \in \mathcal{S},\ i \in \{0, \dots, n-1\}.
\end{equation}
In the general formulation, atoms $p^s_c$ are not assumed to be mutually exclusive. When mutual exclusivity holds, i.e., each observation represents a single class, a softmax output layer can be applied to $f_\theta$, giving $\sum_{c \in \mathcal{C}}f_\theta(x_i^s)_c=1$.
% The collection of these vectors, across all streams and instants of $\mathbf{x}$, defines a fuzzy trace $\lambda_\theta$ in the sense of Definition~\ref{def:traces}.
Collecting these values across all instants yields, for each atom $p_c^s$, the sequence $\lambda(p^s_c) = \langle f_\theta(x_0^s)_c, \dots, f_\theta(x_{n-1}^s)_c\rangle$ in the sense of Definition~\ref{def:traces}. 
Across all atoms and streams, these sequences form the fuzzy trace $\lambda_\theta$.

\paragraph{Temporal logic module} The temporal specification $\Phi$ is represented by its syntactic tree, shown in Figure \ref{fig:pipeline}, whose leaves are the atomic propositions in $\mathcal{P}$ and whose internal nodes are the connectives and temporal operators occurring in $\Phi$. 
The formula is evaluated directly on the fuzzy trace $\lambda_\theta$ by application of the \FLTLF[x] semantics defined in Section~\ref{sec:fuzzyLTL}, under any of its three fuzzy semantic instances, without any intermediate automaton construction, proceeding from the leaves of the tree to its root.
The output is the value $v(\Phi, \lambda_\theta, 0) \in [0,1]$, which represents the degree of satisfaction of $\Phi$ at the beginning of the trace. For conciseness we denote it with $v(\Phi, \lambda_\theta)$, as in Section~\ref{sec:fuzzyLTL}.
The evaluation is differentiable almost everywhere with respect to the perception outputs, allowing the perception module to be trained using standard automatic differentiation. 
% The evaluation is differentiable almost everywhere with respect to the perception outputs, since every operator in the \FLTLF[x] semantics is differentiable almost everywhere with respect to its arguments.

\paragraph{Loss} The fuzzy satisfaction value $v(\Phi, \lambda_\theta)$ is compared against the binary label $y$ via binary cross-entropy. Gradients of the resulting loss are backpropagated through the temporal logic module and the perception module to update $\theta$.

\medskip
In contrast to existing temporal neurosymbolic approaches based on \LTLF~\cite{DBLP:conf/kr/UmiliCG23, DBLP:conf/ijcai/ManginasPR25}, which compile $\Phi$ into deterministic finite automata and propagate state values across the input sequence, \DiffLTLf evaluates $\Phi$ by recursive application of the \FLTLF[x] semantics directly on the fuzzy trace produced by the perception module.
Our use of fuzzy semantics to evaluate logical formulas is shared with static neurosymbolic frameworks such as LTN \cite{ltn} and SBR \cite{SBR}, but the setting is different.
There, the formula is assumed to hold for all instances, and its satisfaction is maximised during training.
Here $\Phi$ is not assumed to hold, and its satisfaction is the prediction of the model.
Following \citet{informed_ML}, this makes those frameworks loss-based and \DiffLTLf model-based, with $\Phi$ acting as a differentiable layer that stays in the model at inference time.
\footnote{The \DiffLTLf framework is not limited to this setting, as it can be instantiated in a loss-based fashion, and the theoretical analysis of Section \ref{sec:fuzzyLTL} applies unchanged. We leave such extensions to future work.}

\subsection{Implementation}
\label{sec:implementation}

The temporal logic module is implemented as a Python library built on top of PyTorch\footnote{https://pytorch.org/} and operates on batches of sequences in tensor form.
A batch consists of $B$ sequences, indexed by $b \in \{0, \dots, B-1\}$, of possibly different lengths $n_b$.
Representing a batch as a single tensor allows every step of the evaluation to be carried out by one tensor operation over all its sequences at once, but requires them to share a common length.
We therefore let $T = \max_b n_b$ be the maximum length in the batch, and pad the shorter sequences up to $T$.
Padding is introduced for this purpose alone, and how it is treated is described later in this section.
The batch is then represented by two tensors.
The first, $\Lambda_\theta \in [0,1]^{B\times T \times |\mathcal{P}|}$, collects the fuzzy values produced by the perception module $f_\theta$ on every observation of every sequence. Its slice $(\Lambda_\theta)_b$ is the tensor encoding of the fuzzy trace $\lambda_\theta^{(b)}$ of the $b$-th sequence.
Along the time axis, for a fixed proposition $p$, the slice $(\Lambda_\theta)_{b,:,p}$ encodes $\lambda(p)$ for the $b$-th sequence. Here the colon denotes all the entries along that axis, so that the slice collects the values that $p$ takes at every instant of the sequence.
The second tensor, $\mu \in \{0,1\}^{B \times T}$, is a boolean mask that flags the valid instants with 1 and the padded ones with 0: 
% $\mu_{b,i}=1$ iff $i \le n_b-1$, and $\mu_{b,i}=0$ otherwise.
\begin{equation}
\mu_{b,i} =
\begin{cases}
1 & \text{if } i \le n_b - 1 \\
0 & \text{otherwise.}
\end{cases}
\end{equation}

\begin{algorithm}[h]
\caption{Recursive tensor evaluation for \FLTLF[x] formulas.
}
\label{alg:eval}
\begin{algorithmic}[1]
    \Function{Eval}{$\varphi$, $\Lambda_\theta$, $\mu$}
        \If{$\varphi$ is an atom $p$}
            \State \Return $(\Lambda_\theta)_{:,:,p}$
        \ElsIf{$\varphi = \varphi_1 \land \varphi_2$} \Comment{analogously for $\lor, \neg, \rightarrow$} \label{ln:bool}
            \State \Return $\Call{Eval}{\varphi_1, \Lambda_\theta, \mu} \tand \Call{Eval}{\varphi_2, \Lambda_\theta, \mu}$
        \ElsIf{$\varphi = \X\psi$} \Comment{analogously for $\Xw$ with fill $1$ and last column $1$} \label{ln:next}
            \State $V \gets \Call{Eval}{\psi, \Lambda_\theta, \mu}$
            \State $\widetilde V \gets \text{fill}(V, \mu, 0)$
            \State \Return left shift of $\widetilde V$ by one position, last column $\gets 0$
        \ElsIf{$\varphi = \G\psi$} \Comment{analogously for $\F$ with fill $0$ and $\tor$} \label{ln:always}
            \State $V \gets \Call{Eval}{\psi, \Lambda_\theta, \mu}$
            \State $\widetilde V \gets \text{fill}(V, \mu, 1)$
            \State \Return reverse cumulative $\tand$ of $\widetilde V$ along the time axis
        \ElsIf{$\varphi = \varphi_1 \U \varphi_2$} \Comment{analogously for $\W, \R, \M$} \label{ln:until_start}
            \State $\widetilde V_1 \gets \text{fill}(\Call{Eval}{\varphi_1, \Lambda_\theta, \mu}, \mu, 0)$
            \State $\widetilde V_2 \gets \text{fill}(\Call{Eval}{\varphi_2, \Lambda_\theta, \mu}, \mu, 0)$
            \State $V \gets$ empty tensor of shape $(B, T)$;\ \ $\text{curr} \gets 0$
            \For{$i = T-1$ \textbf{down to} $0$}
                \State $\text{curr} \gets (\widetilde V_2)_{:, i} \tor ((\widetilde V_1)_{:, i} \tand \mathrm{curr})$
                \State $V_{:, i} \gets \text{curr}$
            \EndFor
            \State \Return $V$ \label{ln:until_end}

        \EndIf
    \EndFunction
    
\end{algorithmic}
\end{algorithm}

% The evaluation is carried out by the recursive function \textsc{Eval} of Algorithm \ref{alg:eval}, which visits the syntactic structure of $\Phi$ from the leaves to the root.
% Given a subformula $\varphi$, the call $\textsc{Eval}(\varphi, \Lambda_\theta, \mu)$ returns a tensor $V_\varphi \in [0,1]^{B \times T}$ such that $(V_\varphi)_{b,i} = v(\varphi, \lambda_\theta^{(b)}, i)$ at every valid instant.
% For each sequence $b$ in the batch, the satisfaction value $v(\Phi, \lambda_\theta^{(b)}, 0)$ fed to the loss is the element $(V_\Phi)_{b, 0}$ of the value tensor returned by $\textsc{Eval}(\Phi, \Lambda_\theta, \mu)$.

The recursive evaluator, given as the function \textsc{Eval} of Algorithm \ref{alg:eval}, visits the syntactic structure of $\Phi$ from the leaves to the root, associating to each sub-formula $\varphi$ a tensor $V_\varphi \in [0,1]^{B \times T}$ such that $(V_\varphi)_{b,i} = v(\varphi, \lambda_\theta^{(b)}, i)$ at every valid instant.
For each sequence $b$ in the batch, the satisfaction value $v(\Phi, \lambda_\theta^{(b)}, 0)$ fed to the loss is the element $(V_\Phi)_{b, 0}$ of the value tensor.

For sequences shorter than $T$, the evaluation would be affected by the padded positions. To prevent this, and so to properly manage padding values, we introduce two auxiliary operators $\nabla_0$ and $\nabla_1$.
Applied to an operand value at an instant within $n_b$, they return that value unchanged, while applied past the end of the trace, they return 0 and 1, respectively.
Each temporal operator of \FLTLF[x] is associated with one of the two, determined by its semantics at the boundary.

% To handle uniformly the value of an operand past the end of the trace, we introduce two auxiliary operators $\nabla_0$ and $\nabla_1$. Applied to an operand value at an instant within $n_b$, they return that value unchanged, while applied past the end of the trace, they return 0 and 1, respectively.
% Each temporal operator of \FLTLF[x] is associated with one of the two, determined by its semantics at the boundary.

When a temporal operator associated with $\nabla \in \{\nabla_0, \nabla_1\}$ is applied to a sub-formula with the value tensor $V_\varphi$, the operator first replaces the padded entries of $V_\varphi$ with the value $\nabla$ returns past the end of the trace ($0$ for $\nabla_0$, $1$ for $\nabla_1$), obtaining

\begin{equation}
(\widetilde V_\varphi)_{b,i} =
\begin{cases}
(V_\varphi)_{b,i} & \text{if } \mu_{b,i} = 1 \\
0 & \text{if } \mu_{b,i} = 0 \text{ and } \nabla = \nabla_0 \\
1 & \text{if } \mu_{b,i} = 0 \text{ and } \nabla = \nabla_1
\end{cases}
\end{equation}

It then applies its tensor operation to $\widetilde V_\varphi$.
This substitution is performed by the function $\mathrm{fill}$ of Algorithm \ref{alg:eval}.
Padded positions therefore contribute only this boundary value and cannot affect the result at any valid instant.

Propositional connectives are pointwise and require no padding logic. Their pointwise nature implies that a padded value at $(b,i)$ can affect only the outputs at the same $(b,i)$ in subsequent propositional steps, and it is overwritten by the boundary value of any temporal operator further up.
Since the loss considers $V_\Phi$ only at $i=0$, which is a valid instant in every sequence, what happens at padded positions has no effect on the satisfaction value.

Each operator is computed by a standard tensor operation on its operand(s), one for each branch of \textsc{Eval}.
Propositional connectives apply the elementwise t-norm, t-conorm, and fuzzy negation of the chosen semantics (Algorithm~\ref{alg:eval}, line~\ref{ln:bool}).
$V_{\X\varphi}$ is obtained by shifting $\widetilde V_\varphi$ (with $\nabla = \nabla_0$) one position to the left and filling the last column with $0$ (line~\ref{ln:next}).
$V_{\Xw \varphi}$ is analogous, with $\nabla_1$ and last column filled with $1$.
$V_{\G \varphi}$ at instant $i$ is the iterated t-norm of $\widetilde V_\varphi$ (with $\nabla = \nabla_1$) from $i$ to the end of the trace, computed as a reverse cumulative $\tand$ along the time axis (line~\ref{ln:always}).
$V_{\F \varphi}$ is the iterated t-conorm (with $\nabla = \nabla_0$):
\begin{equation}
(V_{\mathsf{G}\varphi})_{b, i} = 
\bigotimes_{k=i}^{T-1} (\widetilde V_\varphi)_{b, k}, 
\qquad 
(V_{\mathsf{F}\varphi})_{b, i} = \bigoplus_{k=i}^{T-1} (\widetilde V_\varphi)_{b, k}.
\end{equation}

$\varphi_1 \U \varphi_2$ is computed by backward iteration over time on $\widetilde V_{\varphi_1}$ and $\widetilde V_{\varphi_2}$ (both with $\nabla = \nabla_0$), following the recursive computation
\begin{equation}
v(\varphi_1 \U \varphi_2, \lambda, i) = v(\varphi_2, \lambda, i) \tor ( v(\varphi_1, \lambda, i) \tand v(\varphi_1 \U \varphi_2, \lambda, i+1) )
\end{equation}
initialised at the end of the trace with the value returned by $\nabla_0$ (Algorithm \ref{alg:eval}, lines~\ref{ln:until_start}-\ref{ln:until_end}).
The remaining operators of \FLTLF[x] can be handled in the same way.
% Algorithm \ref{alg:eval} summarizes the recursive evaluator.

\begin{example}
Consider $\Phi = \mathsf{G}(a \to \mathsf{F} b)$, interpreted under Gödel semantics, on a trace of length $3$ padded to $T = 5$, so that $\mu = (1,1,1,0,0)$.
% Consider $\Phi = \mathsf{G}(a \to \mathsf{F} b)$ on a trace of true length $3$ padded to $T = 5$, so that $\mu = (1,1,1,0,0)$.
The perception module assigns to $a$ and $b$
% , at the three valid instants, \todo{Ale: secondo me è sbagliato dire che nei tre istanti assegnamo un vettore. Il vettore contiene già tutti e tre, e il padding}
the values $V_a = (0.8,\,0.1,\,0.2,\,*,\,*)$ and $V_b = (0.1,\,0.3,\,0.7,\,*,\,*)$, where $*$ marks padded positions.
% We use Gödel semantics throughout this example.
The evaluator proceeds bottom-up.
The sub-formula $\mathsf{F} b$ uses $\nabla_0$, so its padded entries become $0$, obtaining $\widetilde V_b = (0.1,\,0.3,\,0.7,\,0,\,0)$. Since $0$ is the identity of $\oplus$, the iterated t-conorm is unaffected, and $V_{\mathsf{F} b} = (0.7,\,0.7,\,0.7,\,*,\,*)$.
% at the valid instants.
The implication connective is pointwise, i.e., it is computed independently at each instant,
% The connective $a \to \mathsf{F} b$ is pointwise, that is, it is computed independently at each instant from the values that its two operands take at the same instant.
% Using the material implication equivalence, $a \to \mathsf{F} b$ is $\neg a \lor \mathsf{F} b$,
% \todo{Ale: noi nella parte teorica sulla logica fuzzy definiamo direttamente l'implicazione. Non diciamo che bisogna convertire la formula, quella conversione è implicita. Io salterei questa parte (intendo la conversione da  $a \to \mathsf{F} b$ a $\neg a \lor \mathsf{F} b$)} 
and with $\neg a = (0.2,\,0.9,\,0.8,\,*,\,*)$ this yields $V_{a \to \mathsf{F} b} =(0.7,\,0.9,\,0.8,\,*,\,*)$.
Finally, $\mathsf{G}$ uses $\nabla_1$, filling the padded entries with $1$. Since $1$ is the identity of $\otimes$, the iterated t-norm is unaffected, and $v(\Phi, \lambda) = (V_\Phi)_0 = 0.7$, matching the evaluation on the unpadded trace.
\label{ex:trace_padding}
\end{example}

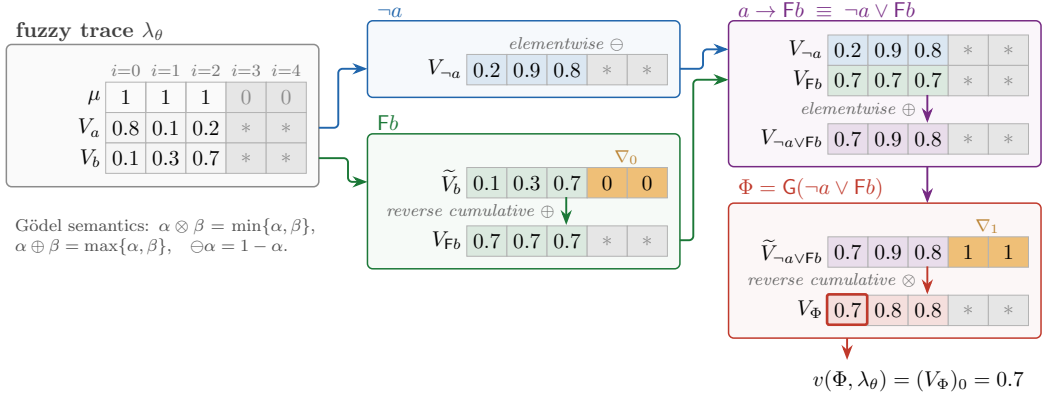
\begin{figure}[t]
\centering
\resizebox{\textwidth}{!}{%
\begin{tikzpicture}[
  >=Stealth, font=\footnotesize,
  cell/.style ={draw=black!30, minimum width=0.66cm, minimum height=0.50cm,
                inner sep=1pt, anchor=center},
  hdr/.style  ={cell, draw=none, minimum height=0.40cm,
                font=\scriptsize, text=black!55},
  pad/.style  ={cell, fill=black!10, text=black!45},
  fillv/.style={cell, fill=cF!70},
  neu/.style  ={cell, fill=black!4},
  va/.style   ={cell, fill=cA!16},
  vb/.style   ={cell, fill=cB!16},
  vo/.style   ={cell, fill=cO!16},
  vg/.style   ={cell, fill=cG!16},
  strip/.style={matrix of nodes, ampersand replacement=\&,
                anchor=north west, inner sep=0pt, nodes={cell},
                row sep=-\pgflinewidth, column sep=-\pgflinewidth},
  lab/.style  ={anchor=east, inner sep=0pt},
  ttl/.style  ={anchor=south west, inner sep=0pt, font=\footnotesize\bfseries},
  op/.style   ={anchor=east, inner sep=0pt, font=\scriptsize\itshape, text=black!60},
  nab/.style  ={font=\scriptsize\bfseries, text=cF!80!black},
  bx/.style   ={rounded corners=2.5pt, line width=0.7pt},
  ar/.style   ={->, line width=0.9pt, rounded corners=2pt},
]

\draw[bx, draw=black!45, fill=black!2] ( 0.00,-0.39) rectangle ( 5.14,-2.73);
\draw[bx, draw=cA,       fill=cA!4   ] ( 5.94, 0.00) rectangle (11.08,-1.26);
\draw[bx, draw=cB,       fill=cB!4   ] ( 5.94,-1.86) rectangle (11.08,-4.08);
\draw[bx, draw=cO,       fill=cO!4   ] (11.88, 0.00) rectangle (17.02,-2.40);
\draw[bx, draw=cG,       fill=cG!4   ] (11.88,-3.00) rectangle (17.02,-5.22);

%%  fuzzy trace 
\node[ttl, text=black!75] at (0.15,-0.34) {fuzzy trace $\lambda_\theta$};
\matrix (T) [strip, row 1/.style={nodes={hdr}}] at (1.62,-0.61) {
  $i{=}0$ \& $i{=}1$ \& $i{=}2$ \& $i{=}3$ \& $i{=}4$ \\
  $1$ \& $1$ \& $1$ \& |[pad]| $0$ \& |[pad]| $0$ \\
  |[neu]| $0.8$ \& |[neu]| $0.1$ \& |[neu]| $0.2$ \& |[pad]| $*$ \& |[pad]| $*$ \\
  |[neu]| $0.1$ \& |[neu]| $0.3$ \& |[neu]| $0.7$ \& |[pad]| $*$ \& |[pad]| $*$ \\
};
\node[lab] at (1.56,-1.26) {$\mu$};
\node[lab] at (1.56,-1.76) {$V_a$};
\node[lab] at (1.56,-2.26) {$V_b$};

%% negation 
\node[ttl, text=cA] at (6.09,0.05) {$\neg a$};
\node[op, anchor=center] at (9.21,-0.38) {elementwise $\tneg$};
\matrix (N) [strip] at (7.56,-0.54) {
  |[va]| $0.2$ \& |[va]| $0.9$ \& |[va]| $0.8$ \& |[pad]| $*$ \& |[pad]| $*$ \\
};
\node[lab] at (7.50,-0.79) {$V_{\neg a}$};

%%  eventually
\node[ttl, text=cB] at (6.09,-1.81) {$\F b$};
\node[nab] at (10.20,-2.24) {$\nabla_0$};
\matrix (F) [strip] at (7.56,-2.40) {
  |[vb]| $0.1$ \& |[vb]| $0.3$ \& |[vb]| $0.7$ \& |[fillv]| $0$ \& |[fillv]| $0$ \\
};
\matrix (Fr) [strip] at (7.56,-3.36) {
  |[vb]| $0.7$ \& |[vb]| $0.7$ \& |[vb]| $0.7$ \& |[pad]| $*$ \& |[pad]| $*$ \\
};
\node[lab] at (7.50,-2.65) {$\widetilde V_b$};
\node[lab] at (7.50,-3.61) {$V_{\F b}$};
\draw[ar, cB] (9.21,-2.90) -- (9.21,-3.36);
\node[op] at (9.03,-3.13) {reverse cumulative $\tor$};

%%  disjunction 
\node[ttl, text=cO] at (12.03,0.05) {$a \to \F b \;\equiv\; \neg a \lor \F b$};
\matrix (O) [strip] at (13.50,-0.22) {
  |[va]| $0.2$ \& |[va]| $0.9$ \& |[va]| $0.8$ \& |[pad]| $*$ \& |[pad]| $*$ \\
  |[vb]| $0.7$ \& |[vb]| $0.7$ \& |[vb]| $0.7$ \& |[pad]| $*$ \& |[pad]| $*$ \\
};
\matrix (Or) [strip] at (13.50,-1.68) {
  |[vo]| $0.7$ \& |[vo]| $0.9$ \& |[vo]| $0.8$ \& |[pad]| $*$ \& |[pad]| $*$ \\
};
\node[lab] at (13.44,-0.47) {$V_{\neg a}$};
\node[lab] at (13.44,-0.97) {$V_{\F b}$};
\node[lab] at (13.44,-1.93) {$V_{\neg a \lor \F b}$};
\draw[ar, cO] (15.15,-1.22) -- (15.15,-1.68);
\node[op] at (14.97,-1.45) {elementwise $\tor$};

%%  globally 
\node[ttl, text=cG] at (12.03,-2.95) {$\Phi = \G(\neg a \lor \F b)$};
\node[nab] at (16.14,-3.38) {$\nabla_1$};
\matrix (G) [strip] at (13.50,-3.54) {
  |[vo]| $0.7$ \& |[vo]| $0.9$ \& |[vo]| $0.8$ \& |[fillv]| $1$ \& |[fillv]| $1$ \\
};
\matrix (Gr) [strip] at (13.50,-4.50) {
  |[vg]| $0.7$ \& |[vg]| $0.8$ \& |[vg]| $0.8$ \& |[pad]| $*$ \& |[pad]| $*$ \\
};
\node[lab] at (13.44,-3.79) {$\widetilde V_{\neg a \lor \F b}$};
\node[lab] at (13.44,-4.75) {$V_{\Phi}$};
\draw[ar, cG] (15.15,-4.04) -- (15.15,-4.50);
\node[op] at (14.97,-4.27) {reverse cumulative $\tand$};
\draw[cG, line width=1.2pt, rounded corners=1pt]
      (13.50,-4.50) rectangle (14.16,-5.00);

%%  output 
\draw[ar, cG] (13.83,-5.22) -- (13.83,-5.58);
\node[anchor=north] at (15.0,-5.60) {$v(\Phi,\lambda_\theta)=(V_\Phi)_0=0.7$};

%%  tree edges
\draw[ar, cA] ( 5.14,-1.76) -- (5.46,-1.76) -- (5.46,-0.79) -- (5.94,-0.79);
\draw[ar, cB] ( 5.14,-2.26) -- (5.62,-2.26) -- (5.62,-2.65) -- (5.94,-2.65);
\draw[ar, cA] (11.08,-0.79) -- (11.46,-0.79) -- (11.46,-0.47) -- (11.88,-0.47);
\draw[ar, cB] (11.08,-3.61) -- (11.30,-3.61) -- (11.30,-0.97) -- (11.88,-0.97);
\draw[ar, cO] (15.15,-2.40) -- (15.15,-3.00);

%%  legend
\node[anchor=north west, font=\scriptsize, text=black!70,
      align=left, text width=5.1cm] at (0.00,-3.10)
     {G\"odel semantics: $\alpha\tand\beta=\min\{\alpha,\beta\}$,\quad
      $\alpha\tor\beta=\max\{\alpha,\beta\}$,\quad $\tneg\alpha=1-\alpha$.};

\end{tikzpicture}}
\caption{Computation of $v(\G(a\to \F b),\lambda)$, from Example~\ref{ex:trace_padding}.}
\label{fig:eventually_pipeline}
\end{figure}

% Formally, given the observation sequence $\mathbf{x} = \langle \mathbf{x}_0, \dots, \mathbf{x}_n  \rangle$ with $\mathbf{x}_i = \langle x_i^1, \dots, x_i^{|S|} \rangle$ at instant $i$ specifying the observation in each stream $s$, then the output of the perception module is a fuzzy trace 
% %
% $$ \lambda_\theta = \langle \lambda_0, \dots, \lambda_n \rangle $$
% where for each instant $i\in[0,n]$ and each observation $x^s_i$ in stream $s$: 
% $$ \lambda_i(p^s_c) = f_\theta(x^s_{i})_c \text{ for each } c\in \mathcal{C}$$

% \noindent
% The result is a complete fuzzy trace $\lambda$ as defined in Section~\ref{sec:fuzzyLTL}, where for each instant $i$ each proposition $p\in \mathcal{P}$ is assigned a value in $[0,1]$. 
% Therefore, we are able to evaluate \FLTLF formulae on these fuzzy traces, and compute the error with respect to label $y$ associated to the original input trace.

%\todo[inline]{P:commentato via sezione su proposiz. con endinput}

\section{Evaluation}
\label{sec:evaluation}
%% Labels are used to cross-reference an item using \ref command.

% Outline:
% 6.1 Research Questions
% 6.2 Benchmark Construction
% - multi-stream temporal banchmark (relation with \\LTLZinc)
% - declare patterns
% - image datasets

% 6.3 Experimental configurations
% 6.3.1
% - varying formula complexity
% - varying sequence length

% 6.4 Baselines
% - \FuzzyDFA (explain here the use of \\LTLZinc variant)
% - \NeSyA
% - \TILR (?)
% - Neural Baselines (dire che rimpiazzano il logic module)

% 6.5 Metrics
% - symbol grounding accuracy
% - seq class accuracy
% - logic module time per epoch
% - 10 independent runs, test statistico

In this section, we evaluate \DiffLTLf on a temporal neurosymbolic learning task involving multi-stream perception traces and complex temporal dependencies. We compare our approach against two state-of-the-art approaches based on deterministic finite-state automata (DFA), introduced in Section~\ref{sec:conceptual}, namely \FuzzyDFA~\cite{DBLP:conf/kr/UmiliCG23} and \NeSyA~\cite{ DBLP:conf/ijcai/ManginasPR25}.
Our evaluation focuses on two complementary aspects: (i) performance in terms of sequence-level classification accuracy, and symbol grounding quality under weak supervision; (ii) computational scalability with respect to increasing formula complexity and sequence lengths.

%We do not compare only \DiffLTLf with the SOTA approaches but also compare the different proposed sematincs within DiffLTLF, should mention above
To structure the experimental analysis, we will address the following research questions:

\begin{enumerate}[noitemsep]
    \item \textbf{RQ1: Impact of fuzzy temporal semantics.} 
    How do different fuzzy \LTLF semantics used in \DiffLTLf affect the sequence classification accuracy and symbol grounding performance of the proposed approach?

    \item \textbf{RQ2: Comparison with state of the art.}
    How does the proposed approach compare against existing state-of-the-art temporal neurosymbolic integration methods in terms of classification accuracy and symbol grounding performance?

    \item \textbf{RQ3: Scalability.}
    How do different temporal NeSy approaches scale in terms of runtime complexity as formula complexity and  sequence length increase?
\end{enumerate}

\textbf{RQ1} aims to determine the impact of the different fuzzy semantics introduced in Section~\ref{sec:fuzzyLTL} within the \DiffLTLf framework on classification performance at both the sequence level and the symbol grounding level. While comparisons of fuzzy semantics have previously been conducted in non-temporal domains~\cite{DBLP:journals/ai/KriekenAH22}, their impact in temporal settings remains largely unexplored.

For \textbf{RQ2}, we compare the best-performing fuzzy semantics with state-of-the-art temporal neurosymbolic approaches in terms of classification accuracy at both the sequence level and the individual symbol grounding level. To this end, we vary both the complexity of the formulas and the length of the input sequences. This allows us to assess how the different methods perform not only in terms of predictive accuracy but also in their ability to scale with increasing formula complexity and input length.

For \textbf{RQ3}, similarly to \textbf{RQ2}, we evaluate how the proposed method compares with state-of-the-art approaches in terms of computational cost during training when increasing sequence length and formula complexity. Specifically, we measure the runtime of the logic modules per training epoch for each technique.

\subsection{Temporal Multi-Stream Benchmark}
\label{sec:benchmark}
We evaluate \DiffLTLf on a temporal multi-stream benchmark, which represents a specific instantiation of the problem defined in Section~\ref{sec:diffLTLf}. 
Our evaluation builds on the \LTLZinc benchmarking framework~\cite{ltlzincIJCAI}, which generates temporal sequence classification tasks from \LTLF specifications and static relational constraints. 
We introduce three deliberate modifications that adapt the framework to our evaluation objectives.

First, we remove all relational constraints from the specification. \LTLZinc tasks interleave temporal operators with relational predicates, resulting from arithmetic comparisons or global constraints (e.g. \texttt{all\_different}).
While this design is appropriate for evaluating the joint effect of relational and temporal reasoning, it may confound the two distinct reasoning capabilities. A method that fails on such a task may do so because of flawed temporal reasoning, flawed relational reasoning, or both. The evaluation cannot distinguish between the two cases.
By restricting our specifications to purely temporal formulas, we isolate a model's temporal reasoning capabilities and obtain a more sound evaluation.
Second, we increase the complexity of the temporal specifications beyond the level considered in prior work.
Third, we adopt a strictly weaker form of supervision. The \LTLZinc benchmark provides annotations at multiple levels: ground-truth image labels at every timestep, relational constraint satisfaction, automaton states, in addition to sequence-level satisfaction label.
Accordingly, the training pipeline proposed by \LTLZinc employs four separate loss terms that supervise image classification, relational constraint classification, next-state prediction, and sequence classification independently. In addition, it employs a pre-training phase with direct supervision on the image labels to ensure convergence~\cite{ltlzincIJCAI}.

In contrast, our setting provides only the binary satisfaction label of the temporal formula as supervision. In other words, the neurosymbolic system must solve the symbol grounding problem entirely from this sequence-level supervisory signal, without access to intermediate annotations.
This choice reflects a more realistic scenario, since in practical applications ground-truth symbolic traces, automaton states, and constraint labels are typically unavailable. 
Moreover, supervision on automaton states is intrinsically DFA-based and it is not applicable to methods such as \DiffLTLf that evaluate temporal formulas directly. As a result, \LTLZinc's evaluation protocol restricts fair comparison to DFA-based approaches only.

% Unlike previous benchmarks that rely on supervision in terms of both intermediate relational constraints within temporal specifications  and further supervision at the symbol grounding level~\cite{ltlzincIJCAI}, our porposed setting requires the neurosymbolic system to reason directly over temporal properties of the raw perceptual streams, resolving the symbol grounding problem without intermediate supervision. This, coupled with the fact of having multiple input streams, rather than a single one~\cite{DBLP:conf/kr/UmiliCG23}, greatly increases the complexity of the learning task.
% \todo{capire come concludere frase e aggiungere che noi, a diff di \\LTLZinc, non diamo supervisione intermedia}.

In our proposed benchmark, the input sequences consist of $S=3$ parallel streams, denoted as $(X, Y, Z)$, each receiving a sequence of images, as in \LTLZinc~\cite{ltlzincIJCAI}.
The temporal specifications used in our experiments are built from the \LTLF templates introduced in Section~\ref{sec:background} and listed in Table~\ref{tab:declare}, which have been previously adopted within the experimental setting of temporal NeSy architectures~\cite{DBLP:conf/kr/UmiliCG23,DBLP:conf/nesy/AndreoniBDGMR25,DBLP:conf/ecai/UmiliC24}.
Based on these eight templates, we define two separate sets of formulas in order to distinguish between different classes of temporal dependencies and to assess their impact independently. In particular, the partition separates constraints that primarily capture forward-looking (response-type) relations from those that enforce backward-looking or co-occurrence (precedence-type) relations, which differ in both semantics and verification complexity.

The first set, Set~1, combines the first four patterns from Table~\ref{tab:declare}, namely Response, Not Succession, Chain Succession, and Not Chain Succession templates. Conversely, Set~2 combines the last four templates: Precedence, Chain Precedence, Responded Existence, and Not Co-existence.

Each formula set is paired with two image datasets, which are also used within the \LTLZinc benchmark~\cite{ltlzincIJCAI}: MNIST \cite{lecunmnist} and Fashion-MNIST~\cite{fashionmnist}, resulting in four experimental cases in total.

\subsubsection{Varying Formula Complexity}
\label{sec:formula_complexity}
In this evaluation direction, the temporal specification $\Phi$ is defined as a conjunction of $K$ independent temporal sub-formulas, $\Phi = \bigwedge_{k=1}^K\varphi_k$%~\footnote{Independent sub-formulas refers to the fact that there is no overlap within the atoms of each of the sub-formulas.}
, where each sub-formula $\varphi_k$ is an instantiation of a template from Table~\ref{tab:declare}. % and predicates on a disjoint subset of classes and streams.
% introduces two distinct classes, not shared with any other sub-formula
% Each sub-formula $\varphi_k$ predicates on a disjoint subset of classes and streams.
Since each sub-formula introduces two additional classes, the total number of classes in the temporal specification grows as $|\mathcal{C}| = 2K$. By varying $K \in \{1,2,3,4\}$, we increase the complexity of the temporal specification and, as a consequence, of the learning task.

%We define two formula sets. Set~1 combines Response, Not Succession, Chain Succession, and Not Chain Succession templates. Set~2 combines Precedence, Chain Precedence, Responded Existence, and Not Co-existence templates.

% Each formula set is associated with two image dataset, MNIST \cite{lecunmnist} and Fashion-MNIST~\cite{fashionmnist} resulting in four experimental settings in total.
The dataset for each setting is composed of sequences of varying lengths, sampled uniformly from the integer interval $[2, 10]$.
Sequence length is deliberately kept short and fixed in distribution across all values of $K$, so that variation in performance is attributable to formula complexity alone, and not confounded by sequence length,
%whose effect is studied in Section \ref{sec:varying_length}.
whose effect is examined using the setting described in Section \ref{sec:varying_length}.

% To give an example of this, Figure~\ref{fig:pipeline} illustrates the task on a Set~1 instance with $K=2$: three streams of MNIST images are processed by the perception module $f_\theta$ and evaluated against the temporal specification $\Phi$, with only the binary satisfaction label $y$ as supervision. The temporal specification in this case, since $K =2 $, will be composed of a conjunction of the Response and Not Succession patterns.
As an example, for Set 1 with $K=2$ the specification is $\Phi=\always (p^X_0 \to \eventually p^Y_1) \land G((p^Y_2 \lor p^Z_3)\to \neg\eventually  p^X_3)$, which conjoins the Response and Not Succession patterns.

% By using the different templates in Table~\ref{tab:declare}, we can evaluate \DiffLTLf not only by increasing the complexity of the temporal specifications, but also in terms of determining how different patterns, which express different temporal relations (and different time horizons, such as $\eventually$ or $\tomorrow$) impact on the performance of the framework. 
% Furthermore, as we will see next, we also aim to test how increasing the length of the input sequences affects the performance of the learned models (cf. Section~\ref{sec:varying_length}).

%%%%%%%%%%%%%%%%%%%%%%%%%%%%%%%%%%%%%%
\subsubsection{Varying Sequence Length}
\label{sec:varying_length}
% In this evaluation direction, the temporal specification is fixed at the highest complexity level, involving $|\mathcal{C}| = 8$ classes distributed across the three streams.
% The specification used here has the form $\Phi = (\varphi_1 \land \varphi_2 ) \lor (\varphi_3 \land \varphi_4)$, where each $\varphi_k$ is an instantiation of the Precedence pattern.
% The image dataset is MNIST \cite{lecunmnist}. The sequence length varies from 20 to 200 in steps of 20. Unlike the previous setting, all sequences within each configuration share the same fixed length, so as to isolate the effect of sequence length on learning performance.
In this evaluation direction, the temporal specification is fixed, and the sequence length varies from 20 to 200 in steps of 20. Unlike the previous setting, all sequences within each configuration share the same fixed length, so as to isolate the effect of sequence length on learning performance. The image dataset is MNIST~\cite{lecunmnist}.

We consider two specifications of different complexity, based on either the Response or Precedence template from Table~\ref{tab:declare}. 
The first involves $|\mathcal{C}| = 4$ classes and is defined as a conjunction of two Response sub-formulas. The second involves $|\mathcal{C}| = 8$ classes and takes the form $\Phi = (\varphi_1 \wedge \varphi_2) \vee (\varphi_3 \wedge \varphi_4)$, where each $\varphi_k$ is an instantiation of the Precedence pattern. 
\footnote{
When traces are sampled uniformly at random, the fraction of positive traces in the generated dataset generally depends on the trace length, and for most templates in Table~\ref{tab:declare} it degenerates towards $0$ or $1$ as the length grows.
The Response and Precedence templates do not show this behaviour, as their satisfaction probability converges to a constant as the trace length increases.
The way the sub-formulas are combined is chosen for the same reason, that is, to obtain datasets that are as balanced as possible.
}
Using two specifications allows us to observe whether the effect of increasing sequence length depends also on the complexity of the temporal specification itself.

\subsection{Evaluation Metrics}
We use two accuracy metrics, both computed on the held-out test set. 
Symbol grounding accuracy is the fraction of correctly classified images. Each observation $x_i^s$ is assigned to the class $\arg \max_c f_\theta (x_i^s)_c$ and compared against its ground-truth class, and averaged over all observations, streams, and sequences.
Sequence classification accuracy is the fraction of sequences whose predicted satisfaction matches the label $y$. A sequence is predicted to satisfy $\Phi$ when $v(\Phi, \lambda_\theta) \geq 0.5$.
The per-image ground-truth labels serve only to evaluate symbol grounding performance and are never used during training, consistently with the weakly supervised setting of Section \ref{sec:task_overview}.
Finally, computational cost is measured as the logic-module time per training epoch.

                \section{Results}
\label{sec:results}
%% Labels are used to cross-reference an item using \ref command.

% Struttura:
% - preambolo (poche frasi) per richiamare le RQs, dire che i risultati sono la media di 10 runs, il test statistico usato, e gli hyperparams
% - parte 1 (impatto della semantica fuzzy) RQ1
%     - confronto di goedel, prod, luk, sugli esperimenti formula-complessità
%     - risultato: product è la più consistent. le altre ottengono ottimi risultati con 2 classi ma degradano con l'aumentare della complessità della formula e hanno molta varianza (minimi locali). product è stabile
%     - conclusione: selezioniamo product come semantica rappresentativa per RQ2 e RQ3
% - parte 2 (confronto con lo stato dell'arte) RQ2-3
%     - mostriamo che \DiffLTLf (product) ha pari o superiore acc dei competitors
%     - mostriamo che è molto più veloce
% - limiti:
%     - benchmark sintetico
%     - dimensione del dataset (ma giustificabile dicendo che i metodi dfa-based hanno il limite di non poter trattare dataset troppo vasti, per problemi di scalabilità)
%     - abbiamo testato solo un tipo di perception (CNN). In realtà non è un grosso limite dato che stiamo confrontando il modulo logico

This section presents the experimental results structured around the three research questions introduced in Section \ref{sec:evaluation}.
\textbf{RQ1} (Section \ref{sec:results_rq1}) investigates the impact of the fuzzy semantics on the sequence classification and on the symbol grounding accuracy performance within \DiffLTLf. %Next,
\textbf{RQ2} (Section \ref{sec:results_rq2}) compares the best-performing \DiffLTLf semantic configuration against state-of-the-art temporal neurosymbolic methods in terms of classification accuracy at both the sequence and symbol grounding level.
\textbf{RQ3} (Section \ref{sec:results_rq3}) evaluates the computational scalability of the compared methods.
The three subsections share the same structure: each one first provides a direct answer to the corresponding research question, supported by the experimental evidence, and then discusses the results in greater depth, providing additional insights.
All results are averaged over 10 independent runs. As all methods share seeds, splits, and sampled sequences, differing only in the logic module, statistical significance is assessed 
with a paired one-sided Wilcoxon signed-rank test. 
Throughout, results marked in bold denote a method that is significantly better than all
competitors at $p<0.05$.
The complete experimental setup, including perception architecture, optimiser, learning rate, batch size, number of training epochs, dataset construction, and hardware, is identical across all compared methods and is reported in \ref{app:exp_setup}.

\subsection{RQ1: Impact of Fuzzy Temporal Semantics}
\label{sec:results_rq1}

\begin{figure}[t]
  \centering
  \includegraphics[width=\textwidth]{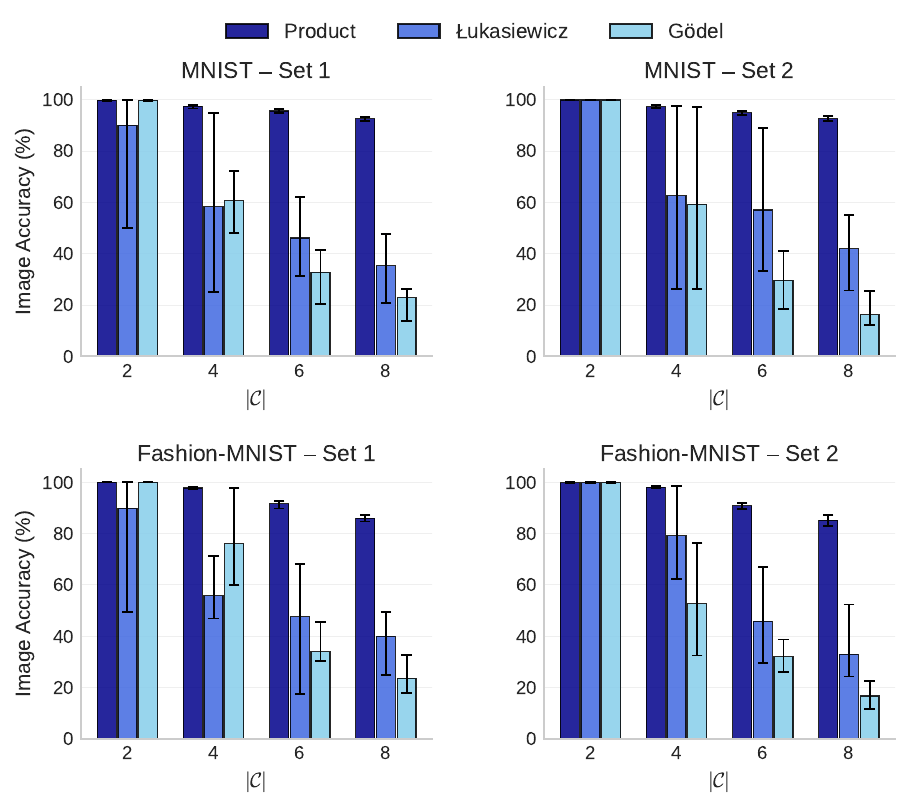}
  \caption{Symbol grounding accuracy (\%) across 
  10 independent runs for \DiffLTLf instantiated with the three fuzzy semantics: Product (\FLTLF[P]), Łukasiewicz (\FLTLF[L]), and Gödel 
  (\FLTLF[G]). 
  Each subplot corresponds to one experimental setting (formula 
  set $\times$ image dataset), with the number of classes $|\mathcal{C}|$ on the horizontal axis.}
  \label{fig:acc_semantics}
\end{figure}

We evaluate the effect of the underlying fuzzy semantics on the classification performance of \DiffLTLf. 
Specifically, we compare the three instantiations of \FLTLF introduced in Section \ref{sec:fuzzyLTL}: Gödel (\FLTLF[G]), Product (\FLTLF[P]), and Łukasiewicz (\FLTLF[L]). 
% The comparison is conducted on the formula-complexity benchmark described in Section \ref{sec:evaluation_formula}, 
% across all four experimental settings (formulas set 1 and 2, on both MNIST and Fashion-MNIST datasets), with the number of conjuncted sub-formulas ranging from $K=1$ to $K=4$, and correspondingly the number of classes $|\mathcal{C}| = 2K$ ranging from $2$ to $8$.
The comparison is conducted across all four experimental settings of Section~\ref{sec:formula_complexity}, with the number of classes $|\mathcal{C}|$, equal to twice the number of conjoined sub-formulas, ranging from $2$ to $8$.
Figure \ref{fig:acc_semantics} shows the symbol grounding accuracy attained by each semantics at each value of $|\mathcal{C}|$.
% across runs for each semantics and each value of $|\mathcal{C}|$. 
Table \ref{tab:semantics_summary} reports the average of the symbol grounding and the sequence classification accuracy over the four settings.

\begin{table}[t]
\centering
\footnotesize
\caption{Symbol grounding and sequence classification accuracy (\%) of \DiffLTLf under the three
fuzzy semantics, averaged over the four experimental settings and 10 runs. Results in \textbf{bold} are significantly better than both other semantics (Wilcoxon, $p<0.05$).}
\label{tab:semantics_summary}
\begin{tabular}{l cccc cccc}
\toprule
& \multicolumn{4}{c}{\textbf{Img. Acc.} (\%) $\uparrow$}
& \multicolumn{4}{c}{\textbf{Seq. Acc.} (\%) $\uparrow$} \\
\cmidrule(lr){2-5} \cmidrule(lr){6-9}
$|\mathcal{C}|$ & 2 & 4 & 6 & 8 & 2 & 4 & 6 & 8 \\
\midrule
\FLTLF[P] & 99.9 & \textbf{97.7} & \textbf{93.3} & \textbf{89.2} & 99.9 & \textbf{98.9} & \textbf{97.0} & \textbf{96.0} \\
\FLTLF[G] & 99.9 & 62.2 & 32.1 & 19.9 & 99.9 & 90.0 & 88.6 & 86.2 \\
\FLTLF[L] & 94.9 & 64.0 & 49.2 & 37.6 & 96.6 & 90.6 & 89.3 & 87.7 \\
\bottomrule
\end{tabular}
\end{table}

\subsubsection{Answer to RQ1}
\label{sec:answer_rq1}

The choice of the fuzzy semantics has a substantial impact on the accuracy performance of \DiffLTLf, and Product is the only semantics that remains robust as the complexity of the temporal specification grows.

\FLTLF[P] maintains consistently high symbol grounding accuracy, with low variance, across all values of $|\mathcal{C}|$ and across all settings. Its average accuracy remains above $92\%$ and $85\%$ for the MNIST and Fashion-MNIST datasets, respectively, even at $|\mathcal{C}|=8$. The low variance throughout the 10 runs indicates that convergence to a correct grounding does not depend on the initialisation.

Gödel and Łukasiewicz semantics, in contrast, are competitive only on the simplest specifications. At $|\mathcal{C}|=2$, where the temporal specification involves only two classes, all three semantics achieve comparable performances, with average grounding accuracy above $99\%$ in almost all settings. Differences emerge as the formula complexity increases, as both Gödel and Łukasiewicz exhibit a significant increase in variance. For $|\mathcal{C}|=4$, while these two semantics occasionally reach near-optimal levels of performance, many runs converge to lower values. For $|\mathcal{C}|\geq 6$, the degradation of performance becomes more severe.
This pattern suggests sensitivity to local optima during training.
Sequence classification accuracy follows the same ranking, but with much smaller differences (Table \ref{tab:semantics_summary}), as sequence classification is an easier task than symbol grounding.

Based on these findings, all subsequent experiments adopt \FLTLF[P] Product semantics as the default configuration of \DiffLTLf.
Having answered RQ1, the remainder of this subsection discusses these results in greater depth and provides further insights.

% \begin{table}[t]
% \centering
% \caption{}
% \label{tab:semantics_summary}
% \begin{tabular}{l cccc cccc}
% \toprule
% & \multicolumn{4}{c}{\textbf{Img. Acc.} (\%) $\uparrow$}
% & \multicolumn{4}{c}{\textbf{Seq. Acc.} (\%) $\uparrow$} \\
% \cmidrule(lr){2-5} \cmidrule(lr){6-9}
% $|\mathcal{C}|$ & 2 & 4 & 6 & 8 & 2 & 4 & 6 & 8 \\
% \midrule
% \FLTLF[G] & 99.9 & 62.2 & 32.1 & 19.9 & 99.9 & 90.0 & 88.6 & 86.2 \\
% \FLTLF[P] & 99.9 & \textbf{97.7} & \textbf{93.3} & \textbf{89.2} & 99.9 & \textbf{98.9} & \textbf{97.0} & \textbf{96.0} \\
% \FLTLF[L] & 94.9 & 64.0 & 49.2 & 37.6 & 96.6 & 90.6 & 89.3 & 87.7 \\
% \bottomrule
% \end{tabular}
% \end{table}

\subsubsection{Discussion on RQ1}
\label{sec:discussion_rq1}

The pattern observed above is consistent with prior findings in non-temporal neurosymbolic settings, where Product logic has been shown to provide more informative gradients than Gödel and Łukasiewicz for constraint satisfaction tasks \cite{DBLP:journals/ai/KriekenAH22}. The present results extend this observation to the temporal domain and can be explained by analysing the gradient properties of the underlying operators. 
In our setting, the temporal operators $\always$ and $\eventually$ reduce to iterated applications of the t-norm and t-conorm respectively, over the trace length.

Under \textbf{Product semantics}, the t-norm $\alpha \tand \beta = \alpha \cdot \beta$ has partial derivative $\beta$ with respect to $\alpha$, and the t-conorm $\alpha \oplus \beta = \alpha + \beta - \alpha \cdot \beta$ has partial derivative $1 - \beta$.
Both are non-vanishing on almost the entire input domain. As a result, when $\always$ is evaluated over a trace of length $n$, 
% every input at each timestep 
every timestep receives a non-zero gradient, enabling the perception module to update all its predictions in each training step.

Under \textbf{Gödel semantics}, the $\min$ and $\max$ operators are single-passing.
% At most one of their arguments receives a non-zero gradient.
As shown in~\cite{DBLP:journals/ai/KriekenAH22}, any composition of single-passing functions is itself single-passing. 
As a consequence, the entire computation graph derived from a temporal formula inherits this property.
Concretely, given a sequence of $n$ observations, exactly one receives a non-zero gradient. For this reason, convergence becomes sensitive to which single observation happens to receive the learning signal.
This makes learning increasingly hard as the formula's complexity grows,
% \todo[inline]{Ale: Perché? Secondo il ragionamento, dovrebbe cambiare rispetto alla lunghezza delle sequenze, non rispetto al numero di elementi nell'alfabeto. O si intende qui rispetto alla complessità della formula? Cioè, formula più lunga, dunque maggiore tempo alla convergenza. A proposito: qui la dimensione delle sequenze sono variabili.. siamo sicuri che non influiscano sulla media dei valori}
 explaining the performance degradation and high variance observed in Figure~\ref{fig:acc_semantics}.

Under \textbf{Łukasiewicz semantics}, the t-norm $\alpha \otimes \beta = \max(\alpha + \beta - 1, 0)$ has gradient $1$ when $\alpha + \beta > 1$ and $0$ otherwise.
For the $n$-ary case, which corresponds to the evaluation of $\always$ over a trace of length $n$, van Krieken et al.~\cite{DBLP:journals/ai/KriekenAH22} show that the gradient is non-vanishing only when the mean truth value of the $n$ arguments is greater than $(n-1)/n$, and that the fraction of the input space satisfying this condition is $1/n!$. 
This condition is particularly difficult to satisfy, especially during early training, when the perception module produces output with intermediate truth values.
The t-conorm $\min(\alpha + \beta, 1)$ shows a symmetrical issue, with its gradient vanishing when the sum of its arguments exceeds $1$. 
As $|\mathcal{C}|$ increases, the outer conjunction further stresses the problem. This explains the performance degradation at higher $|\mathcal{C}|$ (Figure~\ref{fig:acc_semantics}).

\subsection{RQ2: Comparison with State of the Art}
\label{sec:results_rq2}

Following the outcome of RQ1, we instantiate \DiffLTLf with Product semantics and compare it
against two state-of-the-art temporal neurosymbolic methods: \FuzzyDFA\footnote{We use the \FuzzyDFA implementation from the \LTLZinc benchmark \cite{ltlzincIJCAI}, which deviates from the original in state normalisation and training loss; see \ref{app1} for a precise description.} \cite{DBLP:conf/kr/UmiliCG23} and \NeSyA \cite{DBLP:conf/ijcai/ManginasPR25}.
The comparison is conducted along two complementary directions: increasing the formula complexity (Table~\ref{tab:formula_complexity}) and increasing the sequence length (Table~\ref{tab:seq_length}).
Both tables also report the logic module time per epoch, which is discussed separately in Section \ref{sec:results_rq3}.
Figure~\ref{fig:formula_complexity} summarises the formula-complexity results averaged over the four settings. It additionally includes four further methods: \TILR \cite{DBLP:conf/nesy/AndreoniBDGMR25}, the original \FuzzyDFA implementation, and two purely neural baselines, obtained by replacing the logic module with a GRU \cite{DBLP:conf/emnlp/ChoMGBBSB14} and with a Transformer encoder \cite{DBLP:conf/nips/VaswaniSPUJGKP17}, respectively.
% Additionally, we include two purely neural baselines, obtained by replacing the logic module with a GRU \cite{DBLP:conf/emnlp/ChoMGBBSB14} and with a Transformer encoder \cite{DBLP:conf/nips/VaswaniSPUJGKP17}, respectively. Their complete numerical results are reported in \ref{app:full_results}.
All of them fall substantially behind the three neurosymbolic methods compared above, increasingly so as the specification becomes more complex. For this reason, we do not include them in the tables of this section, nor in the sequence-length comparison, and we provide their complete numerical results in \ref{app:full_results}.

\begin{table*}[ht]
\centering
\caption{Performance comparison varying formula complexity. Results are averaged over 10 independent runs.
\textbf{Img. Acc}: Symbol Grounding Accuracy; 
\textbf{Seq. Acc}: Sequence Classification Accuracy; 
\textbf{Time}: Logic Module Time per Epoch. 
Results in \textbf{bold} are statistically significantly better than all other methods (Wilcoxon, $p<0.05$).}
\resizebox{\textwidth}{!}{
\begin{tabular}{lc ccc ccc ccc}
\toprule
& & \multicolumn{3}{c}{\textbf{Img. Acc.} ($\%$) $\uparrow$} & \multicolumn{3}{c}{\textbf{Seq. Acc.} ($\%$) $\uparrow$} & \multicolumn{3}{c}{\textbf{Time} (s) $\downarrow$} \\
\cmidrule(lr){3-5} \cmidrule(lr){6-8} \cmidrule(lr){9-11}
Setting & $|\mathcal{C}|$ & \DiffLTLf & \FuzzyDFA & \NeSyA & \DiffLTLf & \FuzzyDFA & \NeSyA & \DiffLTLf & \FuzzyDFA & \NeSyA \\
\midrule
\multirow{4}{*}{\shortstack[l]{MNIST\\Set1}} & 2 & 99.9 & 99.9 & 99.9 & 99.8 & 99.8 & 99.8 & \textbf{0.04} & 0.17 & 0.17 \\
 & 4 & 97.4 & 97.4 & 97.4 & 98.9 & 98.9 & 99.0 & \textbf{0.05} & 0.33 & 0.40 \\
 & 6 & 95.7 & 95.6 & 91.3 & 98.3 & 98.3 & 97.6 & \textbf{0.05} & 0.63 & 0.99 \\
 & 8 & 92.7 & 91.5 & 88.4 & 96.7 & 96.5 & 95.5 & \textbf{0.05} & 1.97 & 5.08 \\
\midrule
\multirow{4}{*}{\shortstack[l]{MNIST\\Set2}} & 2 & 99.9 & 99.9 & 99.9 & 99.8 & 99.9 & 99.9 & \textbf{0.05} & 0.17 & 0.15 \\
 & 4 & 97.4 & 97.3 & 97.3 & 99.0 & 99.1 & 99.0 & \textbf{0.06} & 0.34 & 0.37 \\
 & 6 & 95.0 & 95.0 & 95.0 & 98.1 & 98.2 & 98.1 & \textbf{0.06} & 0.69 & 0.88 \\
 & 8 & 92.7 & 92.7 & 92.7 & 98.0 & 98.0 & 97.9 & \textbf{0.07} & 4.86 & 10.38 \\
\midrule
\multirow{4}{*}{\shortstack[l]{F-MNIST\\Set1}} & 2 & 100.0 & 100.0 & 100.0 & 100.0 & 100.0 & 100.0 & \textbf{0.04} & 0.14 & 0.30 \\
 & 4 & 97.8 & 97.9 & 97.9 & 98.7 & 98.7 & \textbf{98.8} & \textbf{0.05} & 0.35 & 0.59 \\
 & 6 & 91.7 & 91.7 & \textbf{92.0} & 95.5 & 95.3 & 95.7 & \textbf{0.05} & 0.69 & 1.27 \\
 & 8 & 86.0 & 85.6 & 83.9 & 93.2 & 93.2 & 92.8 & \textbf{0.05} & 2.24 & 5.42 \\
\midrule
\multirow{4}{*}{\shortstack[l]{F-MNIST\\Set2}} & 2 & 99.9 & 99.9 & 99.9 & 99.9 & 99.9 & 99.9 & \textbf{0.05} & 0.14 & 0.14 \\
 & 4 & 98.0 & 98.0 & 98.0 & 99.0 & 98.9 & 99.0 & \textbf{0.05} & 0.25 & 0.34 \\
 & 6 & 90.9 & 91.1 & 91.0 & 96.2 & 96.3 & 96.1 & \textbf{0.06} & 0.71 & 0.87 \\
 & 8 & 85.1 & 85.6 & 85.4 & 96.2 & 96.2 & 96.4 & \textbf{0.07} & 4.83 & 10.49 \\
\bottomrule
\end{tabular}}
\label{tab:formula_complexity}
\end{table*}

\subsubsection{Answer to RQ2}
\label{sec:answer_rq2}

\begin{figure}[t]
  \centering
  \includegraphics[width=\textwidth]{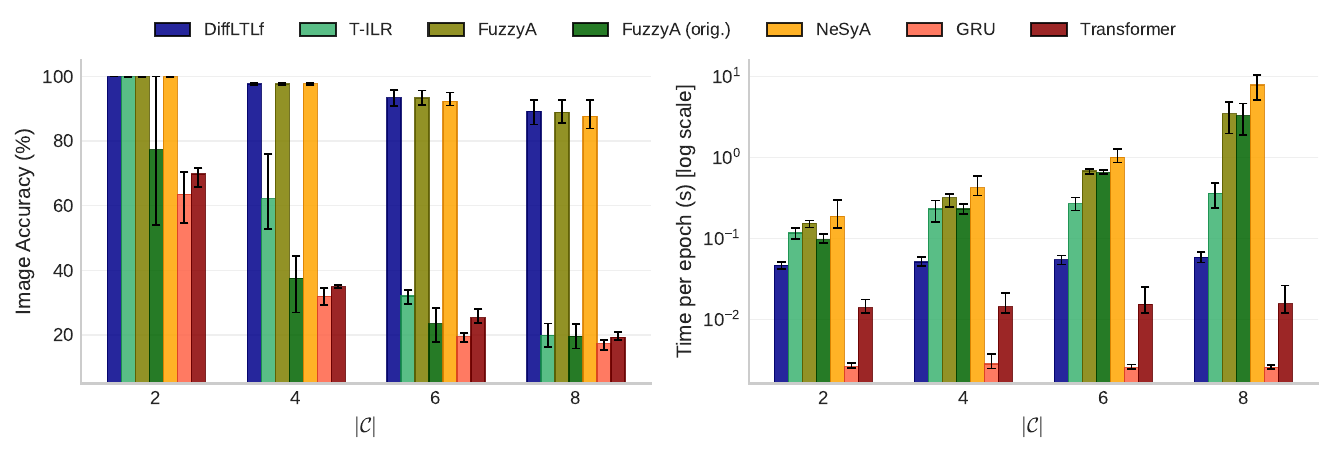}
  \caption{Comparison of symbol grounding accuracy (left) and logic module time per epoch (right) across number of classes $|\mathcal{C}| \in \{2,4,6,8\}$. Values are averaged across all four experimental settings (formula set $\times$ image dataset).}
  
  \label{fig:formula_complexity}
\end{figure}

We answer RQ2 by discussing the two evaluation directions.

% spostata in fondo, alla fine di 7.2.1
% \DiffLTLf attains classification performance on par with, and in the most demanding configurations superior to, both DFA-based competitors, at the symbol grounding as well as at the sequence classification level.
%
% rimossa
% On the most complex specification ($|\mathcal{C}|=8$) with long sequences (Table~\ref{tab:seq_length}, the accuracy of \DiffLTLf varies by less than one percentage point over the entire range of lengths, whereas both competitors degrade steadily.
%
% tolgo, già detto in Varying formula complexity
% The two purely neural baselines never learn a valid symbol grounding in any configuration.

% cancellata (ripete cose che già diciamo sotto)
% On the less demanding configurations, no method is consistently superior.
% Across the formula-complexity setting and the $|\mathcal{C}|=4$ specification of the sequence-length setting, the three neurosymbolic methods are statistically indistinguishable in most configurations. 
%
% spostata in Varying formula complexity
% \NeSyA is the only method that attains a statistically significant advantage over both competitors in a few isolated cases, although with always less than half a percentage point in difference.
%
% spostata in Varying formula complexity
% Finally, the absolute accuracy of all three neurosymbolic methods decreases as $|\mathcal{C}|$ grows, indicating that the increase in specification complexity makes the weakly supervised grounding problem harder for every method considered.

\paragraph{Varying formula complexity} 
Table~\ref{tab:formula_complexity} and Figure~\ref{fig:formula_complexity} report classification performance as the number of classes $|\mathcal{C}|$ increases from $2$ to $8$. 
At $|\mathcal{C}|=2$, all three neurosymbolic methods achieve near-perfect grounding accuracy across all settings, with almost no statistically significant differences. As the symbolic complexity grows, the three methods remain competitive with no single method consistently superior across all configurations.
The only exceptions are two isolated configurations in which \NeSyA attains a statistically significant advantage over both competitors, although with always less than half a percentage point in difference.
% \NeSyA \todo{Ale: questa parte forse è meglio spostarla dopo, visto che non è riferita al caso $|\mathcal{C}|=2$.} is the only method that attains a statistically significant advantage over both competitors in a few isolated cases, although with always less than half a percentage point in difference.
The absolute accuracy of all three neurosymbolic methods decreases as $|\mathcal{C}|$ grows, indicating that the increase in specification complexity makes the weakly supervised grounding problem harder for every method considered.
% The two neural baselines confirm the necessity of explicit temporal logic rule injection, as they prove to be inadequate for this weakly-supervised learning task.
The two purely neural baselines fail to learn a valid symbol grounding, as their grounding accuracy collapses towards random-guess performance as $|\mathcal{C}|$ grows (Figure~\ref{fig:formula_complexity}, left), confirming that explicit temporal logic injection is necessary for this weakly supervised task.

\begin{table*}[ht]
\centering
\caption{Performance comparison varying sequence length. 
Results are averaged over 10 independent runs.
\textbf{Img. Acc}: Symbol Grounding Accuracy; 
\textbf{Seq. Acc}: Sequence Classification Accuracy; 
\textbf{Time}: Logic Module Time per Epoch. 
Results in \textbf{bold} are statistically significantly 
better than all other methods (Wilcoxon, $p<0.05$).}
\resizebox{\textwidth}{!}{
\begin{tabular}{cc ccc ccc ccc}
\toprule
& & \multicolumn{3}{c}{\textbf{Img. Acc.} ($\%$) $\uparrow$} 
& \multicolumn{3}{c}{\textbf{Seq. Acc.} ($\%$) $\uparrow$} 
& \multicolumn{3}{c}{\textbf{Time} (s) $\downarrow$} \\
\cmidrule(lr){3-5} \cmidrule(lr){6-8} \cmidrule(lr){9-11}
$|\mathcal{C}|$ & Len & \DiffLTLf & \FuzzyDFA & \NeSyA 
& \DiffLTLf & \FuzzyDFA & \NeSyA 
& \DiffLTLf & \FuzzyDFA & \NeSyA \\
\midrule
\multirow{10}{*}{4}
 & 20  & 97.9 & 97.9 & \textbf{98.0} & 97.7 & 97.8 & 97.8 & \textbf{0.007} & 0.41 & 0.57 \\
 & 40  & 97.9 & 98.0 & 98.0 & 98.0 & 98.1 & 98.1 & \textbf{0.007} & 0.74 & 1.14 \\
 & 60  & 97.9 & 98.1 & \textbf{98.1} & 97.2 & 97.7 & 97.6 & \textbf{0.007} & 1.09 & 1.74 \\
 & 80  & 98.0 & 97.9 & 97.9 & 98.0 & 98.2 & 98.2 & \textbf{0.007} & 1.46 & 2.33 \\
 & 100 & 98.0 & 98.0 & \textbf{98.1} & 98.1 & 98.0 & 98.0 & \textbf{0.008} & 1.83 & 2.91 \\
 & 120 & 97.9 & 97.9 & 97.9 & 98.0 & 98.2 & 98.0 & \textbf{0.008} & 2.21 & 3.53 \\
 & 140 & 98.1 & 98.1 & 98.1 & 98.9 & 99.0 & 99.0 & \textbf{0.008} & 2.59 & 4.12 \\
 & 160 & 97.7 & 97.7 & \textbf{97.8} & 97.3 & 97.2 & 97.3 & \textbf{0.008} & 2.96 & 4.74 \\
 & 180 & 97.9 & 98.0 & \textbf{98.0} & 98.4 & 98.5 & 98.5 & \textbf{0.008} & 3.31 & 5.34 \\
 & 200 & 97.7 & 97.8 & 97.8 & 97.6 & 97.5 & 97.6 & \textbf{0.008} & 3.69 & 5.89 \\
\midrule
\multirow{10}{*}{8}
 & 20  & 95.8 & 95.3 & 95.8 & 95.3 & 94.0 & 95.5 & \textbf{0.16} & 2.32 & 3.43 \\
 & 40  & 96.0 & 95.6 & 96.0 & 94.5 & 92.8 & 94.7 & \textbf{0.30} & 4.73 & 6.90 \\
 & 60  & \textbf{95.6} & 94.7 & 95.4 & \textbf{94.7} & 92.2 & 94.1 & \textbf{0.44} & 7.12 & 10.45 \\
 & 80  & \textbf{95.7} & 94.6 & 95.4 & 95.0 & 93.5 & 94.2 & \textbf{0.62} & 9.56 & 13.93 \\
 & 100 & 95.6 & 94.2 & 95.6 & 94.5 & 90.1 & 94.2 & \textbf{0.79} & 11.91 & 17.39 \\
 & 120 & \textbf{95.5} & 93.6 & 94.7 & \textbf{93.0} & 90.0 & 91.8 & \textbf{0.95} & 14.35 & 21.00 \\
 & 140 & \textbf{96.1} & 93.7 & 95.5 & \textbf{93.5} & 88.9 & 91.8 & \textbf{1.11} & 16.72 & 24.64 \\
 & 160 & \textbf{95.8} & 94.2 & 95.3 & 94.0 & 90.5 & 92.9 & \textbf{1.25} & 19.09 & 27.87 \\
 & 180 & \textbf{95.8} & 88.3 & 95.3 & \textbf{93.7} & 83.2 & 91.8 & \textbf{1.42} & 21.58 & 29.88 \\
 & 200 & \textbf{95.6} & 92.2 & 95.0 & \textbf{94.7} & 86.3 & 92.9 & \textbf{1.57} & 23.84 & 32.58 \\
\bottomrule
\end{tabular}}
\label{tab:seq_length}
\end{table*}

\begin{figure}[t]
  \centering
  \includegraphics[width=\textwidth]{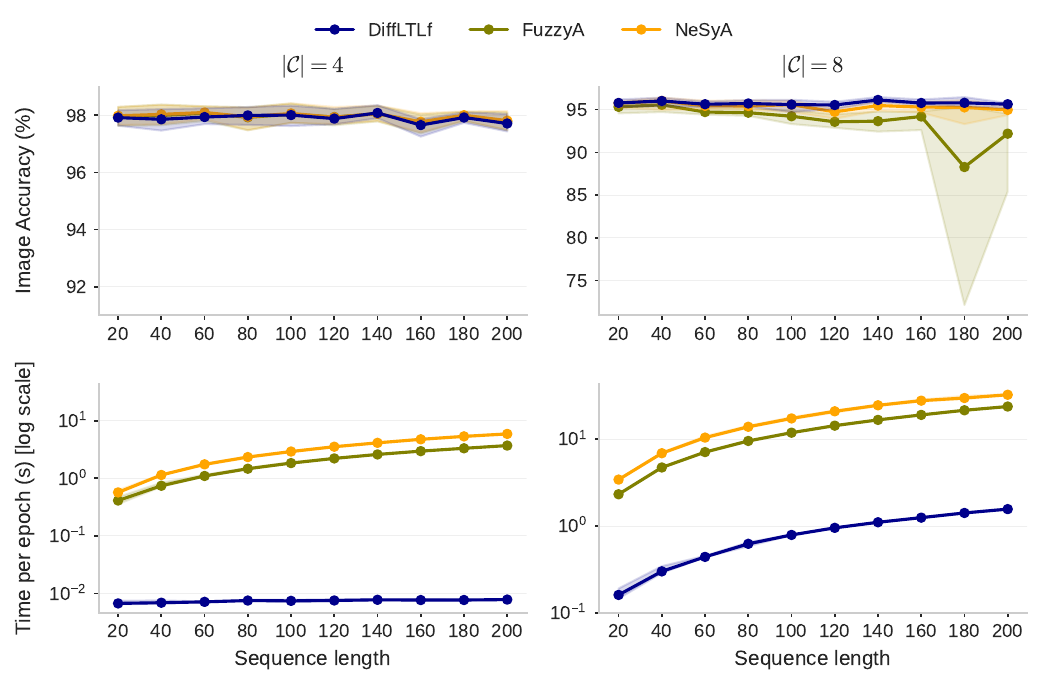}
  \caption{Comparison of symbol grounding accuracy (top row) and logic module time per epoch (bottom row) across sequence lengths varying from 20 to 200.
  Left column: temporal specification with $|\mathcal{C}|=4$ classes. 
  Right column: temporal specification with $|\mathcal{C}|=8$ classes.
  Lines represent the mean over 10 independent runs; shaded bands indicate the range between the minimum and maximum values.
  Time is shown on a logarithmic scale.}
  
  \label{fig:sequence_len}
\end{figure}
\paragraph{Varying sequence length} 
Table~\ref{tab:seq_length} and Figure~\ref{fig:sequence_len} report the performance as the sequence length increases from 20 to 200, following the setup described in Section~\ref{sec:varying_length}.
% Unlike the previous setting, here all sequences within each configuration share the same fixed length, so as to isolate the effect of sequence length on learning performance.
% The temporal specification is ... \todo{qui spendere 2 frasi per spiegare la scelta della formula logica e le motivazioni (distribuzione classi costante al variare di L)} ... involving 8 classes.
With $|\mathcal{C}|=4$, all three methods achieve high and stable grounding accuracy across the entire range of sequence lengths. \NeSyA achieves statistically significant improvements over both competitors at several lengths; however, the absolute differences never exceed
$0.2$ percentage points. In practice, the three methods are on par in terms of accuracy on this specification.
In contrast, with $|\mathcal{C}|=8$, at short sequence lengths \DiffLTLf and \NeSyA achieve comparable grounding accuracy. However, as the sequence length increases, both \FuzzyDFA and \NeSyA show performance degradation and fall behind \DiffLTLf in both image and sequence classification accuracy, with statistically significant differences appearing from length 60 onward. Such degradation is not visible for \DiffLTLf, which maintains stable grounding accuracy across the entire range of sequence lengths. 
% At short sequence lengths, \DiffLTLf and \NeSyA achieve comparable grounding accuracy. As the sequence length increases, both \FuzzyDFA and \NeSyA show performance degradation and fall behind \DiffLTLf in both image and sequence classification accuracy, with statistically significant differences appearing from length 60 onward. Such degradation is not visible in for \DiffLTLf, which maintains stable grounding accuracy across the entire range of sequence lengths.
For clarity, the two purely neural baselines are omitted from
Figure~\ref{fig:sequence_len} and Table~\ref{tab:seq_length}, as they do not reach a valid solution on this task (grounding accuracy close to a random-guess classifier across all sequence lengths).
Their complete numerical results are reported in~\ref{app:full_results}.

% \paragraph{Summary} 
% In the formula-complexity setting, no method consistently outperforms the others: the three neurosymbolic approaches achieve comparable accuracy across all tested configurations. In the sequence-length setting, when the temporal specification is simple ($|\mathcal{C}|=4$), all methods perform comparably. When the specification complexity increases ($|\mathcal{C}|=8$), \DiffLTLf shows an advantage over the DFA-based methods, maintaining stable performances as the problem difficulty increases, while both competitors' performances degrade.
% The following section shows that this accuracy is achieved at a fraction of the computational cost.

\medskip
\noindent
%\DiffLTLf attains classification performance on par with, and in the most demanding configurations superior to, both DFA-based competitors, at the symbol grounding as well as at the sequence classification level.
These results demonstrate that \DiffLTLf attains classification performance on par with both DFA-based competitors, and outperforms them in the most demanding configurations, characterized by high formula complexity and sequence length, at both the symbol grounding and sequence classification levels.

Section~\ref{sec:results_rq3} shows that this level of accuracy is attained at a fraction of the computational cost required by the DFA-based competitors.
Having answered RQ2, the remainder of this subsection discusses these results in greater depth and provides further insights.

\subsubsection{Discussion on RQ2}
\label{sec:discussion_rq2}

% The accuracy differences observed above are setting-dependent. %, and we are not able to attribute them to a single cause.

We have shown in the previous section that the relative performance of the three neurosymbolic methods is not uniform across the experimental configurations.
On most of them, all three methods are statistically indistinguishable; on a few, \NeSyA achieves a small but statistically significant advantage; while on the most demanding configurations, \DiffLTLf consistently achieves higher accuracy.
%Any explanation of these differences must therefore account for the configuration under which the methods are evaluated.
% account not only for why the methods differ, but also for the configuration under which they differ
% Since the compared methods differ along multiple dimensions at once, namely the structure of the differentiable computation (recursion on the syntactic structure of $\Phi$ against recurrence over automaton states) and the semantics used to evaluate it (fuzzy against exact probabilistic), isolating the exact cause of these differences is not straightforward, especially in the many configurations where they are small.
% The three methods
% Understanding which characteristics of the logic module are responsible for these performance differences across experimental configurations is not straightforward. 
Understanding which characteristics of the logic module are responsible for these performance differences, and why such differences appear under certain configurations, remains an open question.
% and the experiments presented in the previous section are not designed to answer this.
%This raises the question of which characteristics of the logic module are responsible for such differences in performance, which the present experiments are not designed to answer.
% The experiments presented in the previous section answer RQ2 but do not suffice to address this analysis, since it requires comparing the methods at the level of the gradient rather than of the final accuracy. 
The end-of-training accuracy metrics, while sufficient to answer RQ2, do not provide the evidence needed to identify the mechanism that produces such differences.
Comparing the methods at the level of the gradient rather than of the final accuracy appears to be a promising starting point.
Feeding the same perception output to each logic module and inspecting the gradient it propagates back (both in terms of magnitude and how it is distributed across the timesteps) for different configurations may %would 
provide more insight about the learning signal each module provides.
% This would extend the gradient-level analysis of Section \ref{sec:discussion_rq1} from the choice of the fuzzy semantics to the choice of the temporal reasoning module, and is out of scope for the current work.
The interest in this direction is motivated by the gradient-level discussion of Section \ref{sec:discussion_rq1}, where the differences between the fuzzy semantics can be attributed to the gradient properties of the underlying operators. However, whether this suffices as an explanation for the choice of the logic module is itself an open question, and we leave its investigation to future work.
% We leave it to future work.

\subsection{RQ3: Computational Scalability}
\label{sec:results_rq3}

We now focus on the computational cost of the compared methods, measured as the time per training epoch spent by the temporal module. 
For the neurosymbolic methods this is the logic module, while for the two purely neural baselines it is the GRU or the Transformer encoder that replaces it.
This metric isolates the overhead introduced by the temporal module, as the perception module is identical across all methods.

\subsubsection{Answer to RQ3}
\label{sec:answer_rq3}
By looking at Table~\ref{tab:formula_complexity} and Table~\ref{tab:seq_length}, we observe that \DiffLTLf is the fastest of the three neurosymbolic methods in every configuration tested and the margin increases with both dimensions of the evaluation (i.e., formula complexity and sequence length).
Along the formula-complexity direction~(Figure~\ref{fig:formula_complexity}), the time required by \DiffLTLf is essentially invariant with respect to the complexity of the specification: going from $|\mathcal{C}|=2$ to $|\mathcal{C}|=8$, i.e.\ from one to four conjoined sub-formulas, produces negligible overhead (from $0.04$\,s to $0.07$\,s per epoch).
Both DFA-based competitors, in contrast, grow super-linearly with $|\mathcal{C}|$. Their cost increases by roughly an order of magnitude in a single step from $|\mathcal{C}|=6$ to $|\mathcal{C}|=8$.
Along the sequence-length direction~(Figure~\ref{fig:sequence_len}), all three neurosymbolic methods scale linearly with the sequence length. However, at the longest sequences tested, \DiffLTLf is faster than both competitors by between one and nearly three orders of magnitude, with \NeSyA consistently the slowest of the two.
Moreover, the reported cost accounts only for training. The cost of constructing the automaton, which \DiffLTLf does not incur at all, is excluded, so the reported gap is a lower bound.
The two purely neural baselines attain lower per-epoch times, as they avoid any symbolic computation, but their speed is not meaningful here since they do not learn the task.

\subsubsection{Discussion on RQ3}
\label{sec:discussion_rq3}
Both \FuzzyDFA and \NeSyA rely on the explicit construction of a finite-state automaton from the \LTLF formula, which carries two distinct limitations.
First, the automaton must be built before training, and its size is, in the worst case, double-exponential in the size of the \LTLF formula~\cite{DBLP:conf/ijcai/GiacomoV13}.
% This is what is measured by the formula-complexity direction of the evaluation, and it accounts for the rapid growth of the cost of both methods with $|\mathcal{C}|$.
This is what the formula-complexity direction measures. Since $\Phi$ conjoins $K=|\mathcal{C}|/2$ sub-formulas (see Section \ref{sec:formula_complexity}), a larger $|\mathcal{C}|$ means a larger formula, hence a larger
automaton, and a fast growth of the computational cost of both methods.
Second, the automaton's states valuation is computed recurrently over the sequence, each timestep depending on the previous one, so the evaluation is inherently sequential, preventing parallelisation over time. 

On top of this common cost, the two methods differ in how each transition is evaluated.
\FuzzyDFA evaluates each transition guard according to fuzzy semantics, and updates the vector of fuzzy truth values over the automaton states.
\NeSyA instead computes the probability that each transition guard is satisfied, and these probabilities determine the transition matrix by which the state distribution is updated.
The probabilistic semantics is exact and avoids the approximation introduced by fuzzy evaluation.
It comes, however, at a computational cost, as weighted model counting must be carried out for every transition at every timestep. This is what makes \NeSyA the most computationally demanding of the compared methods. 
This is reflected in the observation that the gap between \FuzzyDFA and \NeSyA remains stable as the sequence length grows, and widens as the temporal specification becomes more complex (for which larger automata are built).

\DiffLTLf incurs neither cost, as it avoids the automaton altogether and evaluates the temporal specification directly on the perception output.
The cost of the evaluation is governed by the size of the temporal specification $\Phi$, and not by that of an automaton which may be double-exponentially larger.

Combined with the accuracy result of Section \ref{sec:results_rq2}, these results show that direct fuzzy evaluation of \LTLF specifications achieves comparable or higher classification accuracy than DFA-based methods while requiring a fraction of the computational cost. 
%\todo{dire che non consideriamo il costo di costruzione del DFA}

\section{Threats to Validity}
\label{sec:threat_validity}
In this section, we highlight some limitations of the current study. 

\paragraph{Synthetic dataset}
% \paragraph{Synthetic Dataset}
The input sequences used in the evaluation are synthetically built. This choice is deliberate as synthetic generation provides full control over the independent variables of the evaluation, namely formula complexity and sequence length.
Control over such variables is necessary to fully assess the performance and scalability of the temporal reasoning component, and to stress-test it under increasingly demanding conditions.
% che potrebbero anche essere estese oltre a quando facciamo qui (più classi, seq più lunghe)
To the best of our knowledge, no real-world temporal datasets annotated with ground-truth \LTLF satisfaction labels exist at the scale required by our experimental evaluation.
%protocol
The image datasets used for the perceptual component, MNIST and Fashion-MNIST, are standard benchmarks adopted across temporal neurosymbolic works~\cite{DBLP:conf/kr/UmiliCG23, ltlzincIJCAI}.

\paragraph{Selection of temporal specifications}
The temporal specifications used in the evaluation are built as conjunctions or simple combinations of \declare patterns, rather than \LTLF formulas with deeply nested temporal operators. This was done, in line with previous work~\cite{DBLP:conf/kr/UmiliCG23}, due to the fact that
\declare patterns are a set of well-established temporal constraints widely adopted in the process mining literature~\cite{declare_handbook}.
% They cover a representative range of temporal dependencies, including ...
The conjunctive structure of the temporal formula allows control of its complexity through the number of sub-formulas $K$, which serves as the independent variable in the formula-complexity experiments (Section~\ref{sec:formula_complexity}).
% si può mettere: usare operatori molto innestati renderebbe estremamente complesso risalire agli effetti di ciascun operatore (e.g., formula 2 esperimenti synt driving Manginas)
% non è espresso bene

\paragraph{Design choices}
No exhaustive hyperparameter search was performed.
The training configuration, including optimizer (Adam), learning rate ($10^{-4}$), and number of training epochs, directly follows configurations of a related temporal neurosymbolic benchmark~\cite{ltlzincIJCAI}.
All the compared methods share the same training configuration to ensure fair comparison.
The train/test partition follows a standard 80/20 split.
\section{Conclusions}
\label{sec:conclusions}
%% Labels are used to cross-reference an item using \ref command.

This paper introduced \DiffLTLf, a differentiable framework for injecting \LTLF specifications into neural sequence classifiers while avoiding intermediate automata-based representations. The method evaluates temporal formulas directly on finite traces through various fuzzy semantics, allowing the symbolic specification to contribute to the training objective while keeping the perception module's architecture unchanged. Alongside the framework, we provided a systematic study of fuzzy temporal semantics, highlighting how the choice of semantics affects both the interpretation of temporal operators and the resulting learning behaviour.

Moreover, \DiffLTLf was evaluated on classification tasks under different temporal specifications, formula complexities, numbers of classes, and sequence lengths, and compared against DFA-based neurosymbolic baselines.  The results show that different fuzzy semantics lead to different predictive behaviour, confirming that the choice of semantics is an important design decision in temporal NeSy frameworks. At the same time, \DiffLTLf achieves comparable or better classification performance while being substantially more efficient than existing state-of-the-art techniques. This advantage is especially clear as the specifications become more complex, since \DiffLTLf avoids the automaton-size blow-up that can affect DFA-based methods.

Overall, our findings demonstrate that direct fuzzy evaluation of \LTLF formulas provides a practical and scalable alternative to automata-based integration of temporal knowledge. In addition, the theoretical analysis of fuzzy temporal semantics and the evaluation protocol introduced in this work provide a foundation for the development and systematic assessment of future temporal neurosymbolic systems. Future work will investigate richer temporal logics, larger and more realistic benchmarks, and methods for jointly learning or refining temporal specifications from data.

\section*{Acknowledgements}
This work is funded by the project TRUStworthy predictive models for Temporal Event Data (TRUSTED), start-up fund, Free University of Bozen-Bolzano.

\bibliographystyle{elsarticle-harv} 
\bibliography{biblio}

\appendix
\section{Implementation Details of Baselines}
\label{app1}
This appendix provides details regarding the implementation of the baseline methods, \FuzzyDFA \cite{DBLP:conf/kr/UmiliCG23} and \NeSyA \cite{DBLP:conf/ijcai/ManginasPR25}, considered in our comparative analysis. 
% While our experimental framework is built on the \LTLZinc benchmark \cite{ltlzincIJCAI}, 
We adopted specific baseline implementations to ensure a fair comparison in terms of computational efficiency and adherence to the theoretical foundations of each method.

\subsection{\FuzzyDFA}
For the \FuzzyDFA baseline, we adopted the implementation provided in the \LTLZinc codebase \cite{ltlzincIJCAI}, which, 
% reproduces the original method while adding batch sequence evaluation to reduce training times.
in addition to batched sequence evaluation, differs from the original implementation \cite{DBLP:conf/kr/UmiliCG23} in two aspects.
First, at each timestep the automaton's state distribution is renormalized through a softmax before the next transition step.
Second, the training objective is a binary cross-entropy loss between the sum of the fuzzy values assigned to the accepting states and the sequence label, rather than the loss function used in \cite{DBLP:conf/kr/UmiliCG23}. The original loss function combines the fuzzy values of the accepting states through a fuzzy exclusive-or (for positive examples) and through a conjunction of negations (for negative examples).
We verified empirically that, under identical perception module and training parameters, the \LTLZinc implementation variant yields substantially higher symbol-grounding accuracy and sequence classification accuracy than the original implementation variant on our benchmark. 
Moreover, the change of loss function accounts for most of the gap (\ref{app:full_results}). %\todo{puntare a sezione risultati nell'appendice se la mettiamo e se mettiamo un'ablation sull'implementazione di \FuzzyDFA}
The core mechanism of this approach remains unchanged and involves building a propositional formula for each potential \textit{next state} of the automaton. The formula is built as the disjunction of all valid incoming transitions, where a transition is defined as the conjunction of the \textit{previous state} and the corresponding \textit{transition guard}. Each \textit{previous state} in the automaton is treated as a logical variable.
These formulas are evaluated by mapping conjunctions to multiplications and disjunctions to algebraic summations.
\subsection{\NeSyA}
For the \NeSyA baseline, we did not utilize the implementation found in the \LTLZinc codebase, as it compiles state variables and transition guards together into a single logical circuit and can be computationally inefficient for large automata.
Instead, we implemented the architecture described in the original \NeSyA work \cite{DBLP:conf/ijcai/ManginasPR25}. In this setup, only the transition guards are compiled into d-DNNF circuits to enable exact probabilistic reasoning over the input symbols.
At each timestep, these compiled circuits are evaluated to populate a transition matrix. The probability over the states is updated via matrix multiplication between the previous state distribution and the transition matrix.
%%       \cite{<label>} ==> [1]

%%
% \input{Sections/appendixB}
\section{Experimental Setup Details}
\label{app:exp_setup}
% \todo{max: questa app non è mai citata nel testo}

This appendix provides the training configuration, model architecture, and dataset construction details used across all experiments.

\paragraph{Perception module}
All methods share the same convolutional neural network as the perception module, ensuring that differences in performance are to be attributed exclusively to the temporal reasoning component.
The network consists of two convolutional layers with 32 and 64 filters respectively (kernel size $5 \times 5$), each followed by a ReLU activation and $2 \times 2$
max-pooling. The convolutional layers are followed by two fully connected layers of size 1024 and $|C|$, with ReLU activation and dropout ($p = 0.5$) applied after the first fully connected layer. 
The input images are resized to $32 \times 32$ pixels. Since each image represents a single class, mutual exclusivity of the class atoms holds. The output logits are passed through a softmax layer, so that the per-class truth values of each observation are normalized to sum to one.

\paragraph{Training configuration}
All models are trained with the Adam optimizer using a fixed learning rate of $10^{-4}$ and a batch size of 64.
The loss function is a binary cross-entropy between the predicted satisfaction value and the ground-truth sequence label $y \in \{0, 1\}$.
In the formula-complexity experiments (Section~\ref{sec:formula_complexity}), the training runs for 50 epochs.
In the sequence-length experiments (Section~\ref{sec:varying_length}), training is limited to 30 epochs.
The reduced number of training epochs reflects the substantially higher per-epoch training time of DFA-based methods on long sequences. Since a fair comparison requires all methods to train for the same number of epochs, we selected a number of epochs that remains feasible for the slowest method under evaluation.
Each configuration is repeated over 10 independent runs with distinct random seeds.

\paragraph{Dataset construction}
For each experimental configuration, we generate a dataset of 1000 sequences. Each sequence is constructed by independently sampling, at every timestep and for each of the $S=3$ streams, a class label uniformly at random from $\{0, \ldots, |C|-1\}$ and retrieving a corresponding image from the MNIST or Fashion-MNIST training set.
The ground-truth sequence label is computed by evaluating the \LTLF specification $\Phi$ on the crisp symbolic trace obtained by the sampled class label. 
The resulting dataset is partitioned into 80\% training and 20\% test sets; the test set is held out and never observed during training.

\paragraph{Hardware and software}
All experiments were performed on a single NVIDIA RTX A6000 GPU (48 GB), using PyTorch 2.3.1 with CUDA 12.1.
\section{Complete Experimental Results}
\label{app:full_results}

Results in Section \ref{sec:results} report only the representative temporal neurosymbolic methods (\DiffLTLf with \FLTLF[P] semantics, \FuzzyDFA, and \NeSyA).
Here we report the complete results for all evaluated methods: the three \DiffLTLf semantics
(Product, Gödel, Łukasiewicz), \TILR, \FuzzyDFA (both the \LTLZinc variant used
throughout the paper and the original implementation, see~\ref{app1}), \NeSyA,
and the two purely neural baselines (GRU and Transformer).
Tables~\ref{tab:full_complexity_mnist} and ~\ref{tab:full_complexity_fmnist}
report the formula-complexity experiment (Section~\ref{sec:formula_complexity})
and Tables~\ref{tab:full_length_c4} and ~\ref{tab:full_length_c8} the sequence-length experiment
(Section~\ref{sec:varying_length}). All values are averaged over 10 runs and use
the same protocol as the main text; bold marks statistical significance over all
other methods (Wilcoxon signed-rank, $p<0.05$).

\begin{table*}[ht]
\centering
\caption{Full results varying formula complexity on MNIST, averaged over 10 runs. \textbf{Img. Acc.}: Symbol Grounding Accuracy; \textbf{Seq. Acc.}: Sequence Classification Accuracy; \textbf{Time}: Logic Module Time per Epoch. Results in bold are
statistically significantly better than all other methods (Wilcoxon, $p<0.05$).}
\label{tab:full_complexity_mnist}
\resizebox{\textwidth}{!}{%
\begin{tabular}{l c c c c c c c c}
\toprule
Method & \multicolumn{4}{c}{Set 1} & \multicolumn{4}{c}{Set 2} \\
\cmidrule(lr){2-5} \cmidrule(lr){6-9}
$|\mathcal{C}|$: & 2 & 4 & 6 & 8 & 2 & 4 & 6 & 8 \\
\midrule
\multicolumn{9}{l}{\textit{Img. Acc. (\%) $\uparrow$}} \\
\DiffLTLf (Product) & 99.9 & 97.4 & 95.7 & 92.7 & 99.9 & 97.4 & 95.0 & 92.7 \\
\DiffLTLf (G\"odel) & 99.9 & 60.8 & 32.7 & 22.8 & 99.9 & 59.2 & 29.6 & 16.3 \\
\DiffLTLf (\L{}ukasiewicz) & 90.1 & 58.4 & 46.2 & 35.4 & 99.9 & 62.6 & 57.1 & 42.0 \\
\TILR & 99.9 & 60.8 & 32.7 & 22.8 & 99.9 & 59.2 & 29.6 & 16.3 \\
\FuzzyDFA & 99.9 & 97.4 & 95.6 & 91.5 & 99.9 & 97.3 & 95.0 & 92.7 \\
\FuzzyDFA (orig.) & 70.2 & 44.5 & 22.8 & 22.0 & 99.9 & 37.2 & 17.9 & 17.2 \\
\NeSyA & 99.9 & 97.4 & 91.3 & 88.4 & 99.9 & 97.3 & 95.0 & 92.7 \\
GRU & 65.8 & 34.5 & 20.5 & 15.5 & 63.1 & 30.8 & 19.0 & 18.4 \\
Transformer & 70.5 & 34.9 & 23.8 & 18.4 & 65.8 & 34.7 & 24.6 & 21.0 \\
\midrule
\multicolumn{9}{l}{\textit{Seq. Acc. (\%) $\uparrow$}} \\
\DiffLTLf (Product) & 99.8 & 98.9 & 98.3 & 96.7 & 99.8 & 99.0 & 98.1 & 98.0 \\
\DiffLTLf (G\"odel) & 99.8 & 87.2 & 90.2 & 82.8 & 99.8 & 93.0 & 86.8 & 87.3 \\
\DiffLTLf (\L{}ukasiewicz) & 93.3 & 88.5 & 90.8 & 87.3 & 99.8 & 93.5 & 87.9 & 88.5 \\
\TILR & 99.8 & 87.2 & 90.2 & 82.8 & 99.8 & 93.0 & 86.8 & 87.3 \\
\FuzzyDFA & 99.8 & 98.9 & 98.3 & 96.5 & 99.9 & 99.1 & 98.2 & 98.0 \\
\FuzzyDFA (orig.) & 79.5 & 86.5 & 87.0 & 86.7 & 99.8 & 88.2 & 85.5 & 86.2 \\
\NeSyA & 99.8 & 99.0 & 97.6 & 95.5 & 99.9 & 99.0 & 98.1 & 97.9 \\
GRU & 99.5 & 87.0 & 87.3 & 82.5 & 79.4 & 88.1 & 85.5 & 86.8 \\
Transformer & 99.3 & 89.0 & 90.7 & 87.5 & 99.5 & 91.5 & 88.3 & 88.8 \\
\midrule
\multicolumn{9}{l}{\textit{Time (s) $\downarrow$}} \\
\DiffLTLf (Product) & 0.042 & 0.045 & 0.048 & 0.050 & 0.050 & 0.059 & 0.061 & 0.065 \\
\DiffLTLf (G\"odel) & 0.003 & 0.006 & 0.008 & 0.011 & 0.050 & 0.052 & 0.054 & 0.061 \\
\DiffLTLf (\L{}ukasiewicz) & 0.004 & 0.007 & 0.010 & 0.013 & 0.050 & 0.055 & 0.055 & 0.065 \\
\TILR & 0.098 & 0.295 & 0.279 & 0.479 & 0.099 & 0.158 & 0.221 & 0.297 \\
\FuzzyDFA & 0.166 & 0.327 & 0.628 & 1.97 & 0.167 & 0.345 & 0.694 & 4.86 \\
\FuzzyDFA (orig.) & 0.089 & 0.247 & 0.697 & 1.91 & 0.087 & 0.209 & 0.625 & 4.64 \\
\NeSyA & 0.165 & 0.397 & 0.990 & 5.08 & 0.154 & 0.367 & 0.878 & 10.38 \\
GRU & \textbf{0.002} & \textbf{0.002} & \textbf{0.002} & \textbf{0.002} & \textbf{0.003} & \textbf{0.003} & \textbf{0.002} & \textbf{0.002} \\
Transformer & 0.014 & 0.012 & 0.012 & 0.013 & 0.012 & 0.012 & 0.012 & 0.012 \\
\bottomrule
\end{tabular}}
\end{table*}

\begin{table*}[ht]
\centering
\caption{Full results varying formula complexity on Fashion-MNIST, averaged over 10 runs. \textbf{Img. Acc.}: symbol grounding accuracy; \textbf{Seq. Acc.}: sequence classification accuracy; \textbf{Time}: Logic Module Time per Epoch. Results in bold are
statistically significantly better than all other methods (Wilcoxon, $p<0.05$).}
\label{tab:full_complexity_fmnist}
\resizebox{\textwidth}{!}{%
\begin{tabular}{l c c c c c c c c}
\toprule
Method & \multicolumn{4}{c}{Set 1} & \multicolumn{4}{c}{Set 2} \\
\cmidrule(lr){2-5} \cmidrule(lr){6-9}
$|\mathcal{C}|$: & 2 & 4 & 6 & 8 & 2 & 4 & 6 & 8 \\
\midrule
\multicolumn{9}{l}{\textit{Img. Acc. (\%) $\uparrow$}} \\
\DiffLTLf (Product) & 100.0 & 97.8 & 91.7 & 86.0 & 99.9 & 98.0 & 90.9 & 85.1 \\
\DiffLTLf (G\"odel) & 100.0 & 76.0 & 34.0 & 23.6 & 99.9 & 52.7 & 32.0 & 16.7 \\
\DiffLTLf (\L{}ukasiewicz) & 89.9 & 55.9 & 47.6 & 40.0 & 99.9 & 79.2 & 45.8 & 33.0 \\
\TILR & 100.0 & 76.0 & 34.0 & 23.6 & 99.9 & 52.7 & 32.0 & 16.7 \\
\FuzzyDFA & 100.0 & 97.9 & 91.7 & 85.6 & 99.9 & 98.0 & 91.1 & 85.6 \\
\FuzzyDFA (orig.) & 54.0 & 41.6 & 28.3 & 23.5 & 84.8 & 27.0 & 25.2 & 15.9 \\
\NeSyA & 100.0 & 97.9 & \textbf{92.0} & 83.9 & 99.9 & 98.0 & 91.0 & 85.4 \\
GRU & 70.4 & 33.5 & 20.6 & 18.2 & 54.5 & 29.3 & 17.8 & 17.3 \\
Transformer & 71.6 & 35.5 & 28.0 & 19.4 & 71.2 & 34.6 & 24.8 & 18.4 \\
\midrule
\multicolumn{9}{l}{\textit{Seq. Acc. (\%) $\uparrow$}} \\
\DiffLTLf (Product) & 100.0 & 98.7 & 95.5 & 93.2 & 99.9 & 99.0 & 96.2 & 96.2 \\
\DiffLTLf (G\"odel) & 100.0 & 89.0 & 89.9 & 85.5 & 99.9 & 91.0 & 87.3 & 89.2 \\
\DiffLTLf (\L{}ukasiewicz) & 93.3 & 87.1 & 89.5 & 85.8 & 99.9 & 93.1 & 88.9 & 89.3 \\
\TILR & 100.0 & 89.0 & 89.9 & 85.5 & 99.9 & 91.0 & 87.3 & 89.2 \\
\FuzzyDFA & 100.0 & 98.7 & 95.3 & 93.2 & 99.9 & 98.9 & 96.3 & 96.2 \\
\FuzzyDFA (orig.) & 67.0 & 82.7 & 87.0 & 86.3 & 90.0 & 88.5 & 87.2 & 87.6 \\
\NeSyA & 100.0 & \textbf{98.8} & 95.7 & 92.8 & 99.9 & 99.0 & 96.1 & 96.4 \\
GRU & 99.8 & 86.0 & 87.8 & 83.3 & 73.0 & 88.5 & 88.7 & 88.2 \\
Transformer & 99.7 & 89.0 & 89.2 & 86.5 & 99.5 & 91.8 & 90.0 & 90.6 \\
\midrule
\multicolumn{9}{l}{\textit{Time (s) $\downarrow$}} \\
\DiffLTLf (Product) & 0.042 & 0.046 & 0.048 & 0.050 & 0.051 & 0.053 & 0.062 & 0.068 \\
\DiffLTLf (G\"odel) & 0.003 & 0.006 & 0.008 & 0.010 & 0.050 & 0.053 & 0.058 & 0.062 \\
\DiffLTLf (\L{}ukasiewicz) & 0.004 & 0.007 & 0.009 & 0.012 & 0.048 & 0.053 & 0.053 & 0.062 \\
\TILR & 0.135 & 0.222 & 0.321 & 0.427 & 0.134 & 0.246 & 0.256 & 0.239 \\
\FuzzyDFA & 0.144 & 0.353 & 0.690 & 2.24 & 0.137 & 0.245 & 0.714 & 4.83 \\
\FuzzyDFA (orig.) & 0.095 & 0.201 & 0.637 & 1.89 & 0.113 & 0.266 & 0.624 & 4.66 \\
\NeSyA & 0.299 & 0.595 & 1.27 & 5.42 & 0.135 & 0.339 & 0.867 & 10.49 \\
GRU & \textbf{0.003} & \textbf{0.004} & \textbf{0.003} & \textbf{0.003} & \textbf{0.003} & \textbf{0.002} & \textbf{0.002} & \textbf{0.002} \\
Transformer & 0.018 & 0.021 & 0.025 & 0.026 & 0.012 & 0.012 & 0.012 & 0.012 \\
\bottomrule
\end{tabular}}
\end{table*}

\begin{table*}[ht]
\centering
\caption{Full results varying the sequence length, $|\mathcal{C}|=4$, averaged over 10 runs. \textbf{Img. Acc.}: Symbol Grounding Accuracy; \textbf{Seq. Acc.}: Sequence Classification Accuracy; \textbf{Time}: Logic
Module Time per Epoch. Bold marks a value statistically significantly better than all other methods (Wilcoxon, $p<0.05$).}
\label{tab:full_length_c4}
\resizebox{\textwidth}{!}{%
\begin{tabular}{l c c c c c c c c c c}
\toprule
& \multicolumn{10}{c}{Sequence length} \\
\cmidrule(lr){2-11}
Method & 20 & 40 & 60 & 80 & 100 & 120 & 140 & 160 & 180 & 200 \\
\midrule
\multicolumn{11}{l}{\textit{Img. Acc. (\%) $\uparrow$}} \\
\DiffLTLf (Product) & 97.9 & 97.9 & 97.9 & 98.0 & 98.0 & 97.9 & 98.1 & 97.7 & 97.9 & 97.7 \\
\DiffLTLf (G\"odel) & 43.5 & 35.0 & 47.2 & 35.9 & 49.9 & 48.5 & 41.1 & 55.4 & 55.0 & 54.8 \\
\DiffLTLf (\L{}ukasiewicz) & 72.9 & 92.7 & 87.8 & 92.1 & 85.5 & 78.6 & 82.6 & 86.6 & 91.1 & 81.2 \\
\TILR & 43.5 & 35.0 & 47.2 & 35.9 & 49.9 & 48.5 & 41.1 & 55.4 & 55.0 & 54.8 \\
\FuzzyDFA & 97.9 & 98.0 & 98.1 & 97.9 & 98.0 & 97.9 & 98.1 & 97.7 & 98.0 & 97.8 \\
\FuzzyDFA (orig.) & 22.7 & 22.7 & 22.8 & 22.8 & 22.8 & 22.7 & 22.7 & 22.8 & 22.7 & 22.8 \\
\NeSyA & \textbf{98.0} & 98.0 & \textbf{98.1} & 97.9 & \textbf{98.1} & 97.9 & 98.1 & \textbf{97.8} & \textbf{98.0} & 97.8 \\
GRU & 44.9 & 41.1 & 36.9 & 43.0 & 39.4 & 41.4 & 39.5 & 37.8 & 36.3 & 38.0 \\
Transformer & 46.6 & 36.3 & 42.3 & 39.0 & 36.0 & 38.4 & 43.2 & 32.8 & 34.3 & 44.0 \\
\midrule
\multicolumn{11}{l}{\textit{Seq. Acc. (\%) $\uparrow$}} \\
\DiffLTLf (Product) & 97.7 & 98.0 & 97.2 & 98.0 & 98.1 & 98.0 & 98.9 & 97.3 & 98.4 & 97.6 \\
\DiffLTLf (G\"odel) & 70.6 & 64.7 & 72.0 & 68.5 & 74.2 & 75.0 & 70.1 & 75.5 & 71.9 & 75.0 \\
\DiffLTLf (\L{}ukasiewicz) & 84.4 & 90.0 & 89.5 & 91.2 & 87.8 & 85.7 & 86.8 & 87.3 & 90.2 & 84.8 \\
\TILR & 70.6 & 64.7 & 72.0 & 68.5 & 74.2 & 75.0 & 70.1 & 75.5 & 71.9 & 75.0 \\
\FuzzyDFA & 97.8 & 98.1 & 97.7 & 98.2 & 98.0 & 98.2 & 99.0 & 97.2 & 98.5 & 97.5 \\
\FuzzyDFA (orig.) & 32.7 & 36.2 & 35.5 & 35.1 & 33.5 & 34.9 & 35.7 & 35.2 & 37.5 & 37.5 \\
\NeSyA & 97.8 & 98.1 & 97.6 & 98.2 & 98.0 & 98.0 & 99.0 & 97.3 & 98.5 & 97.6 \\
GRU & 78.4 & 79.1 & 78.3 & 81.1 & 79.5 & 79.2 & 79.2 & 77.5 & 76.5 & 78.2 \\
Transformer & 80.4 & 78.3 & 76.9 & 74.1 & 73.5 & 71.0 & 67.2 & 69.5 & 62.8 & 67.2 \\
\midrule
\multicolumn{11}{l}{\textit{Time (s) $\downarrow$}} \\
\DiffLTLf (Product) & 0.007 & 0.007 & 0.007 & 0.008 & 0.007 & 0.008 & 0.008 & 0.008 & 0.008 & 0.008 \\
\DiffLTLf (G\"odel) & 0.025 & 0.026 & 0.023 & 0.026 & 0.023 & 0.023 & 0.022 & 0.025 & 0.024 & 0.027 \\
\DiffLTLf (\L{}ukasiewicz) & 0.026 & 0.028 & 0.025 & 0.028 & 0.024 & 0.025 & 0.025 & 0.027 & 0.025 & 0.027 \\
\TILR & 0.340 & 0.760 & 1.32 & 2.02 & 2.86 & 3.84 & 4.93 & 6.38 & 7.63 & 9.25 \\
\FuzzyDFA & 0.410 & 0.740 & 1.09 & 1.46 & 1.83 & 2.21 & 2.59 & 2.96 & 3.31 & 3.69 \\
\FuzzyDFA (orig.) & 0.351 & 0.686 & 1.03 & 1.39 & 1.75 & 2.11 & 2.44 & 2.80 & 3.20 & 3.50 \\
\NeSyA & 0.568 & 1.14 & 1.74 & 2.33 & 2.91 & 3.53 & 4.12 & 4.74 & 5.34 & 5.89 \\
GRU & \textbf{0.004} & \textbf{0.004} & \textbf{0.004} & \textbf{0.004} & \textbf{0.005} & \textbf{0.005} & \textbf{0.005} & \textbf{0.005} & \textbf{0.006} & \textbf{0.006} \\
Transformer & 0.016 & 0.018 & 0.015 & 0.017 & 0.020 & 0.023 & 0.028 & 0.031 & 0.033 & 0.039 \\
\bottomrule
\end{tabular}}
\end{table*}

\begin{table*}[ht]
\centering
\caption{Full results varying the sequence length, $|\mathcal{C}|=8$, averaged over 10 runs. \textbf{Img. Acc.}: Symbol Grounding Accuracy; \textbf{Seq. Acc.}: Sequence Classification Accuracy; \textbf{Time}: Logic
Module Time per Epoch. 
Results in bold are statistically significantly better than all
other methods (Wilcoxon, $p<0.05$).}
\label{tab:full_length_c8}
\resizebox{\textwidth}{!}{%
\begin{tabular}{l c c c c c c c c c c}
\toprule
& \multicolumn{10}{c}{Sequence length} \\
\cmidrule(lr){2-11}
Method & 20 & 40 & 60 & 80 & 100 & 120 & 140 & 160 & 180 & 200 \\
\midrule
\multicolumn{11}{l}{\textit{Img. Acc. (\%) $\uparrow$}} \\
\DiffLTLf (Product) & 95.8 & 96.0 & \textbf{95.6} & \textbf{95.7} & 95.6 & \textbf{95.5} & \textbf{96.1} & \textbf{95.8} & \textbf{95.8} & \textbf{95.6} \\
\DiffLTLf (G\"odel) & 17.5 & 17.1 & 17.7 & 16.7 & 17.3 & 15.1 & 14.8 & 16.4 & 15.2 & 14.9 \\
\DiffLTLf (\L{}ukasiewicz) & 14.6 & 14.6 & 14.6 & 14.6 & 14.5 & 14.5 & 14.5 & 14.5 & 14.6 & 14.6 \\
\TILR & 17.5 & 17.1 & 17.7 & 16.7 & 17.3 & 15.1 & 14.8 & 16.4 & 15.2 & 14.9 \\
\FuzzyDFA & 95.3 & 95.6 & 94.7 & 94.6 & 94.2 & 93.6 & 93.7 & 94.2 & 88.3 & 92.2 \\
\FuzzyDFA (orig.) & 15.5 & 24.8 & 39.1 & 18.9 & 42.6 & 38.5 & 29.0 & 45.7 & 45.1 & 51.2 \\
\NeSyA & 95.8 & 96.0 & 95.4 & 95.4 & 95.6 & 94.7 & 95.5 & 95.3 & 95.3 & 95.0 \\
GRU & 16.1 & 17.1 & 16.7 & 18.2 & 16.2 & 16.1 & 17.7 & 17.0 & 16.1 & 17.2 \\
Transformer & 21.4 & 22.7 & 21.2 & 19.6 & 21.0 & 21.3 & 17.0 & 17.9 & 21.4 & 19.2 \\
\midrule
\multicolumn{11}{l}{\textit{Seq. Acc. (\%) $\uparrow$}} \\
\DiffLTLf (Product) & 95.3 & 94.5 & \textbf{94.7} & 95.0 & 94.5 & \textbf{93.0} & \textbf{93.5} & 94.0 & \textbf{93.7} & \textbf{94.7} \\
\DiffLTLf (G\"odel) & 57.0 & 55.1 & 53.4 & 54.1 & 54.1 & 53.3 & 50.3 & 53.8 & 52.2 & 52.6 \\
\DiffLTLf (\L{}ukasiewicz) & 50.8 & 52.9 & 51.0 & 47.9 & 51.5 & 51.5 & 47.4 & 50.8 & 50.9 & 51.6 \\
\TILR & 57.0 & 55.1 & 53.4 & 54.1 & 54.1 & 53.3 & 50.3 & 53.8 & 52.2 & 52.6 \\
\FuzzyDFA & 94.0 & 92.8 & 92.2 & 93.5 & 90.1 & 90.0 & 88.9 & 90.5 & 83.2 & 86.3 \\
\FuzzyDFA (orig.) & 50.8 & 58.8 & 71.2 & 59.7 & 71.0 & 68.3 & 65.8 & 70.0 & 71.4 & 71.7 \\
\NeSyA & 95.5 & 94.7 & 94.1 & 94.2 & 94.2 & 91.8 & 91.8 & 92.9 & 91.8 & 92.9 \\
GRU & 54.0 & 54.5 & 53.7 & 55.4 & 54.5 & 51.6 & 54.5 & 55.1 & 52.8 & 56.4 \\
Transformer & 63.2 & 63.8 & 61.7 & 57.1 & 54.5 & 58.1 & 53.4 & 54.2 & 55.2 & 53.0 \\
\midrule
\multicolumn{11}{l}{\textit{Time (s) $\downarrow$}} \\
\DiffLTLf (Product) & 0.161 & 0.300 & 0.441 & 0.624 & 0.790 & 0.953 & 1.11 & 1.25 & 1.42 & 1.57 \\
\DiffLTLf (G\"odel) & 0.053 & 0.098 & 0.134 & 0.186 & 0.229 & 0.275 & 0.354 & 0.393 & 0.420 & 0.465 \\
\DiffLTLf (\L{}ukasiewicz) & 0.064 & 0.109 & 0.137 & 0.185 & 0.230 & 0.276 & 0.324 & 0.371 & 0.421 & 0.471 \\
\TILR & 0.911 & 1.90 & 2.66 & 2.74 & 3.02 & 3.62 & 4.29 & 4.95 & 6.95 & 8.04 \\
\FuzzyDFA & 2.32 & 4.73 & 7.12 & 9.56 & 11.91 & 14.35 & 16.72 & 19.09 & 21.58 & 23.84 \\
\FuzzyDFA (orig.) & 3.06 & 4.93 & 7.94 & 9.57 & 12.27 & 13.04 & 15.32 & 17.33 & 19.00 & 21.20 \\
\NeSyA & 3.43 & 6.90 & 10.45 & 13.93 & 17.39 & 21.00 & 24.64 & 27.87 & 29.88 & 32.58 \\
GRU & \textbf{0.004} & \textbf{0.004} & \textbf{0.004} & \textbf{0.005} & \textbf{0.005} & \textbf{0.006} & \textbf{0.006} & \textbf{0.006} & \textbf{0.006} & \textbf{0.006} \\
Transformer & 0.014 & 0.014 & 0.014 & 0.017 & 0.020 & 0.022 & 0.027 & 0.029 & 0.032 & 0.039 \\
\bottomrule
\end{tabular}}
\end{table*}

\end{document}